\documentclass{article}

\usepackage{microtype}
\usepackage{graphicx}
\usepackage{diagbox}
\usepackage{bm}
\usepackage{subfigure}
\usepackage{booktabs} % for professional tables
\usepackage{multirow}
\usepackage{hyperref}

\usepackage[accepted]{icml2024}

\usepackage{amsmath}
\usepackage{amssymb}
\usepackage{mathtools}
\usepackage{amsthm}

\usepackage[capitalize,noabbrev]{cleveref}

\theoremstyle{plain}
\newtheorem{theorem}{Theorem}[section]

\newtheorem{lemma}[theorem]{Lemma}

\theoremstyle{definition}

\theoremstyle{remark}

\usepackage[textsize=tiny]{todonotes}

\icmltitlerunning{Exploration and Anti-Exploration with Distributional Random Network Distillation}

\begin{document}

\twocolumn[
\icmltitle{Exploration and Anti-Exploration with Distributional Random Network Distillation}

% It is OKAY to include author information, even for blind
% submissions: the style file will automatically remove it for you
% unless you've provided the [accepted] option to the icml2024
% package.

% List of affiliations: The first argument should be a (short)
% identifier you will use later to specify author affiliations
% Academic affiliations should list Department, University, City, Region, Country
% Industry affiliations should list Company, City, Region, Country

% You can specify symbols, otherwise they are numbered in order.
% Ideally, you should not use this facility. Affiliations will be numbered
% in order of appearance and this is the preferred way.
\icmlsetsymbol{equal}{*}
\icmlsetsymbol{cor}{$\dagger$}

\begin{icmlauthorlist}
\icmlauthor{Kai Yang}{equal,Tsinghua Shenzhen International Graduate School}
\icmlauthor{Jian Tao}{equal,Tsinghua Shenzhen International Graduate School}
\icmlauthor{Jiafei Lyu}{Tsinghua Shenzhen International Graduate School}
\icmlauthor{Xiu Li}{cor,Tsinghua Shenzhen International Graduate School}
\end{icmlauthorlist}

\icmlaffiliation{Tsinghua Shenzhen International Graduate School}{Tsinghua Shenzhen International Graduate School, Tsinghua University}
% \icmlcorrespondingauthor{\{Kai Yang, Jian Tao, Jiafei Lyu\}}{\{yk22,tj22,lvjf20\}@mails.tsinghua.edu.cn}
\icmlcorrespondingauthor{Xiu Li}{li.xiu@sz.tsinghua.edu.cn}

% You may provide any keywords that you
% find helpful for describing your paper; these are used to populate
% the "keywords" metadata in the PDF but will not be shown in the document
\icmlkeywords{Reinforcement Learning, Exploration, Anti-exploration}

\vskip 0.3in
]

% this must go after the closing bracket ] following \twocolumn[ ...

% This command actually creates the footnote in the first column
% listing the affiliations and the copyright notice.
% The command takes one argument, which is text to display at the start of the footnote.
% The \icmlEqualContribution command is standard text for equal contribution.
% Remove it (just {}) if you do not need this facility.

%\printAffiliationsAndNotice{}  % leave blank if no need to mention equal contribution
\printAffiliationsAndNotice{*Equal contribution. $^\dagger$Corresponding author.} % otherwise use the standard text.
\begin{abstract}
Exploration remains a critical issue in deep reinforcement learning for an agent to attain high returns in unknown environments. Although the prevailing exploration Random Network Distillation (RND) algorithm has been demonstrated to be effective in numerous environments, it often needs more discriminative power in bonus allocation. This paper highlights the ``bonus inconsistency'' issue within RND, pinpointing its primary limitation. To address this issue, we introduce the Distributional RND (DRND), a derivative of the RND. DRND enhances the exploration process by distilling a distribution of random networks and implicitly incorporating pseudo counts to improve the precision of bonus allocation. This refinement encourages agents to engage in more extensive exploration. Our method effectively mitigates the inconsistency issue without introducing significant computational overhead. Both theoretical analysis and experimental results demonstrate the superiority of our approach over the original RND algorithm. Our method excels in challenging online exploration scenarios and effectively serves as an anti-exploration mechanism in D4RL offline tasks. Our code is publicly available at \href{https://github.com/yk7333/DRND}{https://github.com/yk7333/DRND}.
\end{abstract}

\section{Introduction}
\label{submission}

Exploration is a pivotal consideration in reinforcement learning, especially when dealing with environments that offer sparse or intricate reward information. Several methods have been proposed to promote deep exploration \cite{osband2016deep}, including count-based and curiosity-driven approaches. Count-based techniques in environments with constrained state spaces rely on recording state visitation frequencies to allocate exploration bonuses \cite{strehl2008analysis, azar2017minimax}. However, this method encounters challenges in massive or continuous state spaces. In expansive state spaces, ``pseudo counts'' have been introduced as an alternative \cite{bellemare2016unifying, lobel2023flipping, ostrovski2017count, machado2020count}. However, establishing a correlation between counts and probability density requires rigorous criteria \cite{ostrovski2017count}, complicating the implementation of density-based pseudo counts resulting in a significant dependency on network design and hyperparameters.

Curiosity-driven methods motivate agents to explore and learn by leveraging intrinsic motivation. This inherent motivation, often called ``curiosity'', pushes the agent to explore unfamiliar states or actions. Certain approaches derive intrinsic rewards from the prediction loss of environmental dynamics \cite{achiam2017surprise, burda2018large, pathak2017curiosity}. As states and actions become more familiar, these methods become more efficient. However, these methods can face difficulties when essential information is missing or the target function is inherently stochastic, as highlighted by the ``noisy-TV'' problem \cite{pathak2017curiosity}. The Random Network Distillation (RND) method uses the matching loss of two networks for a particular state to be the intrinsic motivation \cite{burda2018large}. It leverages a randomly initialized target network to generate a fixed value for specific states and trains a prediction network to match this output. RND has demonstrated remarkable results in exploration-demanding environments with sparse rewards, such as Montezuma's Revenge. However, RND has its limitations. While its reliance on network loss for intrinsic rewards lacks a robust mathematical foundation, its interpretability should be more evident compared to count-based techniques. Moreover, the RND method grapples with the issue of \textbf{bonus inconsistency}, which becomes apparent during the initial stages of training when no states have been encountered, leading to bonuses that exhibit considerable deviations from a random distribution. RND struggles to precisely represent the dataset's distribution as training progresses, resulting in indistinguishable bonuses.

We introduce the Distributional Random Network Distillation (DRND) approach to tackle the challenge of bonus inconsistency in RND. In contrast to the RND method, our approach employs a predictor network to \textit{distill multiple random target networks}. Our findings demonstrate that the DRND predictor effectively operates as a pseudo-count model. This unique characteristic allows DRND to seamlessly merge the advantages of count-based techniques with the RND method, thereby enhancing performance without incurring additional computational and spatial overheads, as the target networks remain fixed and do not require updates. The curiosity-driven RND method and the pseudo-count Coin Flip Network (CFN, \cite{lobel2023flipping}) method are special cases of our DRND method. Through theoretical analysis and an initial experiment (see Section \ref{sec: inconsistency exp}), we validate that, compared to RND, DRND demonstrates improved resilience to variations in initial state values, provides a more accurate estimate of state transition frequencies, and better discriminates dataset distributions. As a result, DRND outperforms RND by providing better intrinsic rewards.

In online experiments, we combine the DRND method with Proximal Policy Optimization (PPO, \cite{schulman2017proximal}). On the image-based exploration benchmark environments Montezuma's Revenge, Gravitar, and Venture, DRND outperforms baseline PPO, RND, pseudo-count method CFN, and curiosity-driven method ICM \cite{pathak2017curiosity}. In continuous-control gym-robotics environments, our method also outperforms existing approaches. Furthermore, we demonstrate that DRND can also serve as a good anti-exploration penalty term in the offline setting, confirming its ability to provide a better bonus based on the dataset distribution. We follow the setting of SAC-RND \cite{nikulin2023anti} and propose a novel offline RL algorithm, SAC-DRND. We run experiments in D4RL \cite{fu2020d4rl} offline tasks and find that SAC-DRND outperforms many recent strong baselines across various D4RL locomotion and Antmaze datasets.

\begin{figure*}
    \centering
    \includegraphics[height=120pt,width=420pt]{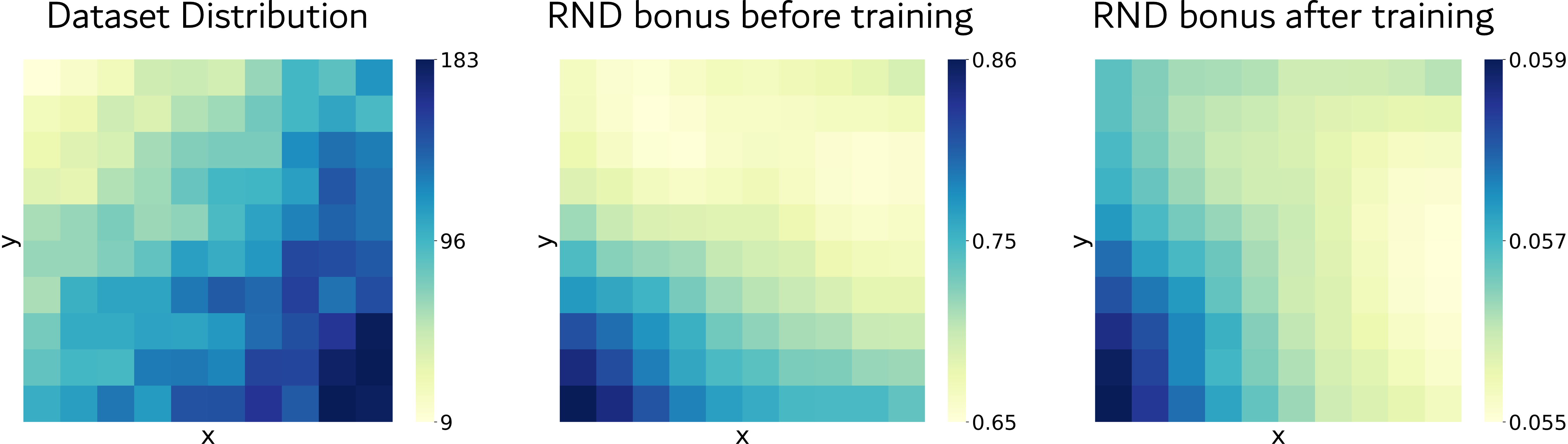}
    \caption{Bonus Heatmap of dataset distribution and RND bonus. The left image illustrates the dataset distribution, the middle image represents the RND bonus before training, and the right image represents the RND bonus after training. A more detailed change process is in Appendix \ref{rndgif}. Ideally, we aim for a uniform bonus distribution before any training and without exposure to the dataset. After extensive training, the expected bonus should inversely correlate with the dataset distribution. The bonus distribution of RND is inconsistent with the desired distribution, indicating a problem with bonus inconsistency. The details of the experiment settings can be found in Appendix \ref{app_exp_detail}}
    \label{bonus inconsistency}
\end{figure*}

\section{Related Work}
\textbf{Count-based exploration.} 
Count-based exploration is a strategy in RL where an agent uses count information to guide its exploration of unknown environments. By keeping track of counts for different states or actions, the agent can estimate the level of unknowns associated with each state or action, prioritizing exploration of those with high unknowns \cite{bellemare2016unifying,machado2020count,martin2017count,tang2017exploration}. These approaches use $r_t = N(s_t)^{-\frac{1}{2}}$ or $r_t = N(s_t, a_t)^{-\frac{1}{2}}$, aiming to balance exploration and exploitation in stochastic MDPs \cite{strehl2008analysis}. Various methods, including CTS \cite{bellemare2016unifying}, PixelCNN \cite{ostrovski2017count}, Successor Counts \cite{machado2020count}, and CFN \cite{lobel2023flipping}, have explored calculating pseudocounts in large state spaces to approximate $N(s_t)$. Furthermore, count-based techniques \cite{kim2023model, hong2022confidence} are employed in offline RL to do anti-exploration. While effective in finite state spaces, these methods heavily rely on the network's ability to approximate probability density functions in large state spaces. Accurately estimating density requires a significant number of samples, which limits the effectiveness of counting methods in situations with small sample sizes or regions of low probability density.

\textbf{Curiosity-driven exploration.}
In curiosity-driven methods, the agent's motivation stems from intrinsic curiosity, often quantified using information-theoretic or novelty-based metrics. One widely used metric involves employing a dynamic model to predict the difference between the expected and actual states, serving as the intrinsic reward \cite{stadie2015incentivizing, achiam2017surprise, pathak2017curiosity}, which helps identify unfamiliar patterns, encouraging exploration in less familiar areas. Alternatively, some approaches use information gain as an intrinsic reward \cite{still2012information, houthooft2016vime}. Still, they demand computationally intensive network fitting and can struggle in highly stochastic environments due to the ``noise TV'' problem \cite{burda2018exploration}.

Another curiosity-driven method is RND \cite{burda2018exploration}, which is a prominent RL exploration baseline. RND employs two neural networks: a static prior and a trainable predictor. Both networks map states to embeddings, with state novelty assessed based on their prediction error, which serves as an exploration bonus. This simplicity has bolstered RND's popularity in exploration algorithms and demonstrated its potential in supervised settings, even suggesting its use as an ensemble alternative for estimating epistemic uncertainty \cite{ciosek2019conservative, kuznetsov2020controlling}. However, common practices, such as using identical architectures for both networks and estimating novelty solely from states, can result in substantial inconsistencies in reward bonuses.

\textbf{Anti-Exploration in Model-free Offline RL}
Offline RL addresses the problem of learning policies from a logged static dataset. Model-free offline algorithms do not require an estimated model and focus on correcting the extrapolation error \cite{fujimoto2019off} in the off-policy algorithms. The first category emphasizes regularizing the learned policy to align with the behavior policy \cite{kostrikov2021offline,wang2018exponentially,wang2020critic,wu2019behavior,xie2021bellman,fujimoto2021minimalist}. The second category aims to prevent the OOD actions by modifying the value function \cite{kumar2020conservative,lyu2023state,lyu2022mildly,an2021uncertainty,ghasemipour2022so,yang2022rorl}. These methods employ dual penalization techniques in actor-critic algorithms to facilitate effective offline RL policy learning. These approaches can be further categorized into ensemble-free methods and ensemble-based methods. The ensemble-based methods quantify the uncertainty with ensemble techniques to obtain a robust value function, such as SAC-N \cite{an2021uncertainty} and RORL \cite{yang2022rorl}. The ensemble-free methods adapt conservatism to a value function instead of many value functions \cite{kumar2020conservative,lyu2022mildly,rezaeifar2022offline}. These methods require punishment for states and actions outside of the dataset distribution, which is called an ``anti-exploration'' bonus \cite{rezaeifar2022offline} for the agent. Unlike online RL, where novelty bonuses incentivize exploration, offline RL leans towards conservatism, aiming to reduce rewards in uncharted scenarios. In this work, we introduce a distributional random network distillation approach to serve as a novel anti-exploration method, demonstrating the efficacy of SAC-DRND across various offline RL datasets.

\section{Preliminaries}
\textbf{MDP.} We base our framework on the conventional Markov Decision Process (MDP) formulation as described in \cite{sutton1998introduction}. In this setting, an agent perceives an observation $o\in \mathcal{O}$ and executes an action $a\in \mathcal{A}$. The transition probability function, denoted by $P(s'|s, a)$, governs the progression from the current state $s$ to the subsequent state $s'$ upon the agent's action $a$. Concurrently, the agent is awarded a reward $r$, determined by the reward function $r:\mathcal{A}\times\mathcal{S}\to\mathbb{R}$. The agent's objective is to ascertain a policy $\pi(a|o)$ that optimizes the anticipated cumulative discounted returns, represented as $\mathbb{E}_{\pi}\left[\sum_{t=0}^{\infty} \gamma^{t} r\left(s_{t}, a_{t}\right)\right]$, where $\gamma \in[0,1)$ serves as the discount factor.

\begin{figure*}[t]
    \centering
    \subfigure[RND]{
    \begin{minipage}[t]{0.48\linewidth}
    \centering
    \includegraphics[height=78pt,width=220pt]{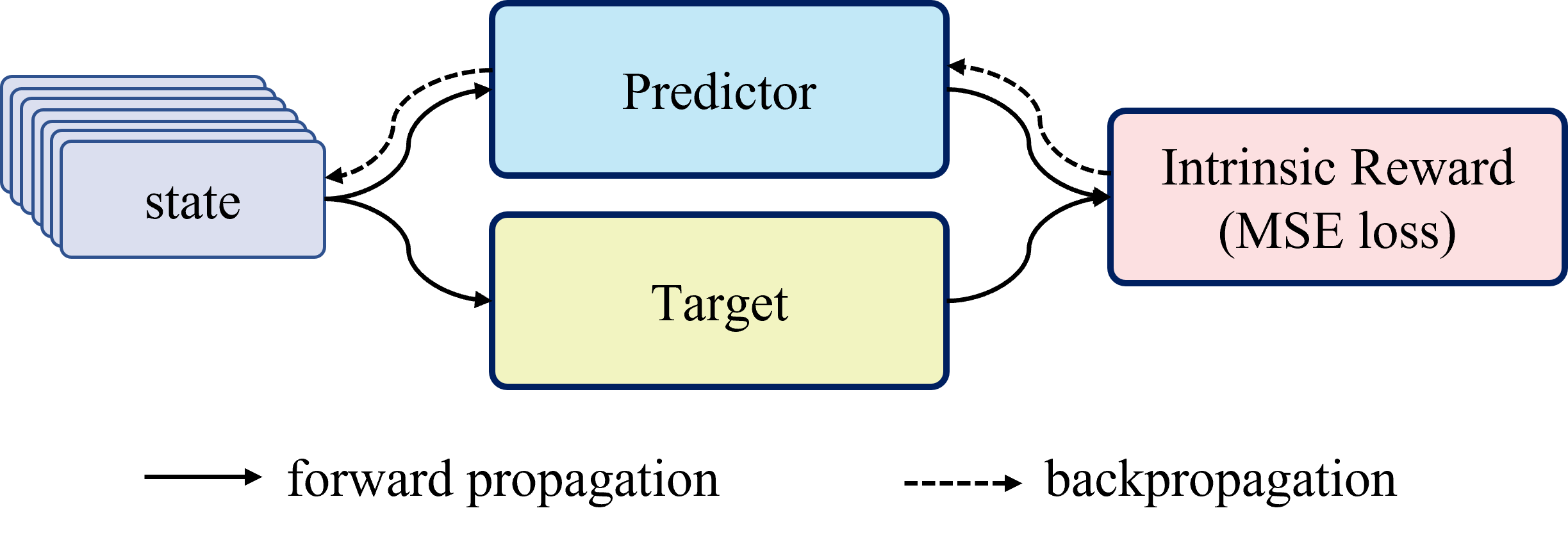}
    \end{minipage}
    }
    \subfigure[DRND]{
    \begin{minipage}[t]{0.46\linewidth}
    \centering
    \includegraphics[height=78pt,width=240pt]{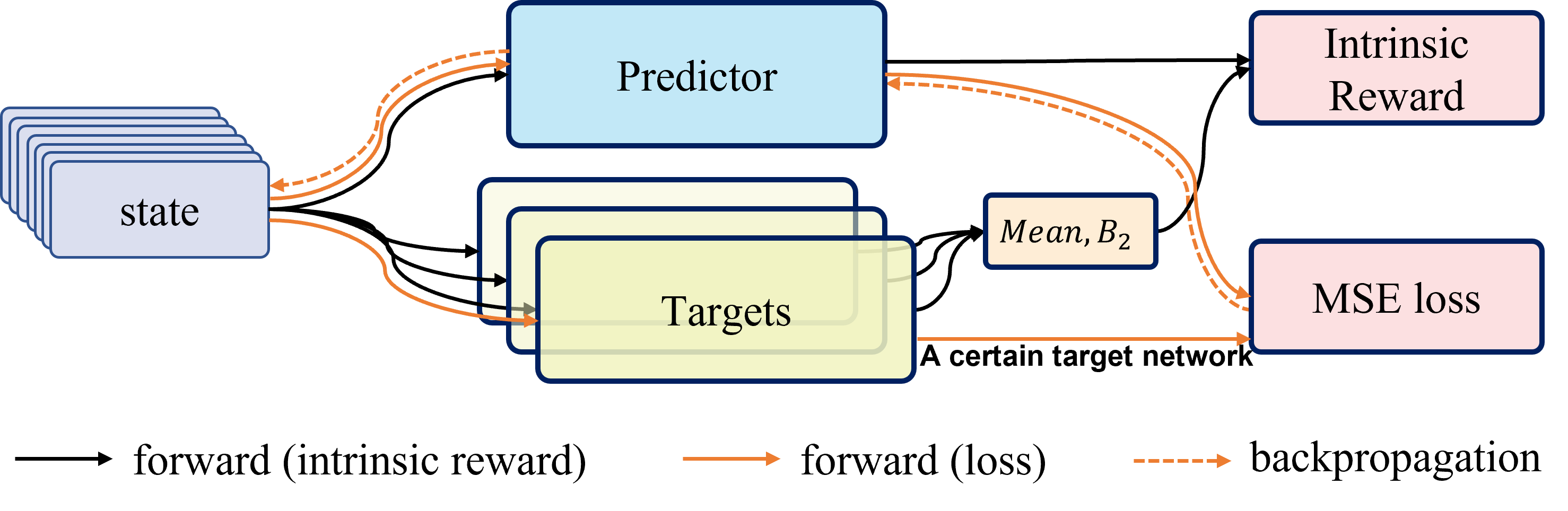}
    \end{minipage}
    }
    \caption{Diagram of RND and DRND. Compared to the RND method that only distills a fixed target network, our method distills a randomly distributed target network and utilizes statistical metrics to assign a bonus to each state.}
    \label{diagram}
\end{figure*}

\textbf{Intrinsic reward.}
To enhance exploration, a common approach involves augmenting the agent's rewards in the environment with intrinsic rewards as a bonus. These intrinsic rewards, denoted as $b_t(s_t, a_t)$, incentivize agents to explore unfamiliar states and take unfamiliar actions. In offline RL, the intrinsic reward is an anti-exploration penalty term to discourage OOD actions. Upon incorporating the intrinsic reward $b(s_t, a_t)$ into the original target of the Q value function, the adjusted target can be expressed as follows:
$$
y_{t} = 
\begin{cases} 
{r}_{t}+\lambda b\left(s_t, a_t\right)+\gamma \max _{a^\prime} Q_{\theta^{\prime}}\left(s_{t+1}, a^\prime\right) & {\small \text{(online)}} \\
{r}_{t}+\gamma \mathbb{E}_{a^\prime}\left[Q_{\theta^{\prime}}\left(s_{t+1}, a^\prime\right)-\lambda b\left(s_{t+1}, a^\prime\right)\right] &{\small \text{(offline)}}
\end{cases}
$$
where $\lambda$ is the scale of the bonus for the update of the value network.

\section{Method}
The RND method utilizes two neural networks: a fixed, randomly initialized target network $\hat f: \mathcal{O} \rightarrow \mathbb{R}^{k}$, and a predictor network $f: \mathcal{O} \rightarrow \mathbb{R}^{k}$ trained on agent-collected data, where $\mathcal{O}$ is the observation space. In this section, we highlight RND's primary issue and introduce our method, Distributional Random Network Distillation (DRND).
\subsection{Bonus Inconsistencies in Random Network Distillation}

\label{section3.1}
The RND method faces challenges with bonus inconsistencies, which can be categorized into initial and final bonus inconsistencies. The initial bonus inconsistency relates to the uneven distribution of bonuses among states at the beginning of training. Addressing this issue is crucial to preventing significant bonus value disparities among states. Conversely, the final bonus inconsistency arises when the final bonuses do not align with the dataset distribution, making it hard for the agent to effectively distinguish between frequently visited states and those encountered relatively fewer times. This issue becomes particularly pronounced after substantial updates to the predictor network, which hinders the agent's ability to engage in deep exploration. This issue is depicted in Figure \ref{bonus inconsistency}.

To tackle this, we introduce a method that distills a random distribution, enhancing performance with minimal computational overhead and addressing the bonus inconsistency challenges.

\subsection{Distill the target network of random distribution}

\begin{figure*}[t]
    \centering
    \includegraphics[height=220pt,width=460pt]{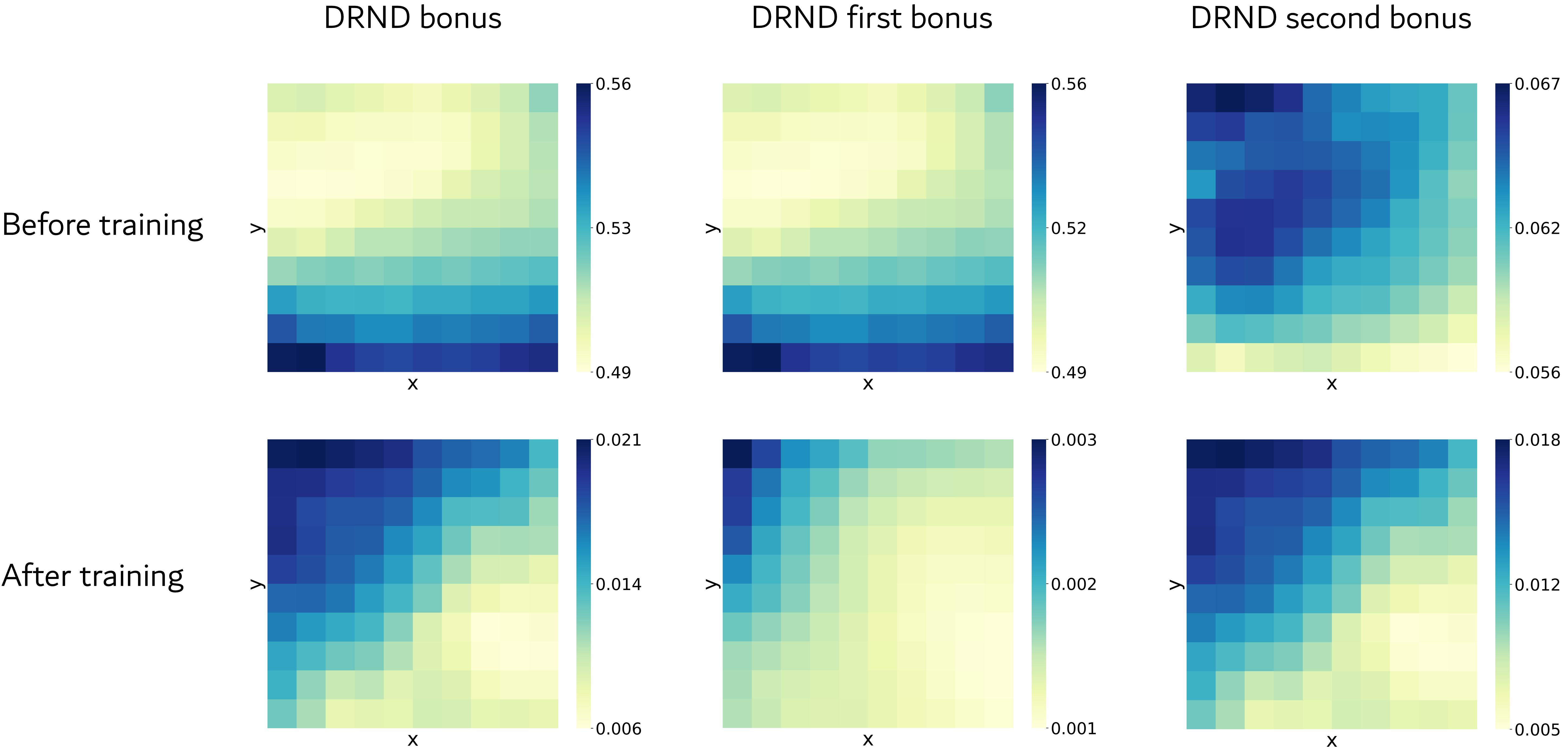}
    \caption{Distribution of DRND bonus. The dataset distribution is the same as Figure \ref{bonus inconsistency}. These illustrations depict the distribution of the DRND bonus, including the first bonus and the second bonus. The first bonus is predominant before training, and the second bonus becomes more prominent after training.}
    \label{relitu}
\end{figure*}

Unlike RND, which only has one target network $\bar{f}$, the DRND algorithm has $N$ target networks $\bar{f}_1,\bar{f}_2,...,\bar{f}_N$, which are from a random distribution with randomly initialized parameters and do not participate in training. In DRND, we use $s$ as input in the online setting and $(s,a)$ pair as input in the offline setting. For simplicity, we define $x=(s, a)$ (offline setting) or $x=(s)$ (online setting). For each state-action pair $x$, we construct a variable $c(x)$ that satisfies the distribution:
\renewcommand{\arraystretch}{1.5}
$$
c(x)\sim \begin{array}{c|c c c c}
 \hline
 X & \bar{f}_1(x) & \bar{f}_2(x) & \dots & \bar{f}_N(x)\\
\hline
 P & \frac{1}{N} & \frac{1}{N} & \dots & \frac{1}{N} \\
 \hline
\end{array}
$$
For simplicity, we use some symbols to record the moments of the distribution:
\begin{align}
\mu(x) &= \mathbb{E}[X] = \frac{1}{N} \sum_{i=1}^{N} \bar{f}_i(x) \\
B_2(x) &= \mathbb{E}[X^2] = \frac{1}{N} \sum_{i=1}^{N} (\bar{f}_i(x))^2.
\end{align}
Each time $x$ occurs, $c(x)$ is sampled from this distribution. We use a predictive network $f_\theta(x)$ to learn the variable $c(x)$, although using a fixed network to learn a random variable is impossible. We use the MSE loss function to force $f_\theta(x)$ to align with $c(x)$ and the loss is 
\begin{equation} \label{loss}
L(\theta) =\|f_\theta(x) - c(x) \|^2.
\end{equation}
By minimizing the loss, the optimal $f_*(x)$ when the state-action pair $x$ appears $n$ times is 
\begin{equation} \label{mean}
f_*(x) = \frac{1}{n} \sum_{i=1}^{n}c_i(x),
\end{equation}
where $c_i(x)$ is defined as $c(x)$ at the $i$-$th$ occurrence of state $x$.  For RND, the more times the predictor network is trained in the same state, the closer the output of the target network is. Therefore, directly using loss as a bonus encourages agents to explore new states. The bonus of our method is not equal to the loss since the loss of random variable fitting is unstable. The expected value of the prediction net is given by
\begin{equation}
\mathbb{E}[f_{\theta^*}(x)] = \mathbb{E}\left[\frac{1}{n} \sum_{i=1}^{n}c_i(x)\right] = \mu(x),
\end{equation}
where $n$ is the count of occurrences of $x$. After multiple training iterations, this value approaches the mean of the target network distribution. Hence, to measure the deviation from this mean, the first bonus of DRND is defined as
\begin{equation} \label{bonus1}
b_1(x) = \|f_\theta(x) -\mu(x)\|^2.
\end{equation}
Compared to predicting the output of one target net, predicting the mean of multiple networks is equivalent to passing through a high-pass filter on the output of multiple networks, which can avoid the problem of initial bonus inconsistency due to extreme values in one network. Especially if the network is linear, this bonus inconsistency can be quantitatively calculated.
\begin{lemma} \label{lemma1}
Let \( \tilde{\theta} \) and \( \bar{\theta}_i, i = 1, 2, \ldots, N \) be i.i.d. samples from \( p(\theta) \). Given the linear model \( f_\theta(x) = \theta^T x \), the expected mean squared error is
{\small
\begin{equation} \label{l1}
E_{\tilde{\theta},\bar{\theta}_1,...\bar{\theta}_N}\left[\left\|f_{\tilde{\theta}}(x) - \frac{1}{N} \sum_{i=1}^N f_{\bar{\theta}_i}(x)\right\|^2\right] = \left(1+\frac{1}{N}\right) x^T \Sigma x,
\end{equation}
}
where \(  \Sigma \) is the variance of \( p(\theta) \).
\end{lemma}
The complete proof can be seen in Appendix \ref{proof1}
. Lemma \ref{lemma1} shows that if the predictor parameters and target parameters are sampled from the same distribution, the expectation of the first bonus is a function of input $x$.
\begin{lemma} \label{lemma2}
Under the assumptions of Lemma \ref{lemma1}, let $x_1,x_2 \in \mathbb{R}^d$, $p(
\theta) \sim N(\mu,\sigma^2)$. The bonus difference of $x_1$ and $x_2$ is 
$\frac{(1+N)\sigma^2}{N} (\|x_2\|^2-\|x_1\|^2)$.
\end{lemma}

\begin{proof}[Proof Sketch]
When $p(\theta) \sim N(\mu,\sigma^2)$, the variance of $p(\theta)$ is a constant $\sigma^2$. The right side of \cref{l1} can be rewritten as $\left(1+\frac{1}{N}\right)\sigma^2\|x\|^2$. So the bonus difference of $x_1$ and $x_2$ is $\left(1+\frac{1}{N}\right)\sigma^2(\|x_2\|^2-\|x_1\|^2)$.
\end{proof}

\begin{table*}[htbp]
  \centering
  \caption{Comparison of the bonus distributions of RND and DRND against (a) the uniform distribution $U$, (b) the distribution of $1/\sqrt{n}$. $P$ denotes the distribution of RND bonus or DRND bonus.}
  \small
  \resizebox{0.58\linewidth}{!}{
    \begin{tabular}{cccc}
    \toprule
    Metrics & RND & DRND \\
    \midrule
    $D_{KL}(P\|U)$ (Before training) & 0.0377$\pm$0.0248 & \textbf{0.0070}$\pm$0.0063 \\
    $D_{KL}(P\|1/\sqrt{n})$ (After training) & 0.0946$\pm$0.0409 & \textbf{0.0476}$\pm$0.0389 \\
    \bottomrule
    \end{tabular}%
    }
  \label{RNDandDRND}%
\end{table*}%

\begin{table*}[htbp]
  \centering
  \caption{KL-divergence comparison of the distributions of overall bonus $b$, individual bonus terms $b_1$, $b_2$ against (a) the uniform distribution $U$, (b) the distribution of $1/\sqrt{n}$. $P$ represents the bonus distribution of $b$, $b_1$ or $b_2$.}
  \small
  \label{b1b2}%
  \resizebox{0.75\linewidth}{!}{
    \begin{tabular}{cccc}
    \toprule
    Metrics & $b$ & $b_1$ & $b_2$ \\
    \midrule
    $D_{KL}(P\|U)$ (Before training) & 0.0067$\pm$0.0035 & 0.0070$\pm$0.0038 & 0.0104$\pm$0.0055 \\
    $D_{KL}(P\|1/\sqrt{n})$ (After training) & 0.0524$\pm$0.0248 & 0.0703$\pm$0.0404 & 0.0396$\pm$0.0209 \\
    \bottomrule
    \end{tabular}%
    }
\end{table*}%

Lemma \ref{lemma2} suggests that when the input \( x \) is confined to a bounded interval, and when \cref{bonus1} is utilized to calculate the initial bonus, the expected maximal difference is modulated by the number of target networks. Importantly, this anticipated discrepancy tends to decrease as \( N \) increases. This observation substantiates that our DRND method, equipped with \( N \) target networks, exhibits lower bonus inconsistency under a uniform distribution than the RND method, which uses only a single target network.

However, it is essential to note that the network fitting loss determines this bonus. Consequently, it cannot distinguish between states visited multiple times, which stands in contrast to count-based and pseudo-count methods, which do not address the issue of final bonus inconsistency.

\begin{figure*}[t]
    \centering
    \subfigure{
    \begin{minipage}[t]{0.45\linewidth}
    \centering
    \includegraphics[height=128pt,width=205pt]{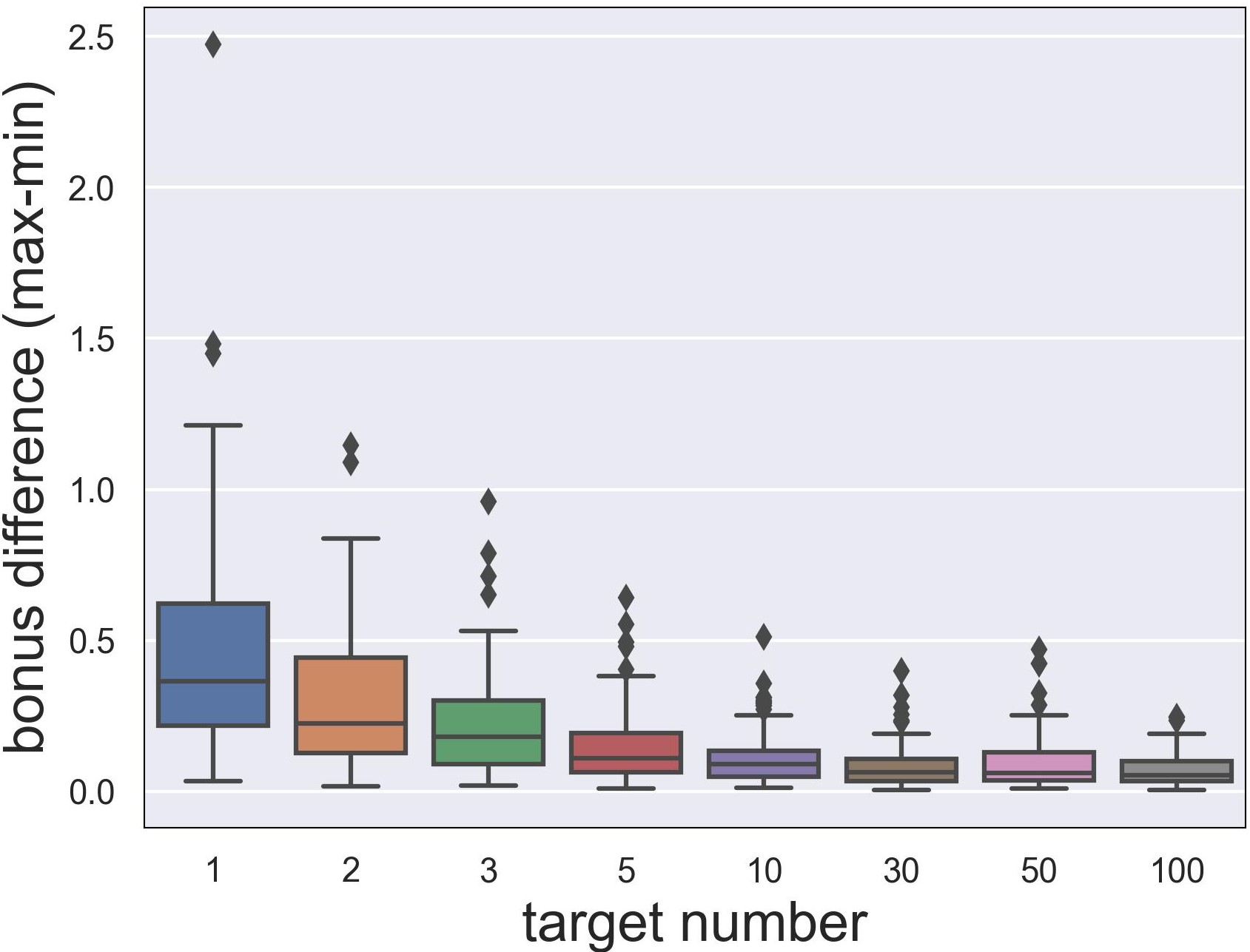}
    \end{minipage}
    }
    \subfigure{
    \begin{minipage}[t]{0.45\linewidth}
    \centering
    \includegraphics[height=125pt,width=225pt]{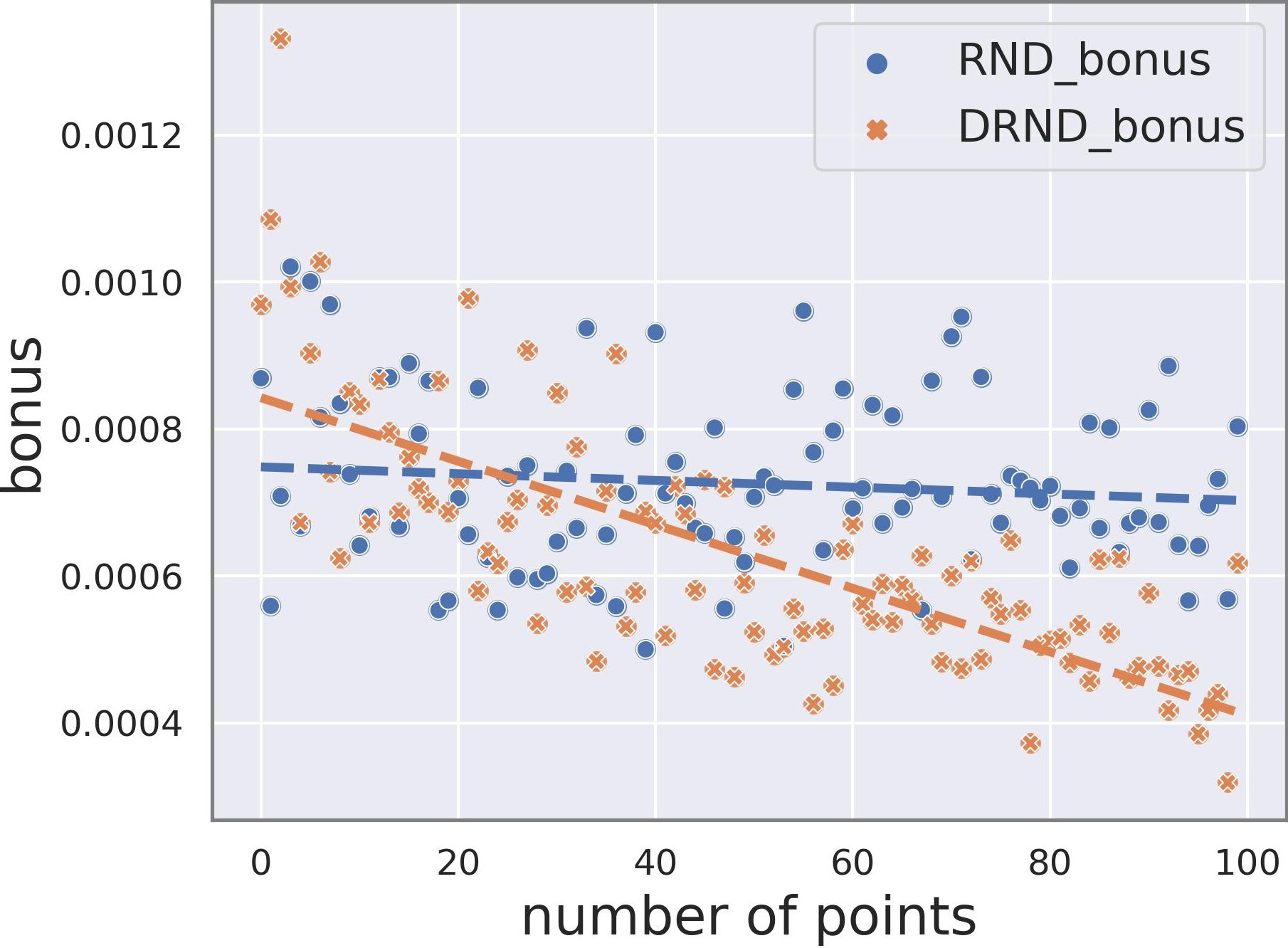}
    \end{minipage}
    }
    \caption{Inconsistency experiments mentioned in Section \ref{sec: inconsistency exp}. We plot the intrinsic reward distribution of RND and DRND before and after training on a mini-dataset. \textbf{Left}: the box plot of the difference between the maximum and minimum intrinsic rewards over 10 independent runs \textbf{before} training. \textbf{Right}: the intrinsic rewards for each data point \textbf{after} training.}
    \label{box}
\end{figure*}
\begin{figure*}[t]
    \centering
    \subfigure{
    \begin{minipage}[t]{0.33\linewidth}
        \centering
        \includegraphics[height=100pt,width=158pt]{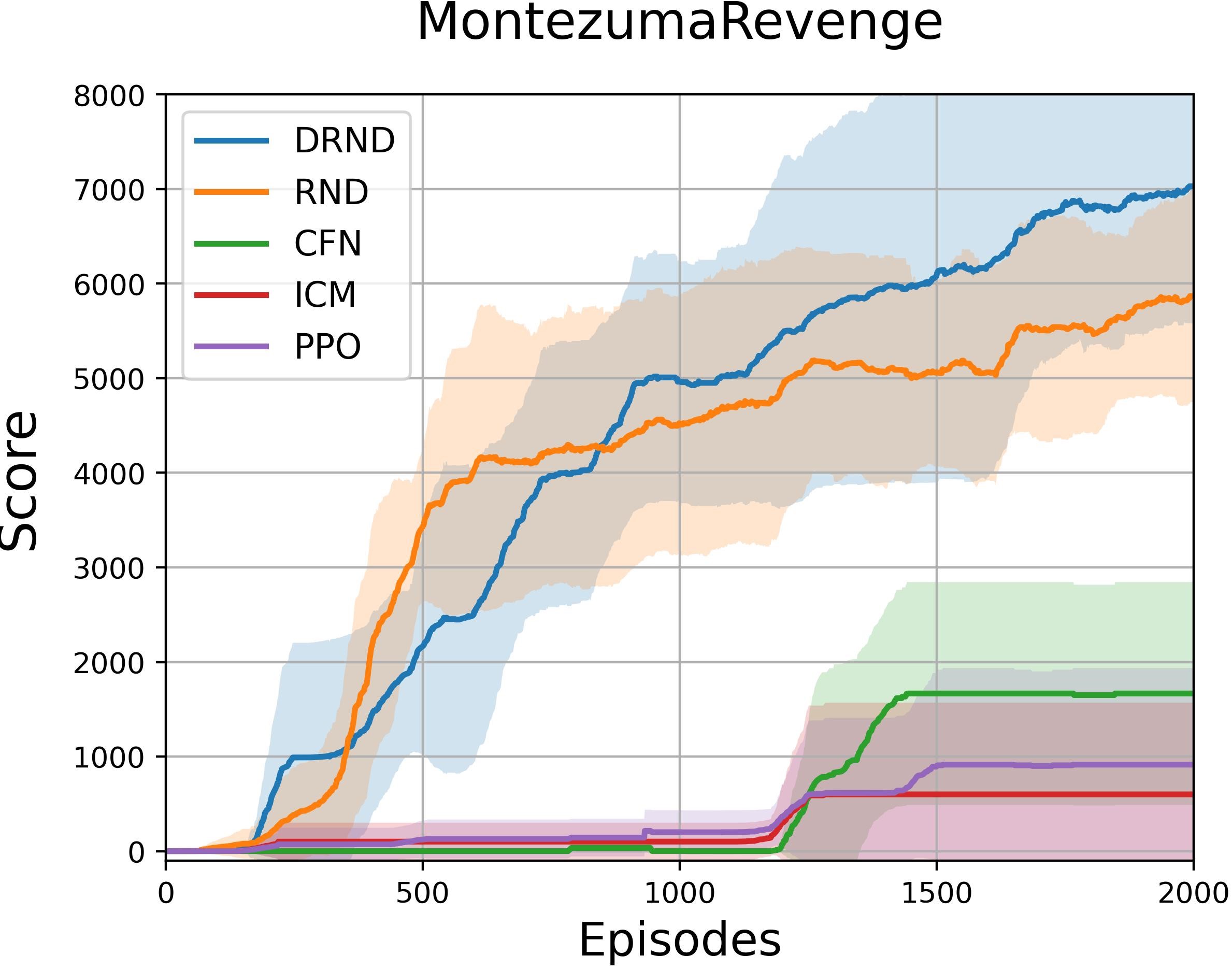}
    \end{minipage}%
    }
    \subfigure{
    \begin{minipage}[t]{0.31\linewidth}
        \centering
        \includegraphics[height=100pt,width=148pt]{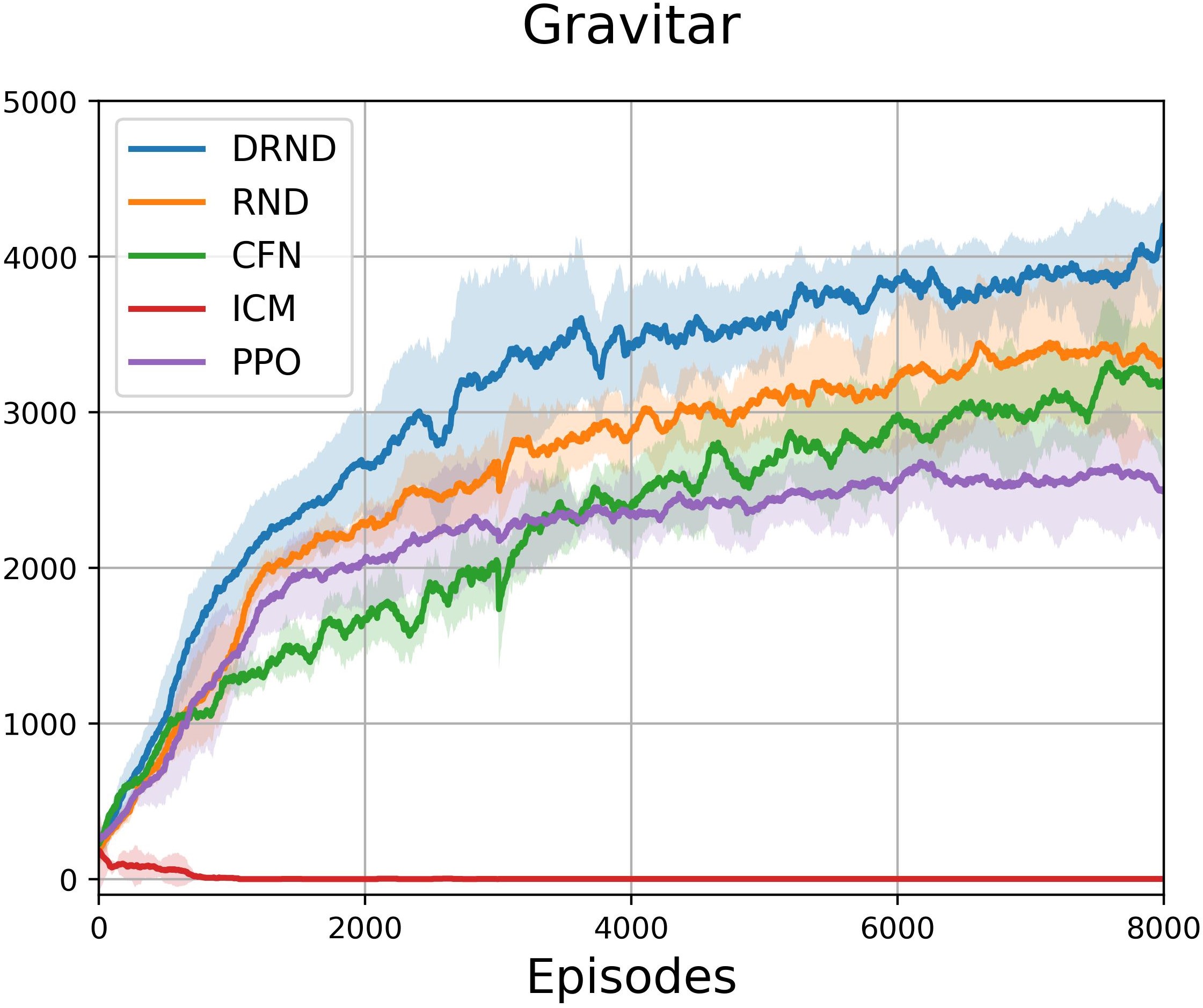}
    \end{minipage}
    }
    \subfigure{
    \begin{minipage}[b]{0.32\linewidth}
        \centering
        \includegraphics[height=100pt,width=150pt]{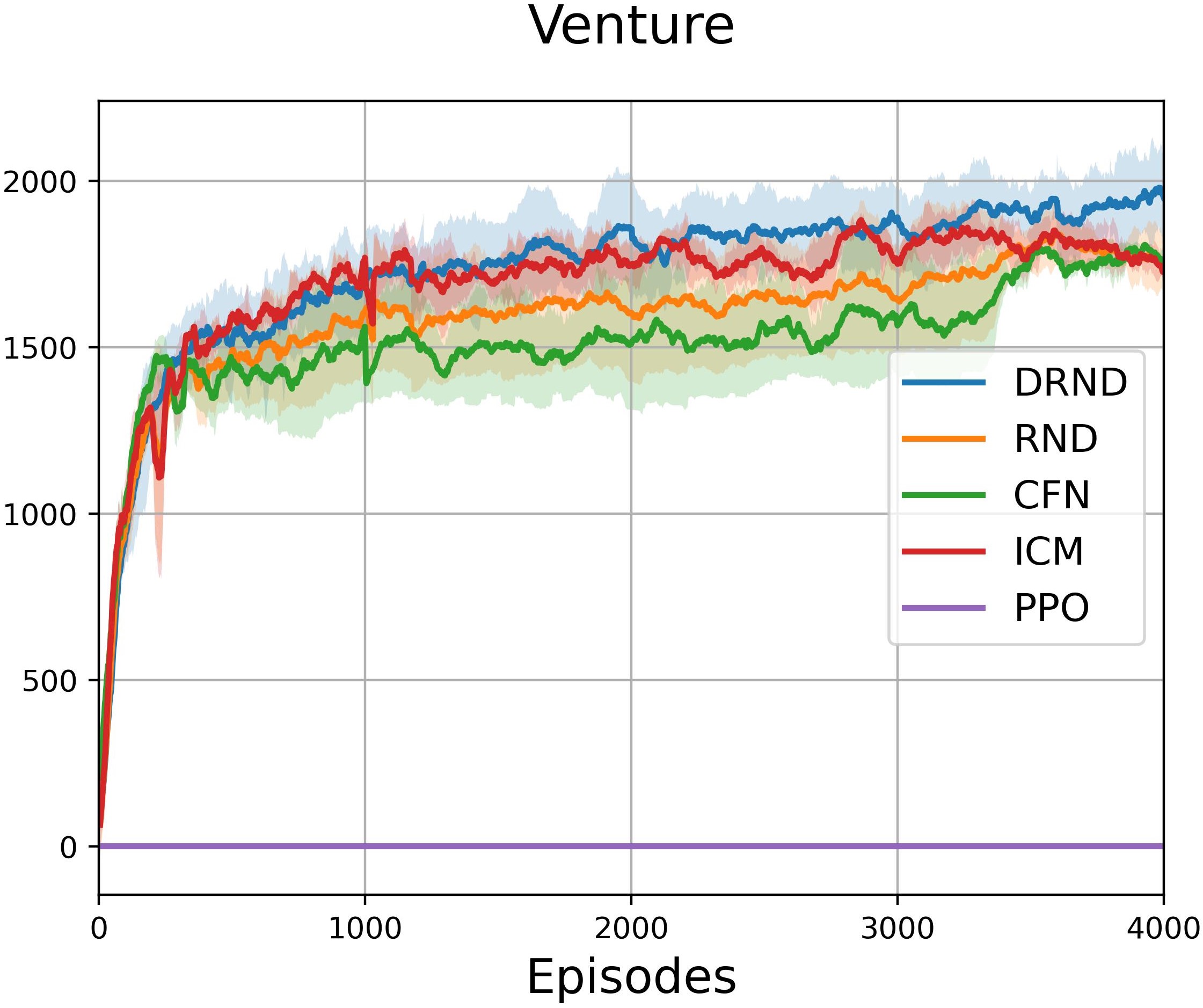}
    \end{minipage}
    }
    \caption{Mean episodic return of DRND method, RND method, and baseline PPO method on three Atari games. All curves are averaged over 5 runs.
 }
    \label{atari}
\end{figure*}

\subsection{The DRND predictor is secretly a pseudo-count model}
It is essential to track data occurrence frequencies to address inconsistent final bonuses. Traditional count-based methods use large tables to record state visitations, while pseudo-count strategies use neural networks for estimation, providing a scalable insight into state visits. However, these methods introduce computational and storage complexities, particularly when dealing with high-dimensional inputs. We constructed a statistic that indirectly estimates state occurrences without extra auxiliary functions.
\begin{lemma} \label{lemma3}
    Let $f_*(x)$ be the optimal function which satisfy \cref{mean}, the statistic
    \begin{equation} \label{y}
y(x) = \frac{[f_*(x)]^2 - [\mu(x)]^2}{B_2(x) - [\mu(x)]^2}
\end{equation} is an unbiased estimator of $1/n$ with consistency.
\end{lemma}
The complete proof can be found in Appendix \ref{proof}. Lemma \ref{lemma3} shows that when $n$ is large, this statistic can effectively recover the number of occurrences of the state $n$, thus implicitly recording the number of occurrences of the state like the pseudo-count method. By minimizing \cref{loss} can make $f_\theta(x)$ and $f_*(x)$ infinitely close, so we replace $f_*(x)$ in $y(x)$ with $f_\theta(x)$ and approximately assume that they are equal. The DRND predictor potentially stores in its weights how much of each state vector is present in the dataset. To correspond to $\sqrt{1/n}$ of the count-based method, the second bonus of the DRND agent is 
\begin{equation} \label{bonus2}
    b_2(x) = \sqrt{\frac{[f_\theta(x)]^2 - [\mu(x)]^2}{B_2(x) - [\mu(x)]^2}},
\end{equation}
which is the estimation of $\sqrt{1/n}$.

\subsection{Bonus of the DRND agent}
% In summary, the total bonus, as seen in \cref{bonus1,bonus2}, is
The total bonus that combines \cref{bonus1,bonus2} gives

\begin{equation} \label{reward}
b(x) = \alpha \|f_\theta(x) - \mu(x)\|^2 + (1-\alpha) \sqrt{\frac{[f_\theta(x)]^2 - [\mu(x)]^2}{B_2(x) - [\mu(x)]^2}},
\end{equation}

where $\alpha$ represents the scaling factor for the two bonus terms. Figure \ref{diagram} shows the diagram of our method and the RND method. For smaller values of $n$, the variance of the second bonus estimate is substantial, rendering the first bonus $b_1$ a more dependable measure for states with infrequent occurrences. Conversely, as $n$ increases, the variance of the second bonus $b_2$ approaches zero, enhancing its reliability. From the experiments, we found that during the initial training phase, the magnitudes of the two bonuses, denoted as $b_1$ and $b_2$, are nearly the same. However, as training progresses, the rate at which the first bonus $b_1$ decreases is much faster than that of the second bonus $b_2$. Eventually, the magnitude of $b_1$ becomes approximately two orders of magnitude lower than $b_2$. Therefore, if we want $b_1$ to dominate during the early stages of training and have $b_2$ dominate during the later stages, we do not need to set a dynamic coefficient to achieve this functionality. Instead, we can simply set it as a fixed constant $\alpha$. The distribution of various bonuses as well as the total bonus of the DRND method before and after training can be seen in Figure \ref{relitu}.

The aforementioned heatmap serves as a qualitative illustration. We now provide quantitative experimental results to demonstrate that DRND offers superior bonuses compared to RND. In count-based exploration methods, it is common to assess the rationality of the bonus distribution by comparing it with the $1/\sqrt{n}$ distribution, where $n$ represents the state visitation count. We evaluate the discrepancy between the intrinsic reward distribution and the uniform distribution $U$ before training, and the distribution of $1/\sqrt{n}$ after training using KL divergence. We randomly select 100 initial dataset distributions and report the mean and variance of the KL divergence. This assessment is denoted as $D_{KL}(P\|U)$ and $D_{KL}(P\|1/\sqrt{n})$, where $P$ represents the distribution of RND bonus or DRND bonus. The results obtained are presented in Table \ref{RNDandDRND}. It is evident that the KL divergence between the bonuses of DRND and the uniform distribution is smaller during the initial training stages, effectively promoting uniform exploration by the agent. In the later stages of training, DRND better aligns with the distribution of $1/\sqrt{n}$, thereby facilitating deeper exploration by the agent.

To validate that $b_1$ gives better and more uniform bonuses during the early stages of training, while $b_2$ better captures the state visitation count $n$ later on, we similarly measure the KL divergence between each bonus term $b_1$, $b_2$, the overall bonus $b$ and (a) the uniform distribution $U$ before training; (b) the distribution of $1/\sqrt{n}$ after training. The results are presented in Table \ref{b1b2}. 

From the table, it can be observed that in the early stages of training, $b_1$ performs better, providing a better solution to the initial bonus inconsistency. Conversely, in the later stages of training, $b_2$ performs better, addressing the final bonus inconsistency more effectively. Moreover, from the KL divergence of the total bonus, it is evident that setting a fixed value for $\alpha$ achieves the desired effect of having $b_1$ dominate during the early stages and $b_2$ dominate during the later stages, without the need for dynamic adjustments. 

The combination of $b_1$ and $b_2$ depends on the environment and task at hand. In some environments, $b_1$ may be more crucial, while in others, the contribution of $b_2$ might be greater. However, regardless of the specifics, properly integrating both can lead to improved performance.

% \subsection{Relationship with RND and pseudo-count methods} 
\subsection{Connections between DRND and prior methods} 
The target network of the DRND method remains static throughout the training process. Moreover, its loss and intrinsic reward computations do not entail additional backpropagation steps, thereby maintaining computational efficiency comparable to that of the RND algorithm. Specifically, when the hyperparameters $\alpha$ and $N$ are set to 1, the expressions for the loss and intrinsic reward become simplified to $\left \| f_\theta(x) - \bar{f}(x) \right \| ^2$, which coincides with the formulation utilized in the RND framework. In contrast to count-based and pseudo-count methods, we forgo the utilization of an additional network or table to track state occurrences. Instead, we estimate these occurrences using the statistical information derived from the prediction network itself.\renewcommand{\arraystretch}{0.85} When $\alpha$ is set to 0 and $c(x)$ follows the distribution $c(x)\sim \begin{array}{c|c c }
\hline
X & -1 & 1\\
\hline
P & 0.5 &0.5 \\
\hline
\end{array}$, the expressions for the loss and intrinsic reward are simplified to $\|f_\theta(x)\|^2$ and $\|f_\theta(x)\|$, which align with the pseudo-count approach CFN.

\begin{figure*}[t]
    \centering
    \subfigure{
    \begin{minipage}[]{0.24\linewidth}
        \centering
        \includegraphics[height=80pt,width=120pt]{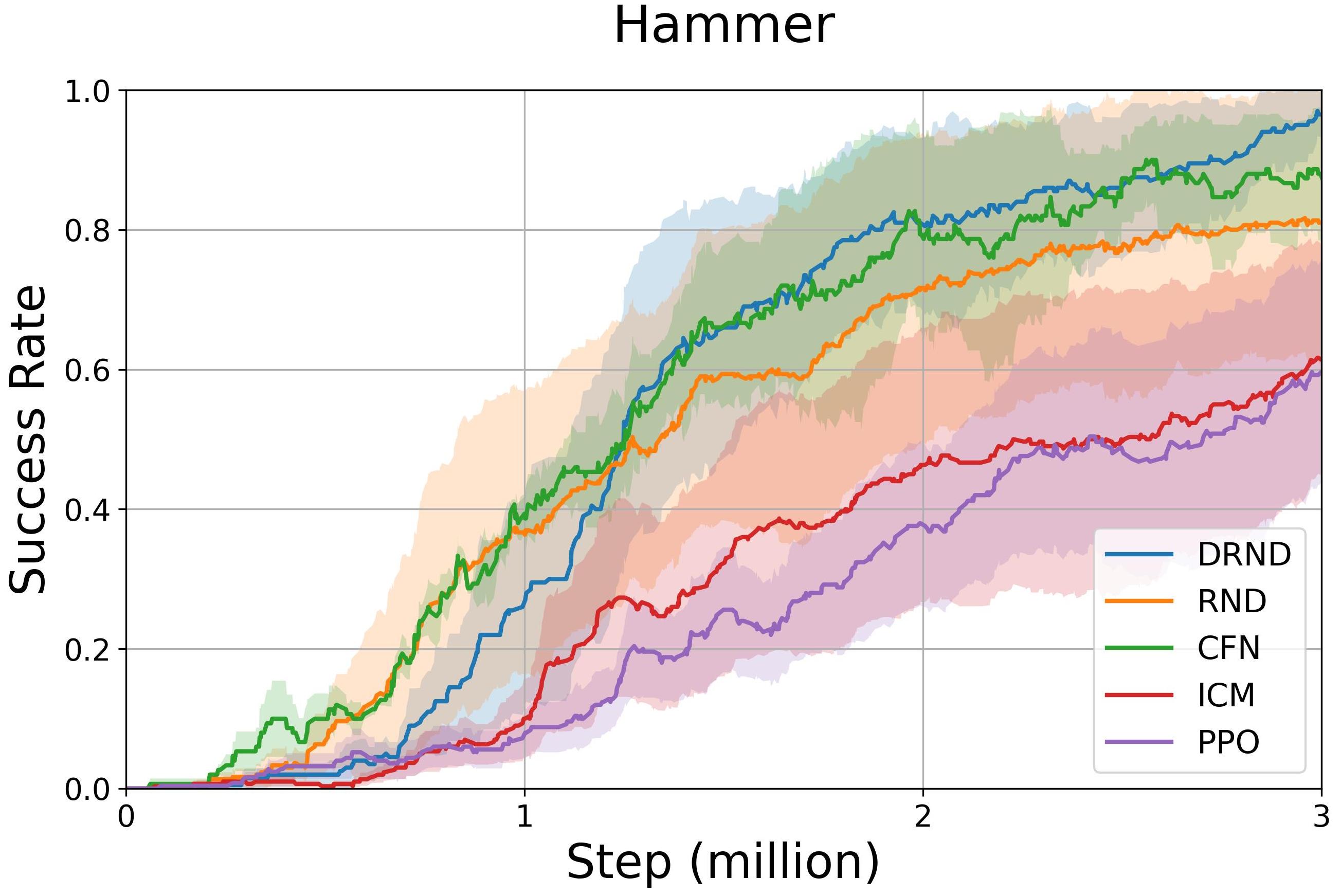}
    \end{minipage}%
    }
    \subfigure{
    \begin{minipage}[]{0.23\linewidth}
        \centering
        \includegraphics[height=80pt,width=110pt]{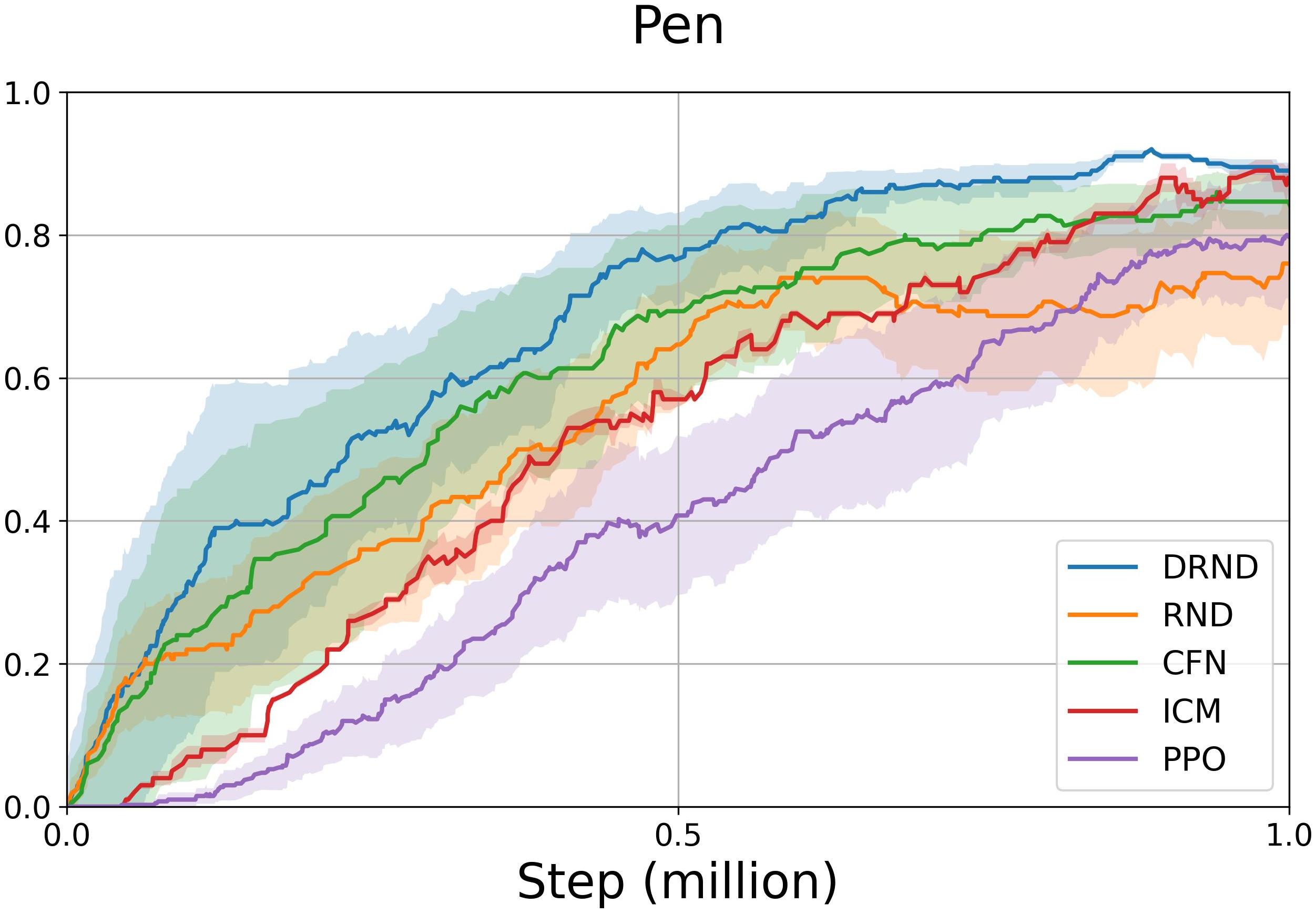}
    \end{minipage}
    }
    \subfigure{
    \begin{minipage}[]{0.23\linewidth}
        \centering
        \includegraphics[height=80pt,width=110pt]{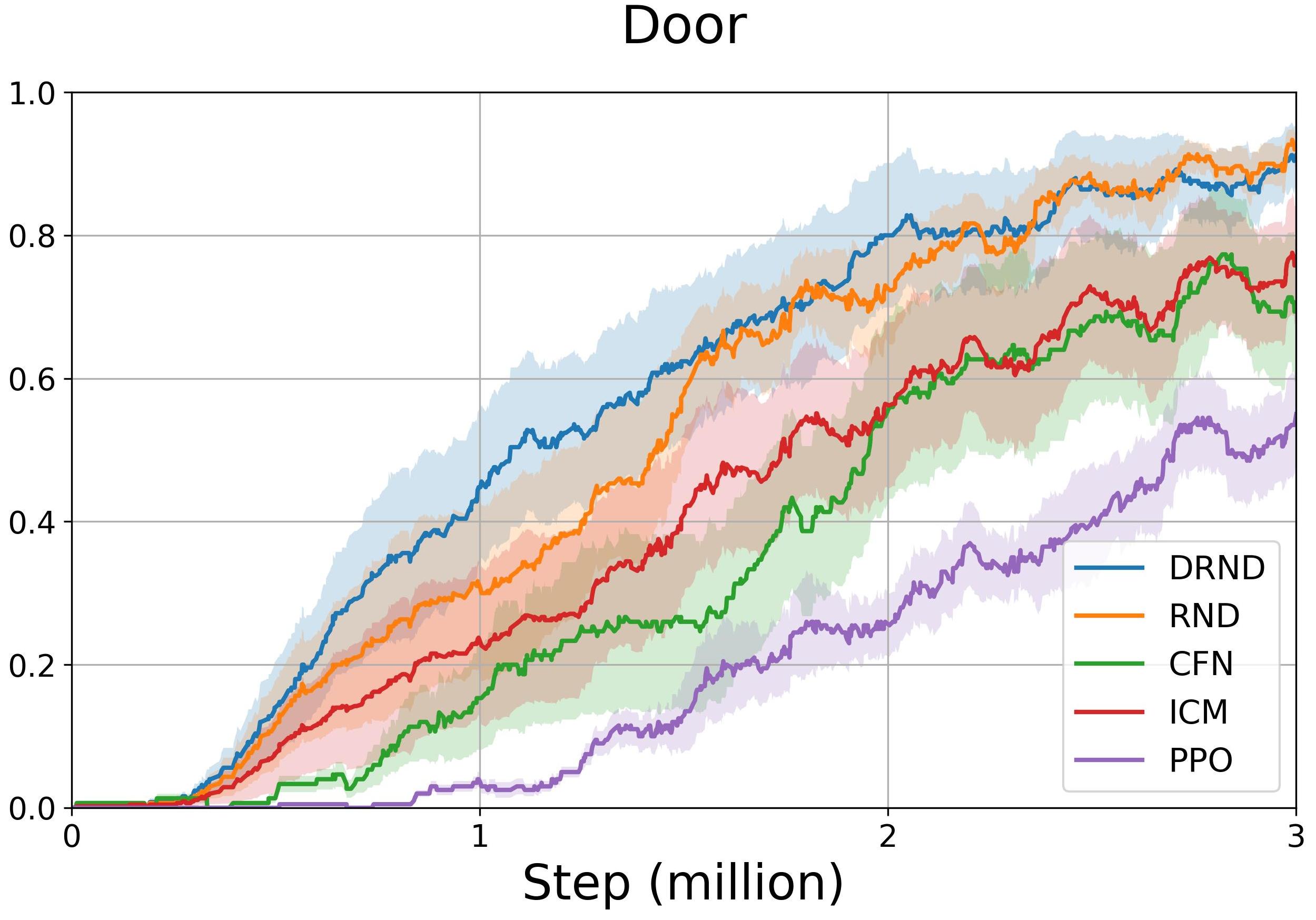}
    \end{minipage}
    }
    \subfigure{
    \begin{minipage}[]{0.23\linewidth}
        \centering
        \includegraphics[height=80pt,width=110pt]{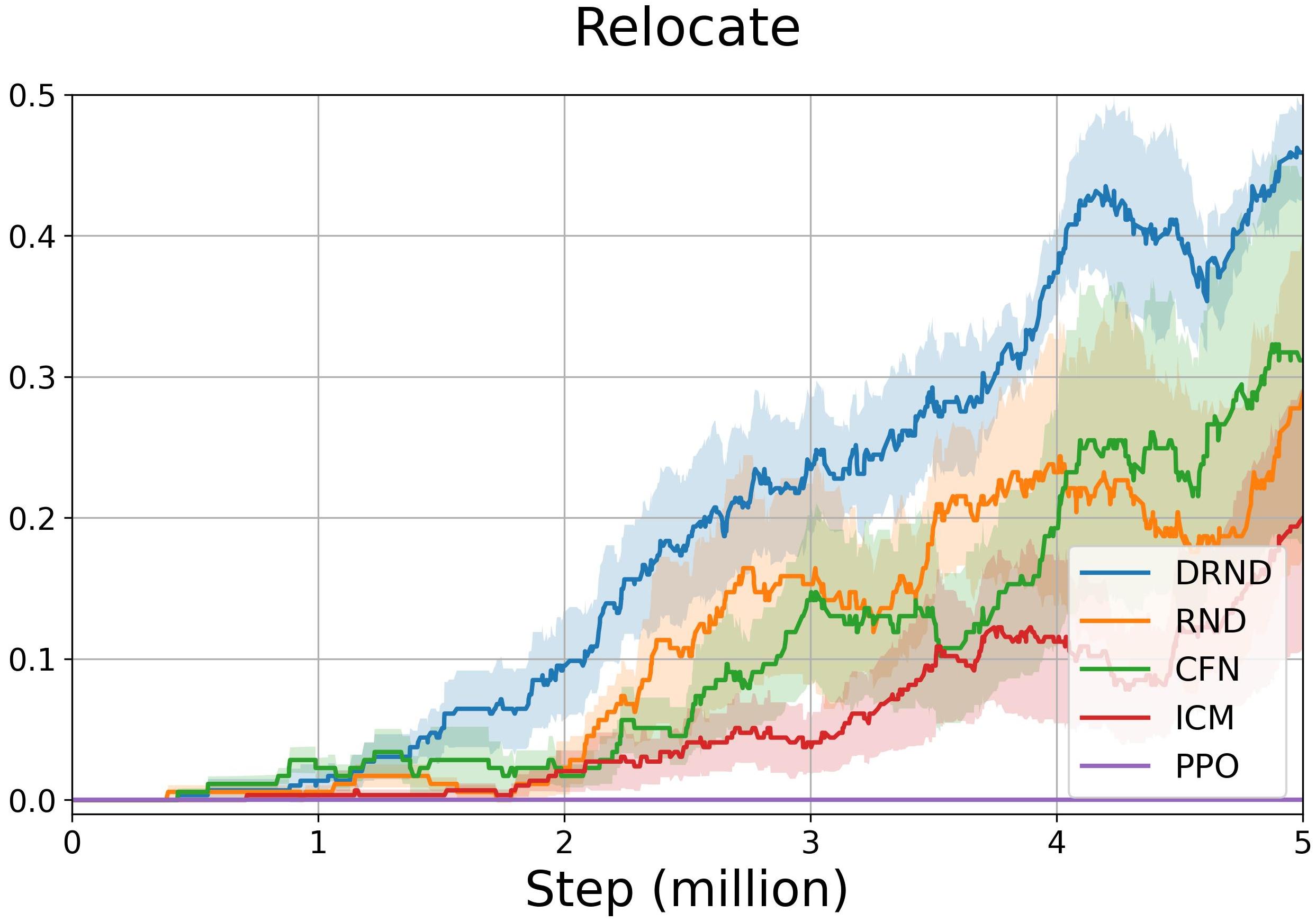}
    \end{minipage}
    }
    \caption{Learning curves in the Adroit continuous control tasks. All curves are averaged over 5 runs.
 }
    \label{adroit}
\end{figure*}

\begin{figure*}[h]
    \centering
    \subfigure{
    \begin{minipage}[]{0.24\linewidth}
        \centering
        \includegraphics[height=80pt,width=120pt]{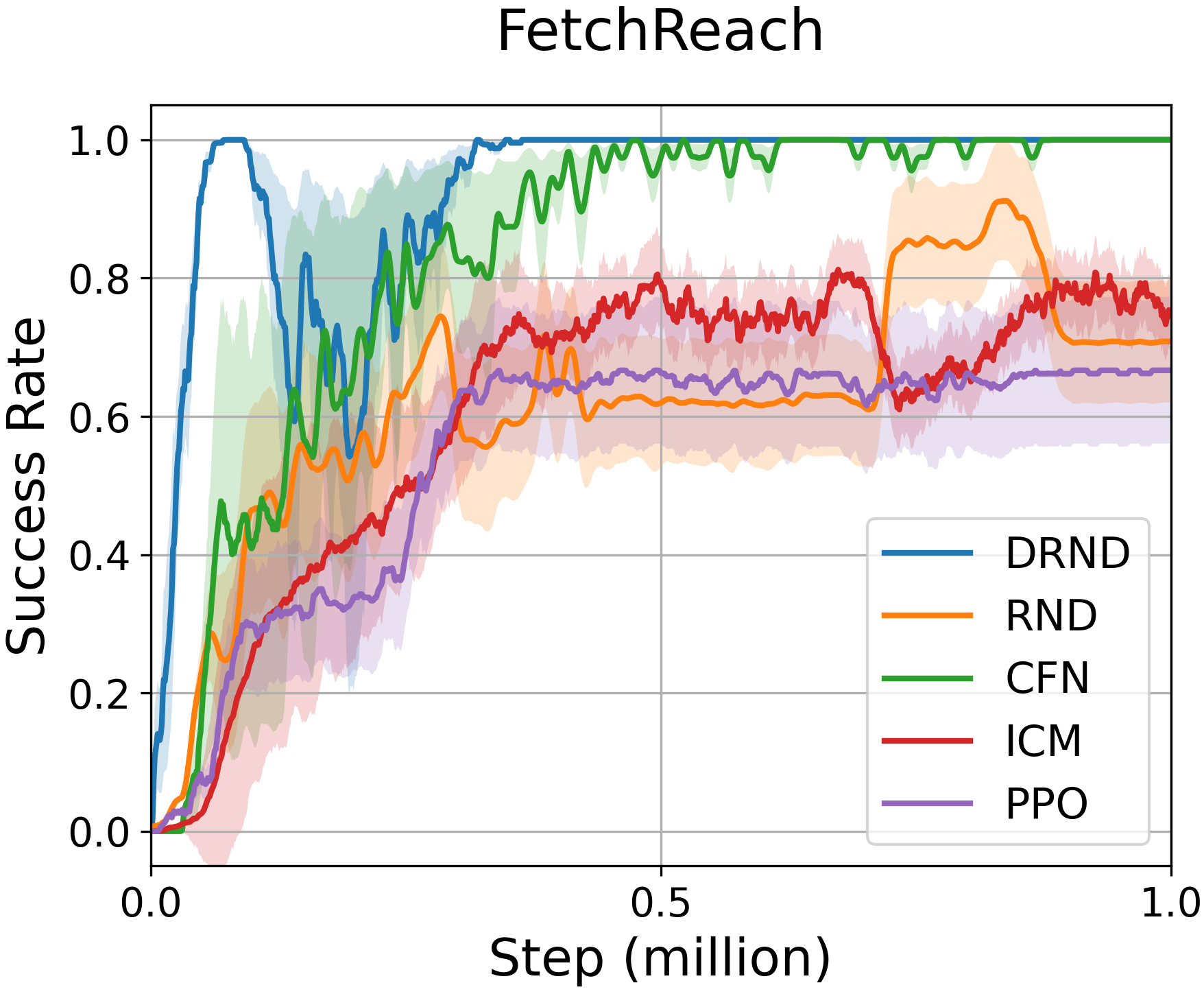}
    \end{minipage}
    }
    \subfigure{
    \begin{minipage}[]{0.23\linewidth}
        \centering
        \includegraphics[height=80pt,width=110pt]{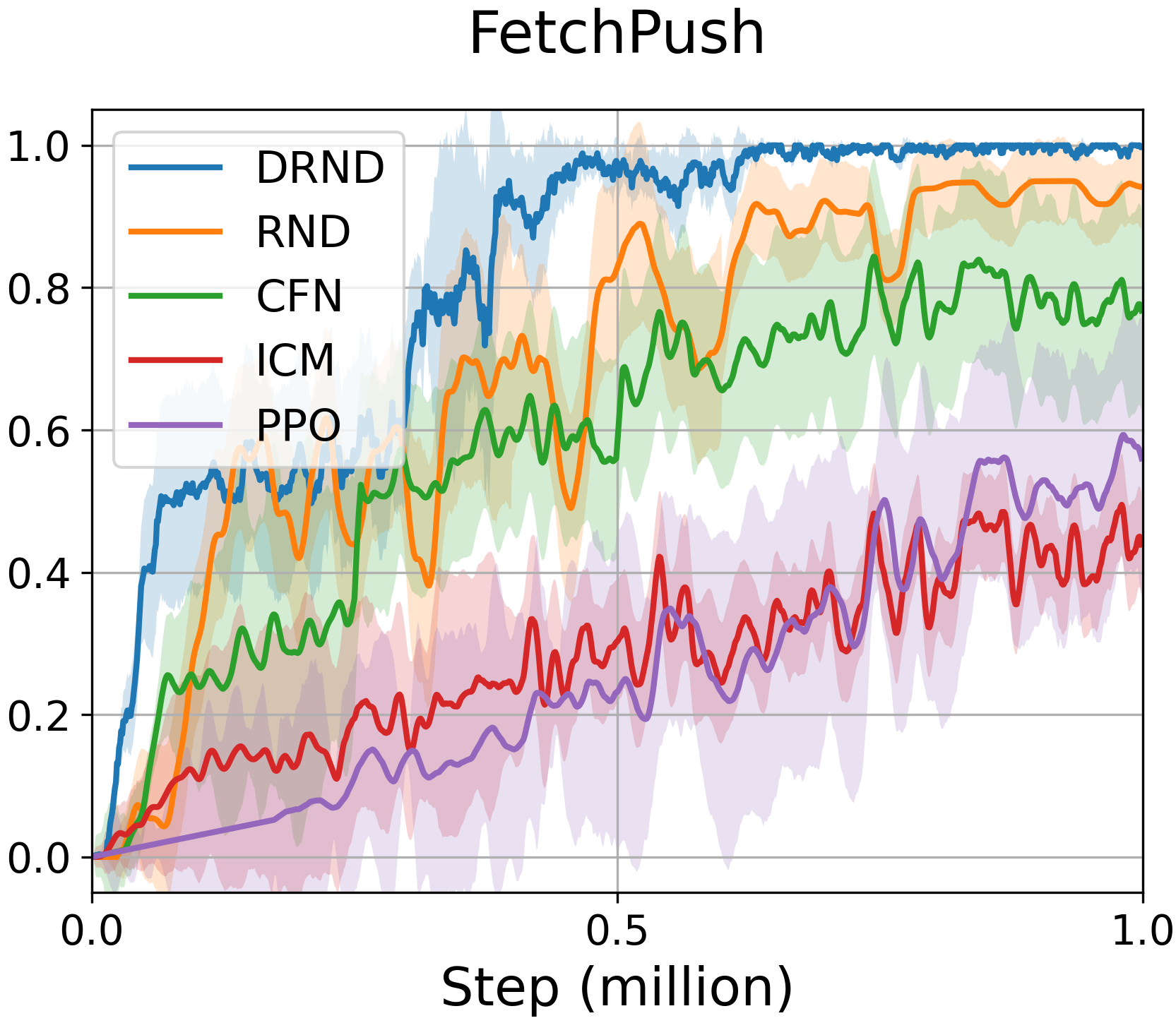}
    \end{minipage}
    }
    \subfigure{
    \begin{minipage}[]{0.225\linewidth}
        \centering
        \includegraphics[height=80pt,width=110pt]{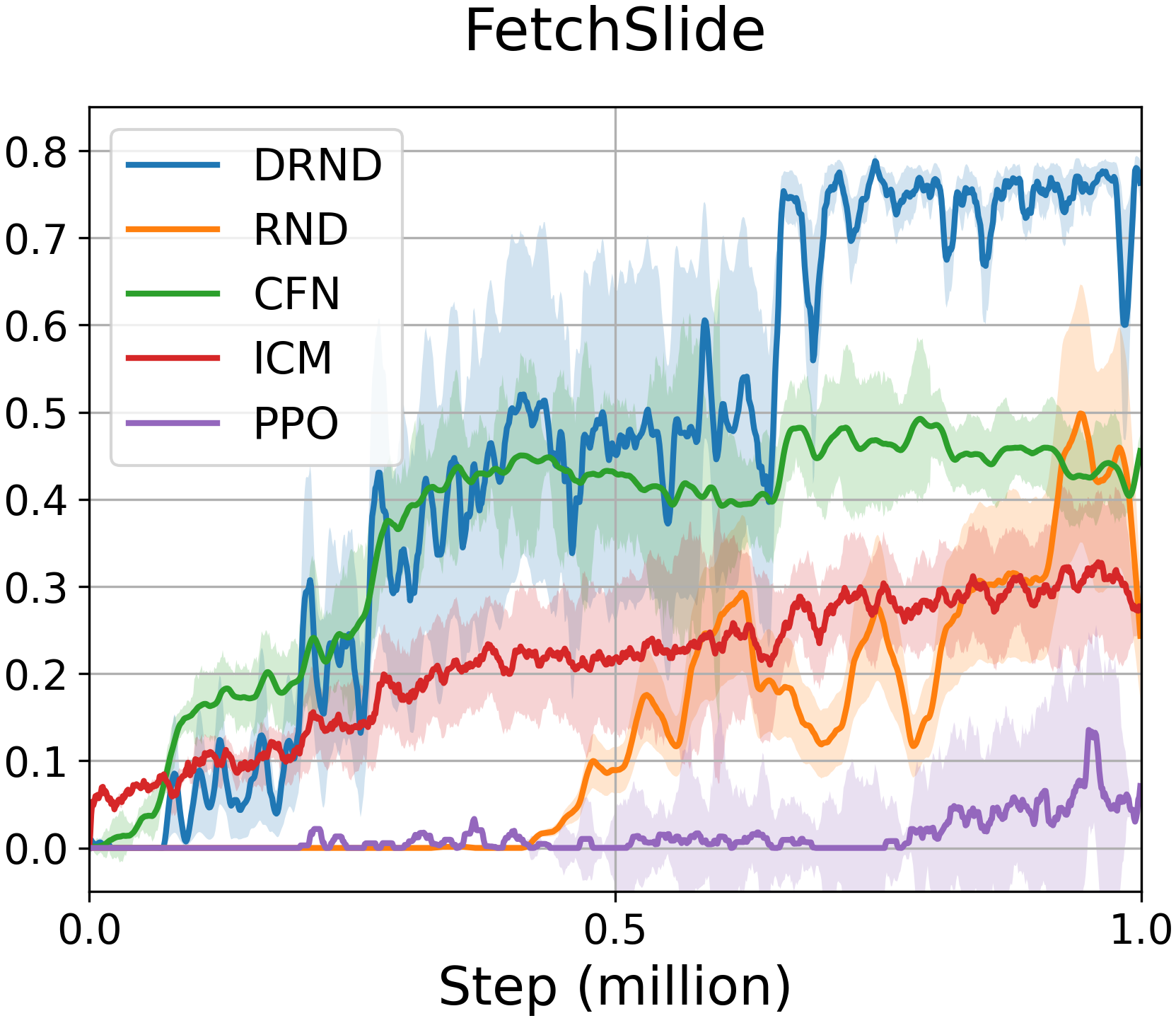}
    \end{minipage}
    }
    \subfigure{
    \begin{minipage}[]{0.23\linewidth}
        \centering
        \includegraphics[height=80pt,width=113pt]{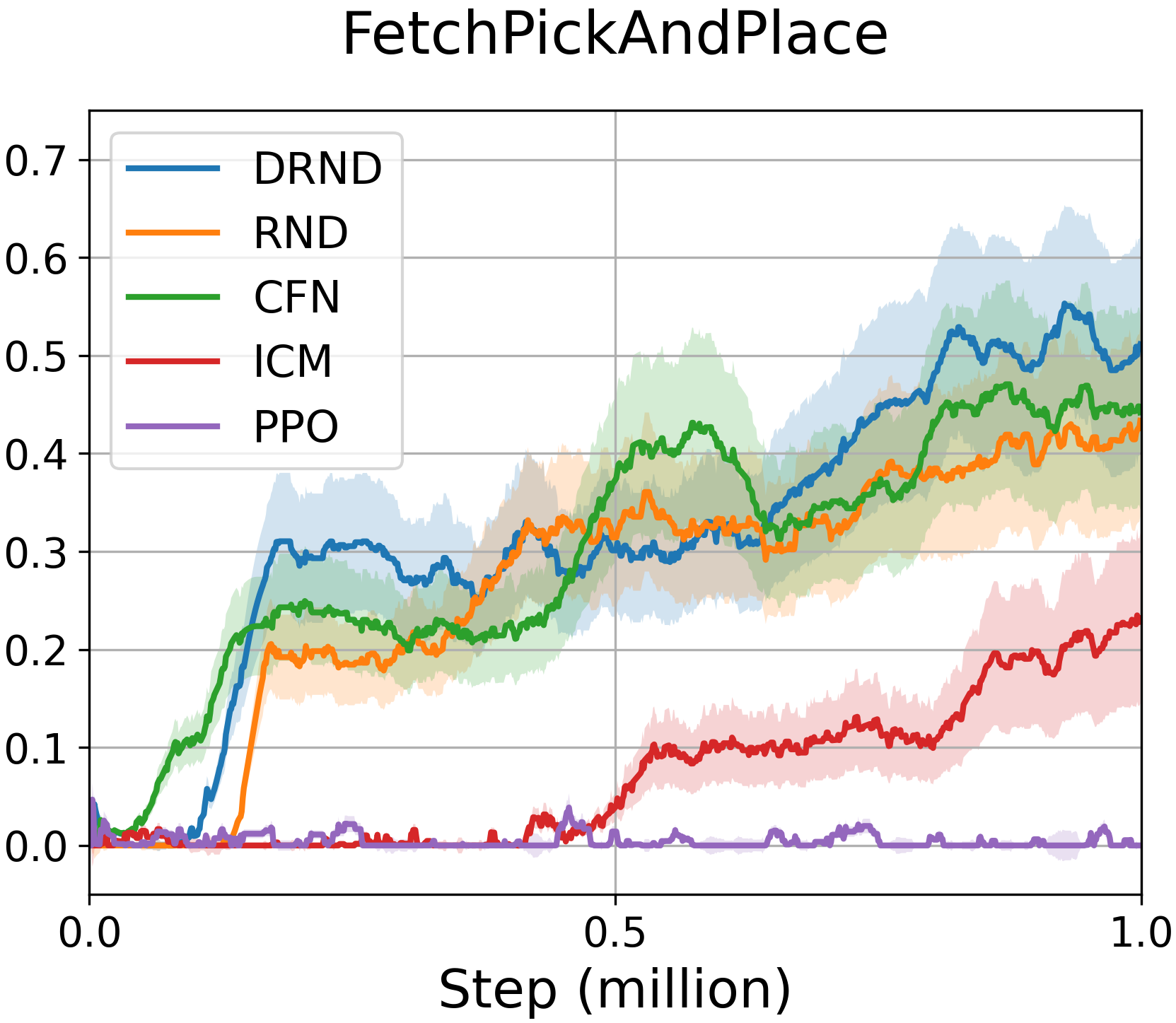}
    \end{minipage}%
    }
    \caption{Results on the Fetch manipulation tasks. All curves are averaged over 5 runs.
 }
    \label{fetch}
\end{figure*}

\section{Experiment}
In this section, we provide empirical evaluations of DRND. Initially, we demonstrate that DRND offers a better bonus than RND, both before and after training. Our online experiments reveal that DRND surpasses numerous baselines, achieving the best results in exploration-intensive environments. In the offline setting, we use DRND as an anti-exploration penalty term and propose the SAC-DRND algorithm, which beats strong baselines in many D4RL datasets.

\subsection{Bonus prediction comparison}\label{sec: inconsistency exp}

In this sub-section, we introduce our inconsistency experiments to compare bonus predictions for both RND and DRND. We created a mini-dataset resembling those used in offline RL or online RL replay buffers in the experiments. This small dataset contains $M$ data categories labeled from $1$ to $m$, with each data type occurring $i$ times proportional to its label. Each data point is represented as a one-hot vector with $M$ dimensions, where $M$ is set to 100. We train both the RND and DRND networks on the dataset and record both the initial intrinsic reward and the final intrinsic reward.

The left panel in Figure \ref{box} illustrates the difference in initial intrinsic rewards between RND and DRND, with the x-axis representing the number of target networks. As $N$ increases, the y-axis, representing the range of intrinsic rewards, becomes narrower, resulting in a more uniform distribution of rewards. In the right panel of Figure \ref{box}, we display the intrinsic reward distribution trained on the mini-dataset, showing that DRND's rewards have a stronger correlation with sample count than RND, as indicated by the regression lines.

\subsection{Performance on Online experiments}
Like many other exploration methods, we conduct our DRND approach in Atari games, Adroit environments \cite{rajeswaran2017learning}, and fetch manipulation tasks \cite{plappert2018multi}, which need deep exploration to get a high score. We integrate our method with the PPO \cite{schulman2017proximal} algorithm. We compare our approach with the RND method, the pseudo-count method CFN \cite{lobel2023flipping}, the curiosity-driven method ICM \cite{pathak2017curiosity}, and the baseline PPO method. The solid lines in the figures represent the mean of multiple experiments, and the shading represents the standard deviation interval.

\textbf{Atari Games.}
We chose three Atari games — Montezuma’s Revenge, Gravitar, and Venture — to evaluate our algorithms. These games require deep exploration to achieve high scores, making them ideal for assessing algorithmic exploratory capabilities. We benchmarked our method against the RND and PPO algorithms, with results presented in Figure \ref{atari}. Our DRND method converges faster and attains the highest final scores in these environments.

\textbf{Adroit Experiments.}
We further delve into the Adroit continuous control tasks. In these challenges, a robot must skillfully manipulate a hand to perform various actions, such as adjusting a pen's orientation or unlocking a door. Considering the complexity of the tasks and the robot's high-dimensional state space, it becomes imperative to explore methods that can facilitate the robot's learning. Figure \ref{adroit} illustrates that our DRND method outperforms all the other methods in exploration, especially in the challenging `Hammer' and `Relocate' environments. However, in the `Pen' environment, our method does not exhibit a significant improvement compared to other exploration algorithms. This could be attributed to the relatively simpler nature of this environment, which does not demand deep exploration.

\textbf{Fetch Manipulation Tasks.}
The Fetch manipulation tasks involve various gym-robotics environments, challenging the Fetch robot arm with complex tasks like reaching, pushing, sliding, and pick-and-place actions. Due to their complexity, these tasks demand advanced exploration strategies. Our evaluation of exploration algorithms in this context highlights their effectiveness in handling intricate robotic manipulations. As shown in Figure \ref{fetch}, our DRND approach excels in assisting the robot in these tasks. Our DRND method effectively combines the strengths of these approaches, outperforming results achievable with either pseudo-count or curiosity-driven methods alone. Consequently, our DRND algorithm performs significantly better than the RND method and other exploration algorithms.

\begin{table*}[t]
\tiny
    \centering
    \setlength{\tabcolsep}{0.5mm}{
    \resizebox{\linewidth}{42mm}{
        \begin{tabular}{l c c c c c c c}
        \toprule
         \textbf{Dataset} & \textbf{SAC} & \textbf{TD3+BC} & \textbf{CQL}& \textbf{IQL} & \textbf{SAC-RND} & \textbf{ReBRAC} & \textbf{SAC-DRND}\\
        \midrule
          hopper-random  &9.9 ± 1.5 &  8.5 ± 0.6  & 5.3 ± 0.6& 10.1 ± 5.9&  19.6 ± 12.4  &8.1 ± 2.4  
& \textbf{32.7} ± 0.4 \\
          hopper-medium &0.8 ± 0.0 & 59.3 ± 4.2 &61.9 ± 6.4 &65.2 ± 4.2& 91.1 ± 10.1& \textbf{102.0} ± 1.0 & 98.5 ± 1.1
 \\
          hopper-expert &0.7 ± 0.0 &107.8 ± 7.0 & 106.5 ± 9.1& 108.8 ± 3.1&\textbf{109.7} ± 0.5
 & 100.1 ± 8.3&\textbf{109.7} ± 0.3\\
         hopper-medium-expert &0.7 ± 0.0 &  98.0 ± 9.4 & 96.9 ± 15.1 &85.5 ± 29.7& \textbf{109.8} ± 0.6&107.0 ± 6.4& 108.7 ± 0.5\\
         hopper-medium-replay &7.4 ± 0.5 &  60.9 ± 18.8 & 86.3 ± 7.3 & 89.6 ± 13.2& 97.2 ± 9.0
&98.1 ± 5.3
&\textbf{100.5} ± 1.0\\
         hopper-full-replay &41.1 ± 17.9 & 97.9 ± 17.5  & 101.9 ± 0.6 & 104.4 ± 10.8 &107.4 ± 0.8
 &107.1 ± 0.4
&\textbf{108.2} ± 0.7\\
         \hline
         halfcheetah-random &29.7 ± 1.4 & 11.0 ± 1.1    & \textbf{31.1} ± 3.5  &19.5 ± 0.8&27.6 ± 2.1
&29.5 ± 1.5& 30.4 ± 4.0\\%300step 31.21078365342147 28.397803808573883 29.51497184943334 31.39021952048656 26.05324946931098 31.652101247917074 25.907170087150764 31.652224957936756 30.281297138733656 33.12797897374533
         halfcheetah-medium &55.2 ± 27.8 &  48.3 ± 0.3   & 46.9 ± 0.4 &50.0 ± 0.2&66.4 ± 1.4 &  65.6 ± 1.0& \textbf{68.3} ± 0.2\\ %1500step
         halfcheetah-expert &-0.8 ± 1.8  &  96.7 ± 1.1     & 97.3 ± 1.1 &95.5 ± 2.1& 102.6 ± 4.2
  &  105.9 ± 1.7& \textbf{106.2} ± 3.7 \\
  % halfcheetah expert record    |  step: 5999
  %actor beta: 5/critic beta: 4  |  99.17 ± 9.62
  %actor beta: 5/critic beta: 5  |  103.70 ± 4.74
  %actor beta: 5/critic beta: 6  |  106.24 ± 3.68
        halfcheetah-medium-expert &28.4 ± 19.4  &  90.7 ± 4.3      & 95.0 ± 1.4 &92.7 ± 2.8& 107.6 ± 2.8
  &  101.1 ± 5.2 & \textbf{108.5} ± 1.1 \\% 3000steps, 4seeds
        halfcheetah-medium-replay &0.8 ± 1.0  &  44.6 ± 0.5      & 45.3 ± 0.3 &42.1 ± 3.6&  51.2 ± 3.2 &51.0 ± 0.8 & \textbf{52.1} ± 4.8\\ %0.1 0.1 500steps 5seed
         halfcheetah-full-replay &\textbf{86.8} ± 1.0  &   75.0 ± 2.5    & 76.9 ± 0.9  & 75.0 ± 0.7 &  81.2 ± 1.3
 &  82.1 ± 1.1 & 81.4 ± 1.7\\% 3000 steps 5 seeds
         \hline
         walker2d-random & 0.9 ± 0.8 & 1.6 ± 1.7 & 5.1 ± 1.7 & 11.3 ± 7.0& 18.7 ± 6.9&18.1 ± 4.5  & \textbf{21.7} ±  0.1\\
         walker2d-medium & -0.3 ± 0.2 &  83.7 ± 2.1     &   79.5 ± 3.2&80.7 ± 3.4&91.6 ± 2.8& 82.5 ± 3.6& \textbf{95.2} ± 0.7\\%79.6&76.6& 80.2&79.5 
        walker2d-expert & 0.7 ± 0.3 & 110.2 ± 0.3 & 109.3 ± 0.1 & 96.9 ± 32.3& 104.5 ± 22.8  &112.3 ± 0.2 &\textbf{114.0}± 0.5\\ %4000
         walker2d-medium-expert& 1.9 ± 3.9 &  110.1 ± 0.5  & 109.1 ± 0.2&112.1 ± 0.5& 104.6 ± 11.2 &\textbf{111.6} ± 0.3 & 109.6 ± 1.0\\
         walker2d-medium-replay& -0.4 ± 0.3 & 81.8 ± 5.5   & 76.8 ± 10.0 &75.4 ± 9.3& 88.7 ± 7.7
 &77.3 ± 7.9& \textbf{91.0} ± 2.9\\
         walker2d-full-replay & 27.9 ± 47.3 & 90.3 ± 5.4  & 94.2 ± 1.9 &97.5 ± 1.4& 105.3 ± 3.2 & 102.2 ± 1.7& \textbf{109.6} ± 0.7\\
\hline
average score&16.2 & 67.5 & 73.6 & 72.9 & 82.6 & 81.2  & \textbf{86.0}\\
         \bottomrule
     	
        \end{tabular}
    }

    \vspace{5mm}  % Increase or decrease this value to adjust the space
    
     }
     \setlength{\tabcolsep}{0.8mm}{
    \resizebox{\linewidth}{21.5mm}{
    \begin{tabular}{l c c c c c cc}
\toprule
    \textbf{Dataset} & \textbf{SAC} & \textbf{TD3+BC} & \textbf{CQL}& \textbf{IQL} & \textbf{SAC-RND} & \textbf{ReBRAC}& \textbf{SAC-DRND}\\
    \midrule
         antmaze-umaze & 0.0 &78.6&74.0 &   83.3  ± 4.5    & 97.0 ± 1.5 
  &   \textbf{97.8} ± 1.0 & 95.8 ± 2.4\\
         antmaze-umaze-diverse & 0.0 &71.4&84.0 &   70.6 ± 3.7   &  66.0 ± 25.0 & \textbf{88.3} ± 13.0&87.2 ± 3.2\\ 
         \hline
         antmaze-medium-play& 0.0 & 10.6&61.2 &64.6 ± 4.9& 38.5 ± 29.4  & 84.0 ± 4.2& \textbf{86.2} ± 5.4\\ % step 2000
         antmaze-medium-diverse & 0.0 &3.0&53.7& 61.7 ± 6.1   & 74.7 ± 10.7
 &76.3 ± 13.5 &\textbf{83.0} ±3.8\\% step 2000
         \hline
         antmaze-large-play& 0.0 & 0.2&15.8& 42.5 ± 6.5 & 43.9 ± 29.2
  &   \textbf{60.4} ± 26.1&53.2 ± 4.1\\
         antmaze-large-diverse & 0.0 & 0.0&14.9& 27.6 ± 7.8 &  45.7 ± 28.5 & \textbf{54.4} ± 25.1 &50.8 ± 10.5\\
\hline
average score& 0.0 & 27.3 & 50.6 & 58.3 & 60.9 & \textbf{76.8} & 76.0 \\
\bottomrule
    \end{tabular}
    }
    }
   \caption{Average normalized scores of ensemble-free algorithms. The figure shows the scores at the final gradient step across 10 different random seeds. We evaluate 10 episodes for MuJoCo tasks and 100 episodes for AntMaze tasks. SAC and TD3+BC scores are taken from \cite{an2021uncertainty}. CQL, IQL, SAC-RND, and ReBRAC scores are taken from \cite{tarasov2023revisiting}. The highest score for each experiment is bolded.}
        \label{TABLE1}
\end{table*}
\vspace{0.15cm}

\subsection{D4RL Offline experiments}
We assessed our method using the D4RL \cite{fu2020d4rl} offline datasets, integrating the DRND approach with the SAC algorithm \cite{haarnoja2018soft}. Considering all available datasets in each domain, we tested SAC-DRND on Gym-MuJoCo and the more intricate AntMaze D4RL tasks. Our analysis compares against notable algorithms as detailed in \cite{rezaeifar2022offline}, including IQL \cite{kostrikov2021offline}, CQL \cite{kumar2020conservative}, and TD3+BC \cite{fujimoto2021minimalist}. It is worth noting that although our method also has $N$ target networks, they are fixed and not trained, making it ensemble-free. Our SAC-DRND is ensemble-free and only involves training double critics networks. We compare our methods against recent strong model-free offline RL algorithms in Table \ref{TABLE1}. Additionally, we compare SAC-DRND against strong ensemble-based algorithms like SAC-N in Appendix \ref{ensemble}. Only the results of the ensemble-free methods are shown in the main text. The results are evaluated at the final gradient step over 10 different seeds.

It can be seen that SAC-DRND excels in the majority of MuJoCo tasks, attaining the best results among all ensemble-free methods. On Antmaze tasks, DRND also reached a level similar to SOTA. Compared to SAC-RND, which has comparable computational and storage requirements as our approach, SAC-DRND more effectively captures the dataset distribution, as reflected in its superior average scores and decreased variance. We also conducted experiments on Adroit tasks (Appendix \ref{adroit_exp}), hyperparameters sensitivity experiments (Appendix \ref{eop}) using Expected Online Performance (EOP, \cite{kurenkov2022showing}) and offline-to-online experiments (Appendix \ref{offline2online}).

\section{Conclusion}
Our research highlights the ``bonus inconsistency'' issue inherent in RND, which hinders its capacity for deep exploration. We introduce DRND, which distills a random target from a random distribution. Our approach efficiently records state-action occurrences without substantial time and space overhead by utilizing specially designed statistics to extract pseudo-counts. Theoretical analysis and empirical results show our method's effectiveness in tackling bonus inconsistency. We observe promising results across Atari games, gym-robotics tasks, and offline D4RL datasets.

\section*{Acknowledgements}
This work was supported by the STI 2030-Major Projects under Grant 2021ZD0201404. 

\section*{Impact Statement}
This paper presents work whose goal is to advance the field of Machine Learning. There are many potential societal consequences of our work, none which we feel must be specifically highlighted here.
% In the unusual situation where you want a paper to appear in the
% references without citing it in the main text, use \nocite
% \nocite{langley00}

\bibliography{example_paper}
\bibliographystyle{icml2024}
%%%%%%%%%%%%%%%%%%%%%%%%%%%%%%%%%%%%%%%%%%%%%%%%%%%%%%%%%%%%%%%%%%%%%%%%%%%%%%%
%%%%%%%%%%%%%%%%%%%%%%%%%%%%%%%%%%%%%%%%%%%%%%%%%%%%%%%%%%%%%%%%%%%%%%%%%%%%%%%
% APPENDIX
\newpage
\appendix
\onecolumn

\section{Proof}
In this section, we will provide all the proofs in the main text.
\subsection{proof of lemma \ref{lemma1}}
\label{proof1}
\begin{align*}
\mathbb{E}\left[\|f_{\tilde{\theta}}(x) - \frac{1}{N} \sum_{i = 1}^N f_{\bar{\theta}_i}(x)\|^2\right] &= \mathbb{E}\left[\|\tilde \theta^T x-\frac{\sum_{i=1}^N \bar \theta_i^T x}{N}\|^2\right] \\
&= Var\left(\tilde \theta^T x-\frac{\sum_{i=1}^N \bar \theta_i^T x}{N}\right) - \left(\mathbb{E}\left[\tilde \theta^T x-\frac{\sum_{i=1}^N \bar \theta_i^T x}{N}\right]\right)^2 \\
&=Var\left(\tilde \theta^T x-\frac{\sum_{i=1}^N \bar \theta_i^T x}{N}\right)- \left(\mathbb{E}\left[ (\tilde \theta-\frac{1}{N}\sum_{i=1}^N \bar \theta_i)^T x\right]\right)^2. \\
\end{align*}

Since $\tilde \theta$ and $\bar \theta_i$  $(i=1,2...,N)$ are i.i.d., $\mathbb{E}\left[ (\tilde \theta-\frac{1}{N}\sum_{i=1}^N \bar \theta_i)^T x\right] = \mathbb{E}[(\tilde \theta-\frac{1}{N}\sum_{i=1}^N \bar \theta_i)]x = 0$.
So we have:
\begin{align*}
\mathbb{E}[\|f_{\tilde{\theta}}(x) &- \frac{1}{N} \sum_{i = 1}^N f_{\bar{\theta}_i}(x)\|^2] \\
=& Var\left(\tilde \theta^T x-\frac{\sum_{i=1}^N \bar \theta_i^T x}{N}\right) - 0 \\
=& Var(\tilde \theta^T x) + \frac{1}{N^2}\sum_{i=1}^N Var(\bar \theta_i^T x) - \frac{2}{N} \text{Cov}(\tilde \theta^T x, \theta_i^T x) \\
=& Var(\tilde \theta^T x) + \frac{1}{N^2}\sum_{i=1}^N Var(\bar \theta_i^T x) \hspace{1em} (\text{Cov}(x,y) = 0 \text{ if $x$ and $y$ are i.i.d}) \\
=& x^T \Sigma x + \frac{1}{N} x^T \Sigma x \\
=& (1+\frac{1}{N}) x^T \Sigma x.
\end{align*}

\subsection{proof pf lemma \ref{lemma3}}
\label{proof}
For simplicity, we use some symbols to record the moments of the distribution of $c(x)$:
\begin{align*}
&\mu(x) = \mathbb{E}[X] = \frac{1}{N} \sum_{i=1}^{N}\bar{f}_i(x), \quad\quad\quad
B_2(x) = \mathbb{E}[X^2] = \frac{1}{N} \sum_{i=1}^{N}(\bar{f}_i(x))^2,\\
&B_3(x) = \mathbb{E}[X^3] = \frac{1}{N} \sum_{i=1}^{N}(\bar{f}_i(x))^3,\quad\quad
B_4(x) = \mathbb{E}[X^4] = \frac{1}{N} \sum_{i=1}^{N}(\bar{f}_i(x))^4.
\end{align*}
The calculation of $f^*(x)$ moment is as follows:
$$
\mathbb{E}[f_*(x)] = \mathbb{E}[\frac{1}{n} \sum_{i=1}^{n}c_i(x)] = \frac{1}{n}\mathbb{E}[\sum_{i=1}^{n}c_i(x)] = \mu(x).
$$
\begin{align*}
\mathbb{E}[f_*^2(x)] &= \mathbb{E}[(\frac{1}{n} \sum_{i=1}^{n}c_i(x))^2] \\
&= \frac{1}{n^2}\mathbb{E}[(\sum_{i=1}^{n}c^2_i(x)+\sum_{i=1}^{n}\sum_{j\ne i}^{n}c_i(x)c_j(x))] \\
&= \frac{1}{n^2}\mathbb{E}[nc^2(x) + n(n-1)\mu^2(x)]\\
&= \frac{B_2(x)}{n}+ \frac{n-1}{n}\mu^2(x).
\end{align*}
\begin{align*}
\mathbb{E} [f^4_*(x)] &= \frac{1}{n^{4}} \mathbb{E}\left[\sum_{i = 1}^{n} c_{i}(x)\right]^{4} \\
&= \frac{1}{n^{4}}\left(\mathbb{E}\left[\sum_{i = 1} c_{i}(x)^{4}\right]+\right. \quad 4 \mathbb{E}\left[\sum_{i \neq j} c_{i}^{3}(x) c_{j}(x)\right]+\quad 3 \mathbb{E}\left[\sum_{i \neq j} c^{2}_{i}(x) c^{2}_{j}(x)\right]\\
&\left.\quad+6E\left[\sum_{i \neq j \neq k} c_{i}(x) c_{j}(x) c_{k}^{2}(x)\right] +\quad \mathbb{E}\left[\sum_{i \neq j \neq k \neq l} c_{i}(x) c_{j}(x) c_{k}(x) c_{l}(x)\right]\right)  \\
&= \frac{n B_4(x)+4 A_{n}^{2} \mu(x) B_3(x)+3 A_{n}^{2}B_2^2(x)+6A_{n}^{3} \mu^{2}(x) B_2(x)+A_{n}^{4} \mu^{4}(x)}{n^{4}}.\\
&(A_n^i=\frac{n!}{(n-i)!} )
\end{align*}
The statistic $y(x)$ is:
$$
y(x) = \frac{f_*^2(x) - \mu^2(x)}{B_2(x) - \mu^2(x)},
$$
and its expectation is:
$$
\mathbb{E}[y(x)] = \frac{\mathbb{E}[f_*^2(x)] - \mu^2(x)}{B_2(x) - \mu^2(x)} = \frac{1}{n}.
$$
This indicates that the statistic $y(x)$ is an unbiased estimator for the reciprocal of the frequency of $x$.
The variance of $y(x)$ is:
\begin{align*}
Var[y(x)] & = \frac{Var[f_*^2(x)]}{(B_2(x) - \mu^2(x))^2}\\ & = \frac{\mathbb{E}[f^4_*(x)]-\mathbb{E}^2[f^2_*(x)]}{(B_2(x) - \mu^2(x))^2}\\ & = \frac{K_1B_4(x)+K_2\mu(x)B_3(x)+K_3B^2_2(x)+K_4\mu^2(x)B_2(x)+K_5\mu^4(x)}{n^3(B_2(x) - \mu^2(x))^2}
\end{align*}
where
\begin{align*}
&K_1 = 1, \quad K_2 = 4n-4, \quad K_3 = 2n-3,\\
&K_4 = 4n^2-16n+12, \quad K_5 = -5n^2+10n-6.
\end{align*}
so we have:
$$
\lim_{n \to \infty} Var[y(x)] = 0.
$$
When $n$ tends to infinity, the variance of the statistic tends to zero, which reflects the stability or consistency of $y(x)$.
\newpage
\section{DRND Pseudo-code}
\renewcommand{\thealgorithm}{1}
    \begin{algorithm}
        \caption{PPO-DRND online pseudo-code}
        \label{online}
        \begin{algorithmic}[1]
            \REQUIRE Number of training steps $M$, number of update steps $K$, number of target networks $N$, scale of intrinsic reward $\lambda$
            \STATE Initialize policy parameters $\phi$
            \STATE Initialize Q-function parameters $\varphi$ and target Q-function parameters $\varphi'$
            \STATE Initialize predictor network parameters $\theta$ and target networks parameters $\theta_1,\theta_2,...,\theta_N$
            \FOR{$i=1:N$}
                \STATE Initialize replay buffer $D$
                \STATE $d \gets 0$, $t \gets 0$
                \STATE $s_0 = $ env.reset()
                \WHILE{not $d$}
                    \STATE $a_t \sim \pi(a_t|s_t)$
                    \STATE Rollout $a_t$ and get $(s_{t+1},r_t,d)$
                    \STATE Compute the mean $\mu(s_t,a_t)$ and second moment $B_2(s_t,a_t)$
                    \STATE Compute intrinsic reward $b(s_{t+1},a_t)$ using \cref{reward}
                    \STATE Add transition $(s_t,a_t,r_t,b(s_{t+1},a_t),s_{t+1})$ to $D$
                    \STATE $t \gets t+1$
                \ENDWHILE
                \STATE Normalize the intrinsic rewards contained in $D$
                \STATE Calculate returns $R_I$ and advantages $A_I$ for intrinsic reward
                \STATE Calculate returns $R_E$ and advantages $R_E$ for extrinsic reward
                \STATE Calculate combined advantages $A$ = $R_I$ + $R_E$
                \STATE $\phi_{old} \gets \phi$
                \FOR{$j=1:K$}
                   
                    \STATE Update $\phi$ with gradient ascent using 
                    \[
\scalebox{0.9}{
    $\nabla_\phi \frac{1}{|D|} \sum_{D}^{} \min \left(\frac{\pi_{\phi}(a \mid s)}{\pi_{\phi_{\text{old}}}(a \mid s)} A, \operatorname{clip}\left(\frac{\pi_{\phi}(a \mid s)}{\pi_{\phi_{\text{old}}}(a \mid s)}, 1-\epsilon, 1+\epsilon\right) A\right)$
}
\]
                    \STATE Update $\varphi$ with gradient descent using 
\[
\scalebox{0.9}{
    $\nabla_\varphi \frac{1}{|D|} \sum_{D}^{} [Q_{\varphi} - {r}_{t}+\lambda b_\theta\left(s_t,a_t\right)+\gamma \max _{a^\prime} Q_{\varphi^{\prime}}\left(s_{t+1}, a^\prime\right)  ]$
}
\]
                    \STATE Update $\theta$ using Equation \cref{loss}
                \ENDFOR
                \STATE Update target networks with $\varphi' = (1-\rho)\varphi' + \rho\varphi$
            \ENDFOR
        \end{algorithmic}
    \end{algorithm}

\renewcommand{\thealgorithm}{2}
    \begin{algorithm}
        \caption{SAC-DRND offline pseudo-code}
        \label{online}
        \begin{algorithmic}[1]
            \REQUIRE Number of training steps $M$, number of DRND update steps $K$, number of target networks $N$, scale of intrinsic reward $\lambda_\text{actor},\lambda_\text{critic}$, dataset buffer $D$
            \STATE Initialize policy parameters $\phi$
            \STATE Initialize two Q-function parameters $\varphi_1,\varphi_2$ and target Q-function parameters $\varphi_1^\prime,\varphi_2^\prime$
            \STATE Initialize predictor network parameters $\theta$ and target networks parameters $\theta_1,\theta_2,...,\theta_N$
            \FOR{$i=1:K$}
                    \STATE Sample minibatch $(s,a,r,b,s') \sim D$
                    \STATE Compute the mean $\mu(s_t,a_t)$ and second moment $B_2(s_t,a_t)$
                    \STATE Update $\theta$ using Equation \cref{loss}
            \ENDFOR
            \FOR{$j=1:N$}
                    \STATE Sample minibatch $(s,a,r,b,s') \sim D$
                    \STATE Update $\phi$ with gradient ascent using 
\[
\scalebox{0.9}{
    $\nabla_\phi \frac{1}{|B|} \sum_{B}^{} \left[\min_{i=1,2} Q_{\varphi_i}\left(s, \tilde{a}_{\phi}(s)\right)-\beta \log \pi\left(\tilde{a}_{\phi}(s) \mid s\right)-\lambda_\text{actor} b_\theta\left(s, \tilde{a}_{\phi}(s)\right)\right]$
}
\]\indent where $\tilde{a}_{\phi}(s)$ is a sample from $\pi_\phi(.|s)$ by using reparametrization trick
                    \STATE Update each Q-function $Q_{\varphi_i}$ with gradient descent using 
\[
\scalebox{0.9}{
    $\nabla_{\varphi_i} \frac{1}{|B|} \sum_{B}^{} [Q_{\varphi_i} - {r}_{t}+\gamma \mathbb{E}_{a^\prime \sim \pi\left(\cdot \mid s_{t+1}\right)}\left[Q_{\varphi_i^{\prime}}\left(s_{t+1}, a^\prime\right)-\beta\log\pi_\phi(a^\prime|s^\prime)-\lambda_\text{critic} b_\theta\left(s_{t+1}, a^\prime\right)\right]  ]$
}
\]\indent  where $a^\prime \sim \pi_\phi(.|s^\prime)$
            \STATE Update target networks with $\varphi_i^\prime = (1-\rho)\varphi_i^\prime + \rho\varphi_i$
            \ENDFOR
        \end{algorithmic}
    \end{algorithm}

\section{Implementation Details and Experimental Settings}\label{app_exp_detail}
Our experiments were performed by using the following hardware and software:
\begin{itemize}
    \item GPUs: NVIDIA GeForce RTX 3090
    \item Python 3.10.8
    \item Numpy 1.23.4
    \item Gymnasium 0.28.1
    \item Gymnasium-robotics 1.2.2
    \item Pytorch 1.13.0
    \item MuJoCo-py 2.1.2.14
    \item MuJoCo 2.3.1
    \item D4RL 1.1
    \item Jax 0.4.13
\end{itemize}

The architecture of predictor networks in Figure 1 and Figure 3 is 3 linear layers, the input dim is 2, hidden dim and output dim is 16, activate function is ReLU. Target networks' architecture is 2 linear layers, the input dim is 2, hidden dim and output dim is 16, activate function is ReLU. 
We set the state as 2-dimensional, with the state space defined as the 2-dimensional space [0,1]x[0,1]. The given distribution is regarded as the dataset distribution of the current agent passing through states, and the dimensions can be arbitrary. Setting it as 2-dimensional is merely for visualization convenience.

In our experiments, for a fair comparison, all methods employed the same predictor and target networks. The fundamental parameters of the base algorithm such as learning rate and batch size were kept identical across all methods. For the hyperparameters of the utilized exploration algorithms, we utilized the author-recommended hyperparameters from respective papers (e.g., CFN). Since the CFN method was originally proposed for off-policy strategies, it utilized a trick of importance sampling for sampling from the replay buffer. However, in our PPO-based approach, there is no replay buffer, and we consider the use of the importance sampling trick unfair compared to other methods. Therefore, we only employed the core formula from the paper $b(s)=\sqrt{\frac{1}{d}\|f_\theta(s)\|}$ as the intrinsic reward, where $d$ represents the output dimension of the predictor network.

In our D4RL experiments, all experiments use the dataset of the `v2' version. We use specific hyperparameters for each task due to varying anti-exploration penalties. Because most of the offline experiment time is spent on gradient calculation of data, we use the faster Jax framework \cite{bradbury2018jax} than the Pytorch framework \cite{paszke2019pytorch} for the experiments. In the online experiments, we still use the easier-to-read and more portable Pytorch framework instead of the faster computing Jax framework because most of the online experiment time is spent interacting with the environment rather than gradient computing.

We employ the `NoFrameskip-v4' version in our Atari game experiments to execute the environments. These experiments encompass 128 parallel environments and adhere to the default configurations and network architecture as delineated in \cite{burda2018exploration}. For Adroit and Fetch manipulation tasks, we employ the `v0' version for Adroit tasks and the `v2' version for Fetch tasks. In the `Relocate' task, we truncate the episode when the ball leaves the table. These tasks pose a significant challenge for conventional methods to learn from, primarily due to the dataset consisting of limited human demonstrations in a sparse-reward, complex, high-dimensional robotic manipulation task \cite{lyu2022doublecheck}. We do not include random state restarts, as they may undermine the necessity for exploration by the observations made by \cite{lobel2022optimistic}. To set the goal locations for the non-default versions of the tasks, we follow the setting of \cite{lobel2023flipping}.

In the context of the D4RL framework, we make specific architectural choices. Instead of simply concatenating the state and action dimensions, we employ a bilinear structure in the first layer, as proposed by \cite{jayakumar2020multiplicative}. Additionally, we apply FiLM (Feature-wise Linear Modulation) architecture on the penultimate layer before the nonlinearity. This modification is effective for offline tasks, as indicated by \cite{nikulin2023anti}.

\section{Hyperparameters}
The hyperparameters are shown in Table \ref{paramonline} in online experiments. We employ distinct parameters and networks for Atari games and continuous control environments because Atari game observations are images, while observations for Adroit and Fetch tasks consist of states. The hyperparameters we use in the D4RL offline experiment are shown in Table \ref{param}. In D4RL offline datasets, we apply varying scales in each experiment due to the differing dataset qualities, as illustrated in Table \ref{hyper}. 
\begin{table}[h]
    \centering
\caption{Hyperparameters of D4RL offline experiments}
\renewcommand{\arraystretch}{1.1}
\label{param}
\resizebox{0.8\linewidth}{!}{
    \begin{tabular}{l |l |l}
        \toprule
         \textbf{Name} & \textbf{Description}&\textbf{Value} \\
        \midrule
          $lr_\text{actor}$ & learning rate of the actor network & 1e-3 (1e-4 on Antmaze)\\ 
          $lr_\text{critic}$& learning rate of the critic network& 1e-3 (1e-4 on Antmaze)\\
          $lr_\text{drnd}$ & learning rate of the DRND network & 1e-6 (1e-5 on Antmaze) \\
         optimizer&   type of optimizer & Adam\\
         target entropy & target entropy of the actor & -action\_dim\\
         $\tau$ & soft update rate & 0.005\\ 
         $\gamma$ &   discount return & 0.99 (0.999 on Antmaze)\\
          $bs$ & batch size of the dataset & 1024\\
        $h$ & number of hidden layer dimensions& 256 \\
        $e$ &  number of DRND output dimensions & 32\\
        $n$ & number of hidden layers& 4\\
        $f$ & activation function & ReLU\\
        $K$ & number of DRND training epochs & 100\\
         $M$ & maximum iteration number of SAC & 3000\\
         $I$ & gradient updates per iteration & 1000 \\
         $N$ & number of DRND target networks & 10 \\
          $\alpha$ & the scale of two intrinsic reward items & 0.9 \\
          
        \bottomrule 
        \end{tabular}
        }
\end{table}
\vspace{0.25cm}

\begin{table}[h]
    \centering
\caption{Anti-exploration scale of D4RL offline datasets}
\label{hyper}
\renewcommand{\arraystretch}{1.02}
\resizebox{0.52\linewidth}{!}{
        \begin{tabular}{l| l l}
        \toprule
      \textbf{Dataset Name}  & $\bm{\lambda_\text{actor}}$ & $\bm{\lambda_\text{critic}}$\\
        \hline
          hopper-random  & 1.0 & 1.0\\ 
           hopper-medium & 15.0& 15.0\\
           hopper-expert & 10.0& 10.0\\
           hopper-medium-expert& 10.0& 10.0 \\
           hopper-medium-replay& 5.0& 10.0 \\
           hopper-full-replay& 2.0& 2.0\\
            \hline
            halfcheetah-random & 0.05& 0.05\\
            % halfcheetah-medium & 0.8 & 0.1 \\
            halfcheetah-medium & 1.0 & 0.1 \\
            % halfcheetah-expert & 5.0& 6.0 \\
            halfcheetah-expert & 5.0& 5.0 \\
            halfcheetah-medium-expert & 0.1 & 0.1 \\
            halfcheetah-medium-replay & 0.1 & 0.1\\
            halfcheetah-full-replay & 1.0 & 1.0 \\
            \hline
            walker2d-random & 1.0& 1.0\\
            walker2d-medium & 10.0 & 10.0 \\
            walker2d-expert & 5.0& 5.0 \\
            walker2d-medium-expert & 15.0 & 20.0 \\
            walker2d-medium-replay &5.0 & 10.0\\
            walker2d-full-replay & 3.0& 3.0\\
            \hline
            antmaze-umaze & 5.0 & 0.001 \\
            antmaze-umaze-diverse & 3.0 & 0.001\\
            antmaze-medium-play & 3.0 & 0.005\\ %step=2000
            antmaze-medium-diverse & 2.0 &0.001 \\ % step=2000
            antmaze-large-play & 1.0 & 0.01 \\
            antmaze-large-diverse & 0.5 & 0.01\\    
         \bottomrule 
        \end{tabular}
        }
\end{table}
\vspace{0.25cm}

\begin{table}[h]
    \centering
\caption{Hyperparameters of online experiments}
\label{paramonline}
\resizebox{0.7\linewidth}{!}{
    \begin{tabular}{l |l |l}
        \toprule
         \textbf{Name} & \textbf{Description}&\textbf{Value} \\
        \midrule
          $lr_\text{actor}$ & learning rate of the actor network & 3e-4 (1e-4 on Atari)\\ 
          $lr_\text{critic}$& learning rate of the critic network& 3e-4 (1e-4 on Atari)\\
          $lr_\text{drnd}$ & learning rate of the DRND network & 3e-4 (1e-4 on Atari)\\
         optimizer&   type of optimizer & Adam\\
         $\tau$ & soft update rate & 0.005\\ 
         $\gamma$ &   discount return & 0.99\\
         $\lambda_\text{GAE}$ & coefficient of GAE & 0.95\\
          $\epsilon$ & PPO clip coefficient & 0.1\\
          $M$ & number of environments & 128 \\
        $h$ & number of hidden layer dimensions& 64 (512 on Atari)\\
        $e$ &  number of output dimensions & 64 (512 on Atari)\\
        $f$ & activation function & ReLU\\
        $K$ & number of training epochs & 4\\
         $N$ & number of DRND target networks & 10 \\
         $\lambda$ & coefficient of intrinsic reward & 1\\
          $\alpha$ & the scale of two intrinsic reward items & 0.9 \\
          
        \bottomrule 
        \end{tabular}
        }
\end{table}

\section{Additional Experimental Results}
% \subsection{Performance Across Different Task Difficulties}
% In Adroit and Fetch environments, We have set a fixed starting point for training. Based on the distance between the starting and ending positions, following the setting in (\cite{lobel2023flipping}), we divide the task difficulty into three levels (easy, medium, and hard) based on different goals and starting points. We run these tasks with different difficulties in Fetch environments and the results are shown in Figure \ref{different}.

\subsection{Comparing to Ensemble-based Methods}
\label{ensemble}
As described in \cite{osband2016deep}, the ensemble method estimates the Q-posterior, leading to varied predictions and imposing significant penalties in regions with limited data. We add the results of ensemble-based methods like SAC-N \cite{an2021uncertainty}, EDAC \cite{an2021uncertainty}, and RORL \cite{yang2022rorl}. Table \ref{TABLEa} displays our results in these experiments. An \underline{underlined} number represents the peak value for ensemble-free methods, while a \textbf{bold} number denotes each task's top score. SAC-DRND outperforms most ensemble-based methods, such as SAC-N and RORL, in total scores on most MuJoCo tasks. For Antmaze tasks, our method leads among ensemble-free approaches and holds its own against ensemble-based methods.

\begin{table}[h]
\tiny
    \centering
    \setlength{\tabcolsep}{0.4mm}{
    \resizebox{\linewidth}{37mm}{
        \begin{tabular}{l| c c c c c c| c c c| c}
        \multicolumn{1}{c}{}
        & \multicolumn{6}{c}{\fontsize{8}{10}\selectfont Ensemble-free}
        & \multicolumn{3}{c}{\fontsize{8}{10}\selectfont Ensemble-based}
        \vspace{1mm}
        &{\fontsize{8}{10}\selectfont Ours} \\
        \toprule
         \textbf{Dataset} & \textbf{SAC} & \textbf{TD3+BC} & \textbf{CQL}& \textbf{IQL} & \textbf{SAC-RND} & \textbf{ReBRAC} & \textbf{SAC-N} & \textbf{EDAC}& \textbf{RORL}& \textbf{SAC-DRND}\\
        \midrule
          hopper-random  &9.9 ± 1.5 &  8.5 ± 0.6  & 5.3 ± 0.6& 10.1 ± 5.9&  19.6 ± 12.4  &8.1 ± 2.4 & 28.0 ± 0.9& 25.3 ± 10.4&31.4 ± 0.1  
& \textbf{32.7} ± 0.4 \\
          hopper-medium &0.8 ± 0.0 & 59.3 ± 4.2 &61.9 ± 6.4 &65.2 ± 4.2& 91.1 ± 10.1& \underline{102.0} ± 1.0
&100.3 ± 0.3 & 101.6 ± 0.6& \textbf{104.8} ± 0.1 & 98.5 ± 1.1
 \\
          hopper-expert &0.7 ± 0.0 &107.8 ± 7.0 & 106.5 ± 9.1& 108.8 ± 3.1&\underline{109.7} ± 0.5
 & 100.1 ± 8.3&110.3 ± 0.3 & 110.1 ± 0.1& \textbf{112.8} ± 0.2 &\underline{109.7} ± 0.3\\
         hopper-medium-expert &0.7 ± 0.0 &  98.0 ± 9.4 & 96.9 ± 15.1 &85.5 ± 29.7& \underline{109.8} ± 0.6&107.0 ± 6.4& 110.1 ± 0.3& 110.7 ± 0.1& \textbf{112.7} ± 0.2& 108.7 ± 0.5\\
         hopper-medium-replay &7.4 ± 0.5 &  60.9 ± 18.8 & 86.3 ± 7.3 & 89.6 ± 13.2& 97.2 ± 9.0
&98.1 ± 5.3& 101.8 ± 0.5& 101.0 ± 0.5&\textbf{102.8} ± 0.5
&\underline{100.5} ± 1.0\\
         hopper-full-replay &41.1 ± 17.9 & 97.9 ± 17.5  & 101.9 ± 0.6 & 104.4 ± 10.8 &107.4 ± 0.8
 &107.1 ± 0.4& 102.9 ± 0.3&105.4 ± 0.7 &-
&\textbf{108.2} ± 0.7\\
         \hline
         halfcheetah-random &29.7 ± 1.4 & 11.0 ± 1.1    & \textbf{31.1} ± 3.5  &19.5 ± 0.8&27.6 ± 2.1
&29.5 ± 1.5&28.0 ± 0.9 & 28.4 ± 1.0& 28.5 ± 0.8& 30.4 ± 4.0\\%300step 31.21078365342147 28.397803808573883 29.51497184943334 31.39021952048656 26.05324946931098 31.652101247917074 25.907170087150764 31.652224957936756 30.281297138733656 33.12797897374533
         halfcheetah-medium &55.2 ± 27.8 &  48.3 ± 0.3   & 46.9 ± 0.4 &50.0 ± 0.2&66.4 ± 1.4 &  65.6 ± 1.0&67.5 ± 1.2& 65.9 ± 0.6& 66.8 ± 0.7& \textbf{68.3} ± 0.2\\ %1500step
         halfcheetah-expert &-0.8 ± 1.8  &  96.7 ± 1.1     & 97.3 ± 1.1 &95.5 ± 2.1& 102.6 ± 4.2
  &  105.9 ± 1.7&105.2 ± 2.6& \textbf{106.8} ± 3.4& 105.2 ± 0.7& \underline{106.2} ± 3.7 \\
  % halfcheetah expert record    |  step: 5999
  %actor beta: 5/critic beta: 4  |  99.17 ± 9.62
  %actor beta: 5/critic beta: 5  |  103.70 ± 4.74
  %actor beta: 5/critic beta: 6  |  106.24 ± 3.68
        halfcheetah-medium-expert &28.4 ± 19.4  &  90.7 ± 4.3      & 95.0 ± 1.4 &92.7 ± 2.8& 107.6 ± 2.8
  &  101.1 ± 5.2&107.1 ± 2.0&106.3 ± 1.9 & 107.8 ± 1.1 & \textbf{108.5} ± 1.1 \\% 3000steps, 4seeds
        halfcheetah-medium-replay &0.8 ± 1.0  &  44.6 ± 0.5      & 45.3 ± 0.3 &42.1 ± 3.6&  51.2 ± 3.2 &51.0 ± 0.8 &\textbf{63.9} ± 0.8& 61.3 ± 1.9&61.9 ± 1.5 & \underline{52.1} ± 4.8\\ %0.1 0.1 500steps 5seed
         halfcheetah-full-replay &\textbf{86.8} ± 1.0  &   75.0 ± 2.5    & 76.9 ± 0.9  & 75.0 ± 0.7 &  81.2 ± 1.3
 &  82.1 ± 1.1 &84.5 ± 1.2&84.6 ± 0.9 & - & 81.4 ± 1.7\\% 3000 steps 5 seeds
         \hline
         walker2d-random & 0.9 ± 0.8 & 1.6 ± 1.7 & 5.1 ± 1.7 & 11.3 ± 7.0& 18.7 ± 6.9&18.1 ± 4.5
&\textbf{21.7} ± 0.0& 16.6 ± 7.0&21.4 ± 0.2  & \textbf{21.7} ±  0.1\\
         walker2d-medium & -0.3 ± 0.2 &  83.7 ± 2.1     &   79.5 ± 3.2&80.7 ± 3.4&91.6 ± 2.8& 82.5 ± 3.6&87.9 ± 0.2& 92.5 ± 0.8& \textbf{102.4} ± 1.4& \underline{95.2} ± 0.7\\%79.6&76.6& 80.2&79.5 
        walker2d-expert & 0.7 ± 0.3 & 110.2 ± 0.3 & 109.3 ± 0.1 & 96.9 ± 32.3& 104.5 ± 22.8  &112.3 ± 0.2
&107.4 ± 2.4& 115.1 ± 1.9&\textbf{115.4} ± 0.5 &\underline{114.0}± 0.5\\ %4000
         walker2d-medium-expert& 1.9 ± 3.9 &  110.1 ± 0.5  & 109.1 ± 0.2&112.1 ± 0.5& 104.6 ± 11.2 &\underline{111.6} ± 0.3
&116.7 ± 0.4& 114.7 ± 0.9& \textbf{121.2} ± 1.5 & 109.6 ± 1.0\\
         walker2d-medium-replay& -0.4 ± 0.3 & 81.8 ± 5.5   & 76.8 ± 10.0 &75.4 ± 9.3& 88.7 ± 7.7
 &77.3 ± 7.9
&78.7 ± 0.7&87.1 ± 2.4 & 90.4 ± 0.5& \textbf{91.0} ± 2.9\\
         walker2d-full-replay & 27.9 ± 47.3 & 90.3 ± 5.4  & 94.2 ± 1.9 &97.5 ± 1.4& 105.3 ± 3.2 & 102.2 ± 1.7
&94.6 ± 0.5 & 99.8 ± 0.7  & -& \textbf{109.6} ± 0.7\\
\hline
average score&16.2 & 67.5 & 73.6 & 72.9 & 82.6 & 81.2 & 84.4 & 85.2 & 85.7 & \textbf{86.0}\\
         \bottomrule
     	
        \end{tabular}
    }

    \vspace{5mm}  % Increase or decrease this value to adjust the space
    
     }
     \setlength{\tabcolsep}{1.3mm}{
    \resizebox{\linewidth}{17mm}{
    \begin{tabular}{l| c c c c c c| c c |c}
\toprule
    \textbf{Dataset} & \textbf{SAC} & \textbf{TD3+BC} & \textbf{CQL}& \textbf{IQL} & \textbf{SAC-RND} & \textbf{ReBRAC}& \textbf{RORL} & \textbf{MSG}& \textbf{SAC-DRND}\\
    \midrule
         antmaze-umaze & 0.0 &78.6&74.0 &   83.3  ± 4.5    & 97.0 ± 1.5 
  &   \underline{97.8} ± 1.0& 97.7 ± 1.9 & \textbf{97.9} ± 1.3 & 95.8 ± 2.4\\
         antmaze-umaze-diverse & 0.0 &71.4&84.0 &   70.6 ± 3.7   &  66.0 ± 25.0 & \underline{88.3} ± 13.0&\textbf{90.7} ± 2.9& 79.3 ± 3.0&87.2 ± 3.2\\ 
         \hline
         antmaze-medium-play& 0.0 & 10.6&61.2 &64.6 ± 4.9& 38.5 ± 29.4  & 84.0 ± 4.2
&76.3 ± 2.5& 85.9 ± 3.9 & \textbf{86.2} ± 5.4\\ % step 2000
         antmaze-medium-diverse & 0.0 &3.0&53.7& 61.7 ± 6.1   & 74.7 ± 10.7
 &76.3 ± 13.5
&69.3 ± 3.3& \textbf{84.6} ± 5.2 &\underline{83.0} ±3.8\\% step 2000
         \hline
         antmaze-large-play& 0.0 & 0.2&15.8& 42.5 ± 6.5 & 43.9 ± 29.2
  &   \underline{60.4} ± 26.1
&16.3 ± 11.1 &\textbf{64.3} ± 12.7&53.2 ± 4.1\\
         antmaze-large-diverse & 0.0 & 0.0&14.9& 27.6 ± 7.8 &  45.7 ± 28.5 &\underline{54.4} ± 25.1
&41.0 ± 10.7 & \textbf{71.3} ± 5.3&50.8 ± 10.5\\
\hline
average score& 0.0 & 27.3 & 50.6 & 58.3 & 60.9 & \underline{76.8} & 65.2 & \textbf{80.5} & 76.0 \\
\bottomrule
    \end{tabular}
    }
    }
   \caption{Average normalized scores of algorithms. The figure shows the scores of MuJoCo's final model evaluated 10 times (AntMaze 100 times) on training seeds and 10 random seeds. SAC, SAC-N, and EDAC scores are taken from \cite{an2021uncertainty}. CQL, IQL, SAC-RND, and ReBRAC scores are taken from \cite{tarasov2023revisiting}. RORL scores are taken from \cite{yang2022rorl}. MSG scores are taken from \cite{ghasemipour2022so}.}
        \label{TABLEa}
\end{table}
\vspace{0.25cm}

\subsection{Results on Adroit Tasks}\label{adroit_exp}
In this subsection, we show the scores of SAC-DRND on Adroit tasks in Table \ref{adroit_table}.

\begin{table}[h]
\tiny
    \centering
     \setlength{\tabcolsep}{1mm}{
    \resizebox{\linewidth}{33mm}{
    \begin{tabular}{l| c c c c c c|c}
\toprule
    \textbf{Task Name} & \textbf{BC} & \textbf{TD3+BC} & \textbf{IQL}& \textbf{CQL} & \textbf{SAC-RND} & \textbf{ReBRAC}& \textbf{SAC-DRND}\\
    \midrule
         pen-human & 34.4 &81.8 ± 14.9 & 81.5 ± 17.5 & 37.5 & 5.6 ± 5.8
  & \textbf{103.5} ± 14.1 & 42.3 ± 11.8\\
         pen-cloned & 56.9 &61.4 ± 19.3&77.2 ± 17.7 & 39.2 & 2.5 ± 6.1 &\textbf{91.8} ± 21.7& 39.5 ± 33.4\\ 
         pen-expert & 85.1 & 146.0 ± 7.3 & 133.6 ± 16.0 & 107.0 & 45.4 ± 22.9 & \textbf{154.1} ± 5.4 & 65.0 ± 17.1\\
         \hline
         door-human & 0.5 & -0.1 ± 0.0 &3.1 ± 2.0 &\textbf{9.9}& 0.0 ± 0.1 & 0.0 ± 0.0 & 1.3 ± 0.8\\
         door-cloned & -0.1 &0.1 ± 0.6&0.8 ± 1.0&0.4&0.2 ± 0.8
 &\textbf{1.1} ± 2.6 & 0.3 ± 0.1\\
         door-expert & 34.9 & 84.6 ± 44.5 & \textbf{105.3} ± 2.8 & 101.5 & 73.6 ± 26.7 & 104.6 ± 2.4 & 85.3 ± 37.9\\
         \hline
         hammer-human& 1.5 & 0.4 ± 0.4 &2.5 ± 1.9& \textbf{4.4}
  & -0.1 ± 0.1 &0.2 ± 0.2& 0.3 ± 0.2 \\
         hammer-cloned & 0.8 & 0.8 ± 0.7&1.1 ± 0.5&2.1&  0.1 ± 0.4&\textbf{6.7} ± 3.7 & 1.1 ± 0.8\\
         hammer-expert & 125.6 & 117.0 ± 30.9 & 129.6 ± 0.5 & 86.7 & 24.8 ± 39.4 & \textbf{133.8} ± 0.7 & 37.1 ± 47.2\\
         \hline
         relocate-human & 0.0 & -0.2 ± 0.0 & 0.1 ± 0.1 & \textbf{0.2} & 0.0 ± 0.0 & 0.0 ± 0.0 & 0.0 ± 0.1\\
         relocate-cloned & -0.1 & -0.1 ± 0.1 & 0.2 ± 0.4 & -0.1 & 0.0 ± 0.0 & \textbf{0.9} ± 1.6 & 0.0 ± 0.0\\
         relocate-expert & 101.3 & \textbf{107.3} ± 1.6 & 106.5 ± 2.5 & 95.0 & 3.4 ± 4.5 & 106.6 ± 3.2 & 10.1 ± 7.1\\
         \hline
         Average w/o expert & 11.7 & 18.0 & 20.8 & 11.7 & 1.0 & \textbf{25.5} & 10.6 \\
         \hline
         Average & 36.7 & 49.9 & 53.4 & 40.3 & 12.9 & \textbf{58.6} & 23.5 \\
\bottomrule
    \end{tabular}
    }
    }
   \caption{Average normalized scores on Adroit tasks. There is still a significant improvement compared to SAC-RND, from 12.9 to 23.5. This illustrates the superiority of DRND compared to RND. In addition, the average score without using the expert dataset has also improved significantly, reaching a level comparable to CQL(11.7), which benefits from the performance in the Pen environment.}
   \label{adroit_table}
\end{table}
\vspace{0.25cm}

\subsection{Expected Online Performance}\label{eop}
We calculated the EOP on Gym-MuJoCo and AntMaze tasks, as shown in Table \ref{table_eop}.

\begin{table}[h]
\tiny
    \centering
     \setlength{\tabcolsep}{1.1mm}{
    \resizebox{\linewidth}{19mm}{
    \begin{tabular}{l l| c c c c c c c}
\toprule
    \textbf{Domain} & \textbf{Algorithm} & \textbf{1 policy} & \textbf{2 policies}& \textbf{3 policies} & \textbf{5 policies} & \textbf{10 policies}& \textbf{15 policies} &\textbf{20 policies}\\
    \midrule
    \multirow{4}*{Gym-MuJoCo} & \textbf{TD3+BC} & 49.8 ± 21.4 & 61.0 ± 14.5 & 65.3 ± 9.3 & 67.8 ± 3.9 & - & - & - \\
    ~ & \textbf{IQL} & 65.0 ± 9.1 & 69.9 ± 5.6 & 71.7 ± 3.5 & 72.9 ± 1.7 & 73.6 ± 0.8 & 73.8 ± 0.7 & 74.0 ± 0.6 \\
    ~ & \textbf{ReBRAC} & 62.0 ± 17.1 &  70.6 ± 9.9 & 73.3 ± 5.5 & 74.8 ± 2.1 & 75.6 ± 0.8 & 75.8 ± 0.6 & 76.0 ± 0.5\\
    ~ & \textbf{SAC-DRND} & \textbf{69.9} ± 30.1 & \textbf{73.2} ± 19.0 & \textbf{79.4} ± 11.9 & \textbf{82.5} ± 7.8 & \textbf{84.0} ± 6.0 & \textbf{84.9} ± 3.1 & \textbf{85.3} ± 2.0 \\
    \hline
    \multirow{4}*{AntMaze} & \textbf{TD3+BC} & 6.9 ± 7.0 & 10.7 ± 6.8 & 13.0 ± 6.0 & 15.5 ± 4.6 & - & - & -\\
    ~ & \textbf{IQL} & 29.8 ± 15.5 & 38.0 ± 15.4 & 43.1 ± 13.8 & 48.7 ± 10.2 & 53.2 ± 4.4 & 54.3 ± 2.1 & 54.7 ± 1.2 \\
    ~ & \textbf{ReBRAC} & 67.9 ± 10.0 & 73.6 ± 7.4 & 76.1 ± 5.5 & 78.3 ± 3.4 & 79.9 ± 1.7 & 80.4 ± 1.1 & - \\
    ~ & \textbf{SAC-DRND} & \textbf{69.3} ± 15.9 & \textbf{75.3} ± 10.1 & \textbf{78.5} ± 7.6 & \textbf{81.5} ± 4.0 & \textbf{83.7} ± 3.1 & \textbf{84.5} ± 1.5 & \textbf{84.9} ± 0.9 \\
\bottomrule
    \end{tabular}
    }
    }
   \caption{Expected Online Performance on Gym-MuJoCo and AntMaze tasks. We calculate the mean value of different domains like the way in ReBRAC. The results show SAC-DRND has the best performance.}
   \label{table_eop}
\end{table}
\vspace{0.25cm}

We also show the EOP line for each task, as shown in Figure \ref{figure_eop}.

\begin{figure}[h]
    \centering
    \begin{minipage}[t]{\linewidth}
    
    \subfigure{
    \begin{minipage}[t]{0.15\linewidth}
        \centering
        \includegraphics[width=\linewidth]{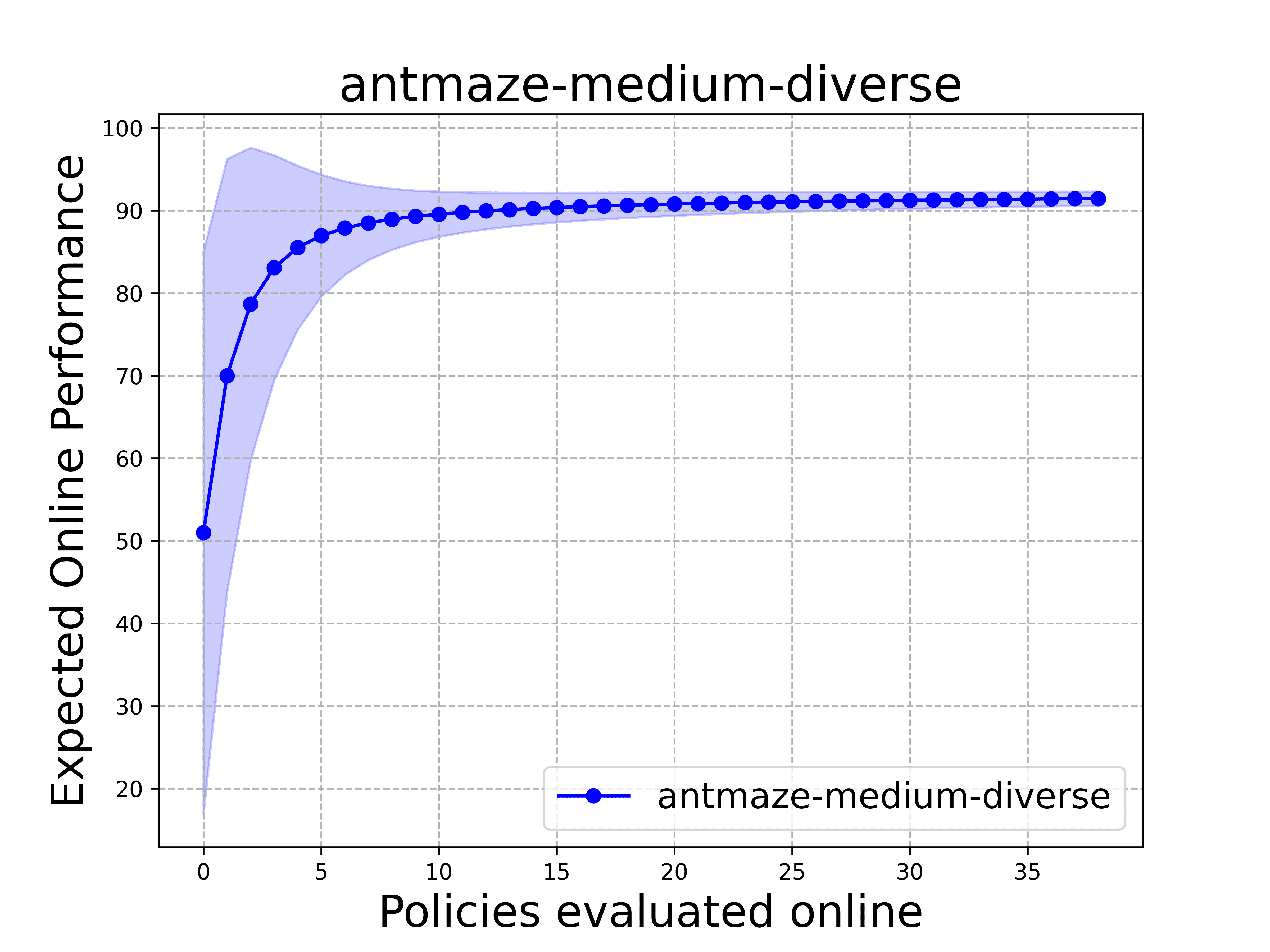}
    \end{minipage}%
    }
    \subfigure{
    \begin{minipage}[t]{0.15\linewidth}
        \centering
        \includegraphics[width=\linewidth]{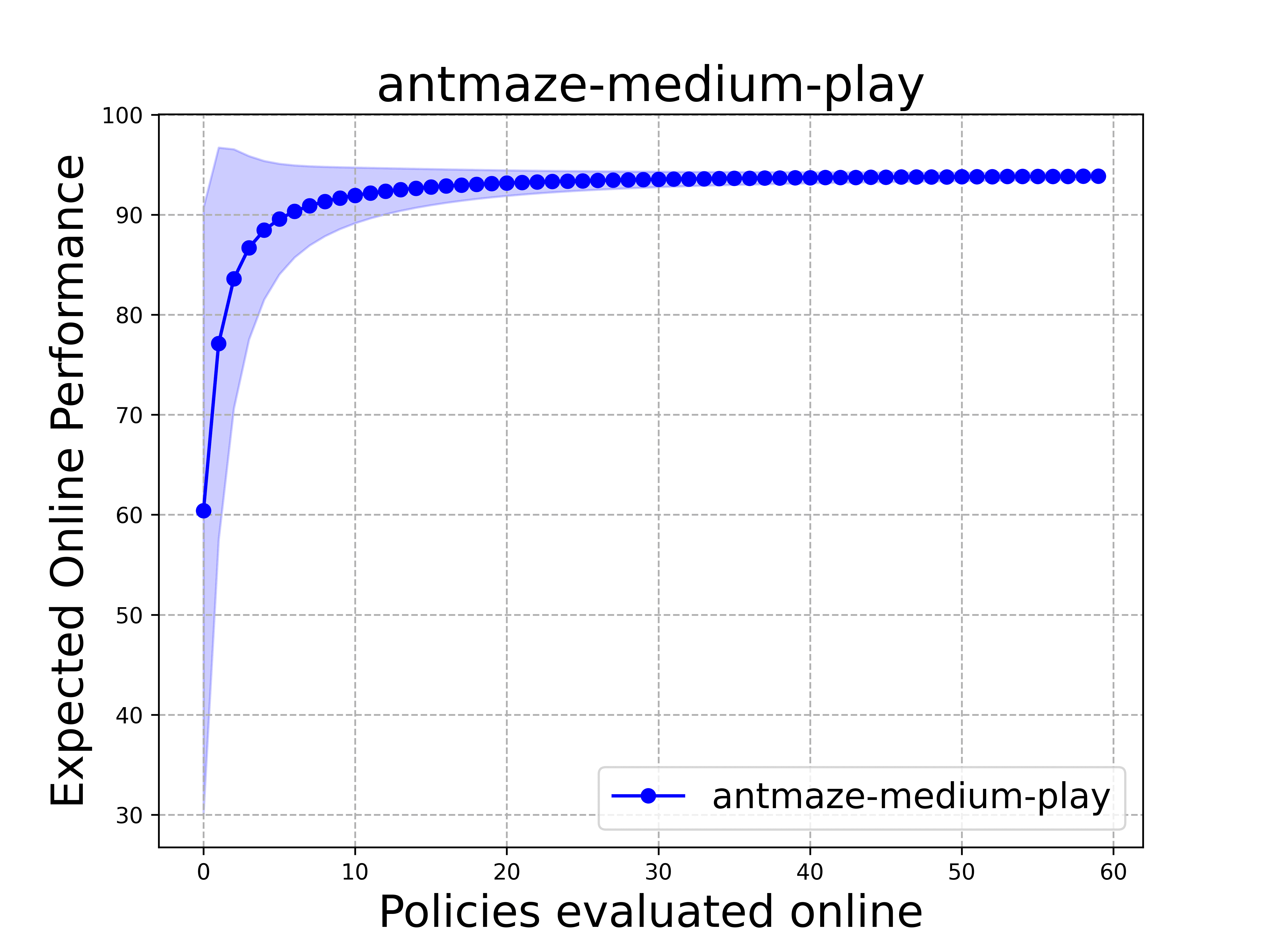}
    \end{minipage}
    }
    \subfigure{
    \begin{minipage}[b]{0.15\linewidth}
        \centering
        \includegraphics[width=\linewidth]{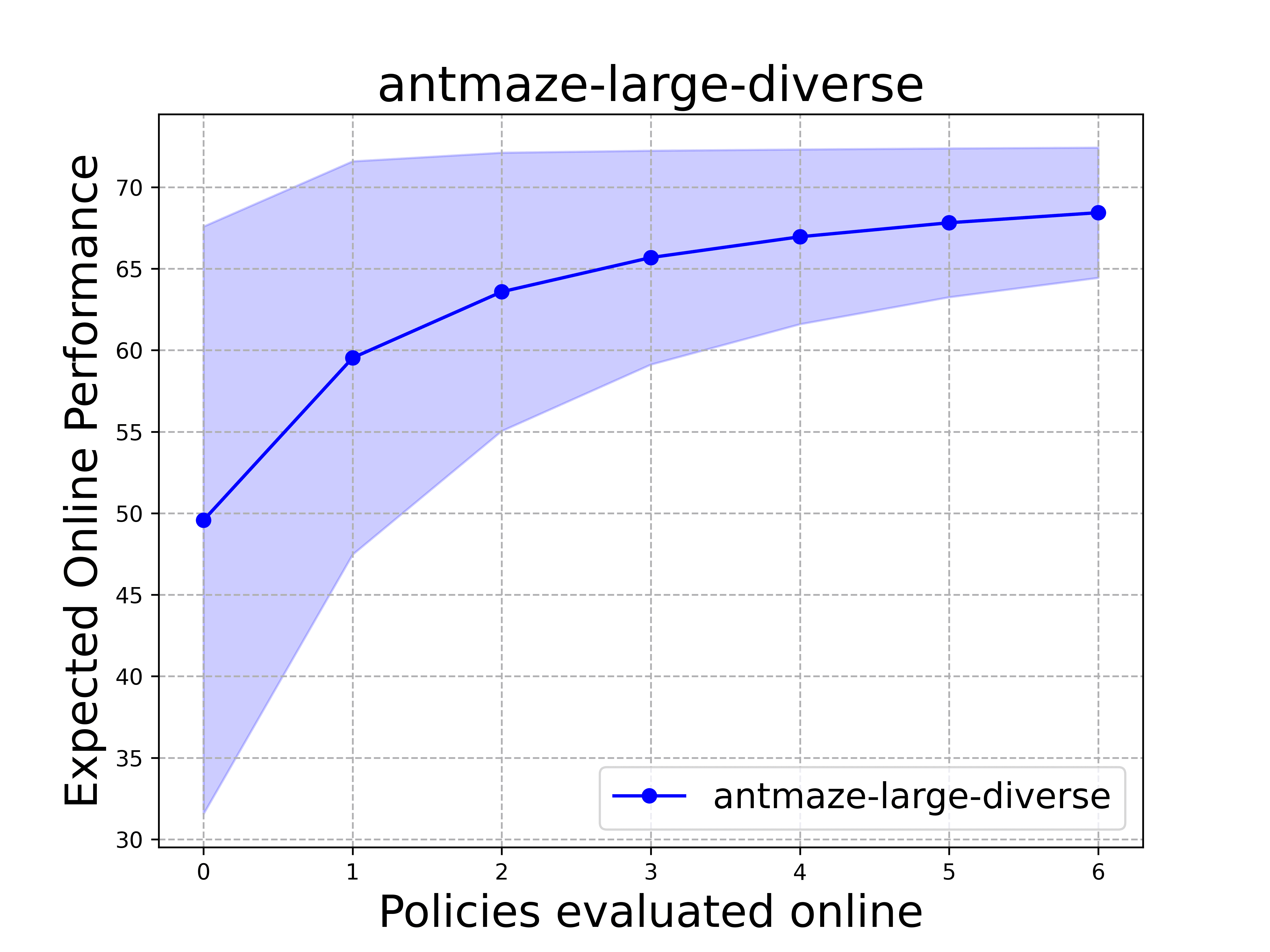}
    \end{minipage}
    }
    \subfigure{
    \begin{minipage}[b]{0.15\linewidth}
        \centering
        \includegraphics[width=\linewidth]{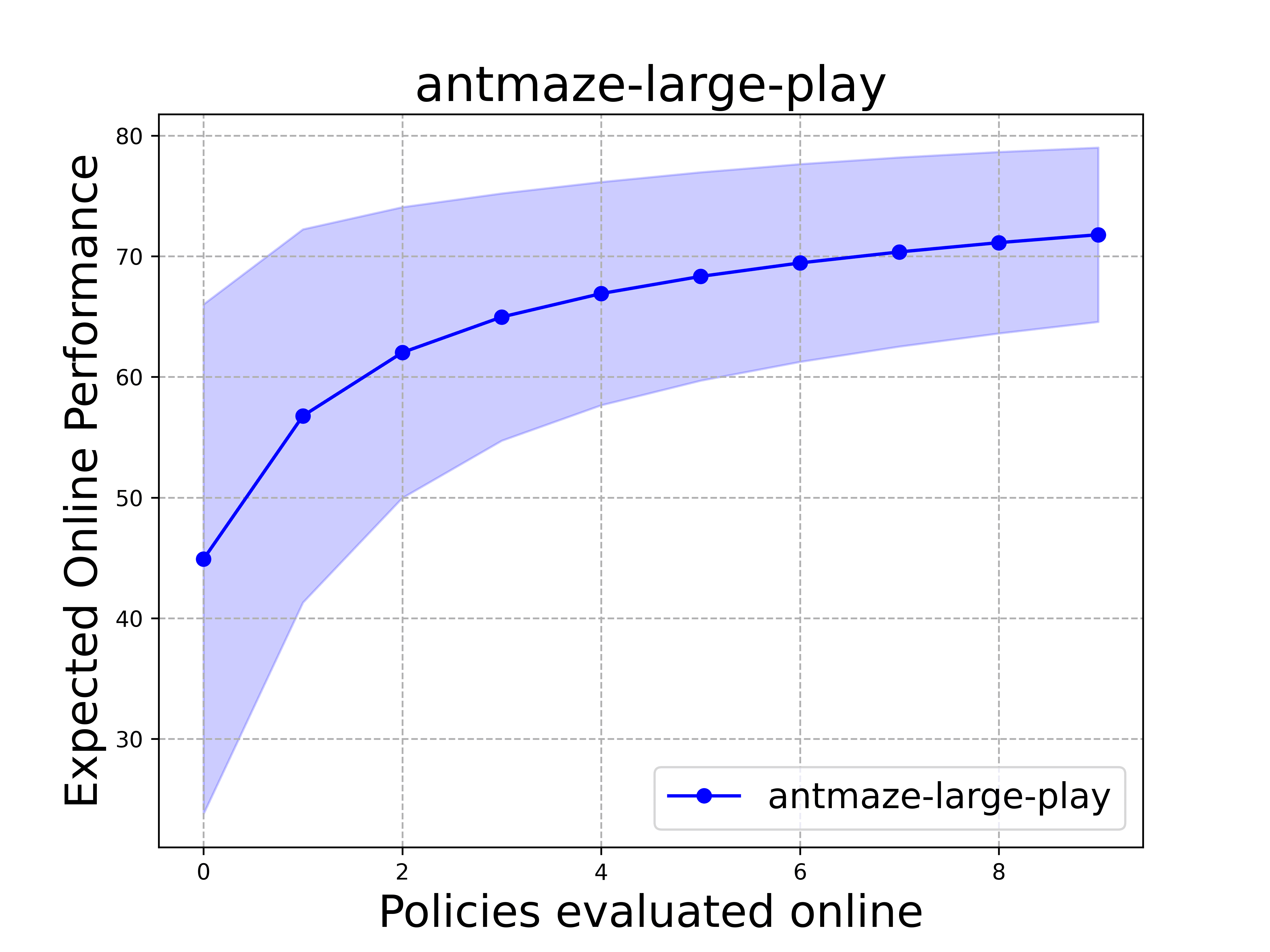}
    \end{minipage}
    }
    \subfigure{
    \begin{minipage}[b]{0.15\linewidth}
        \centering
        \includegraphics[width=\linewidth]{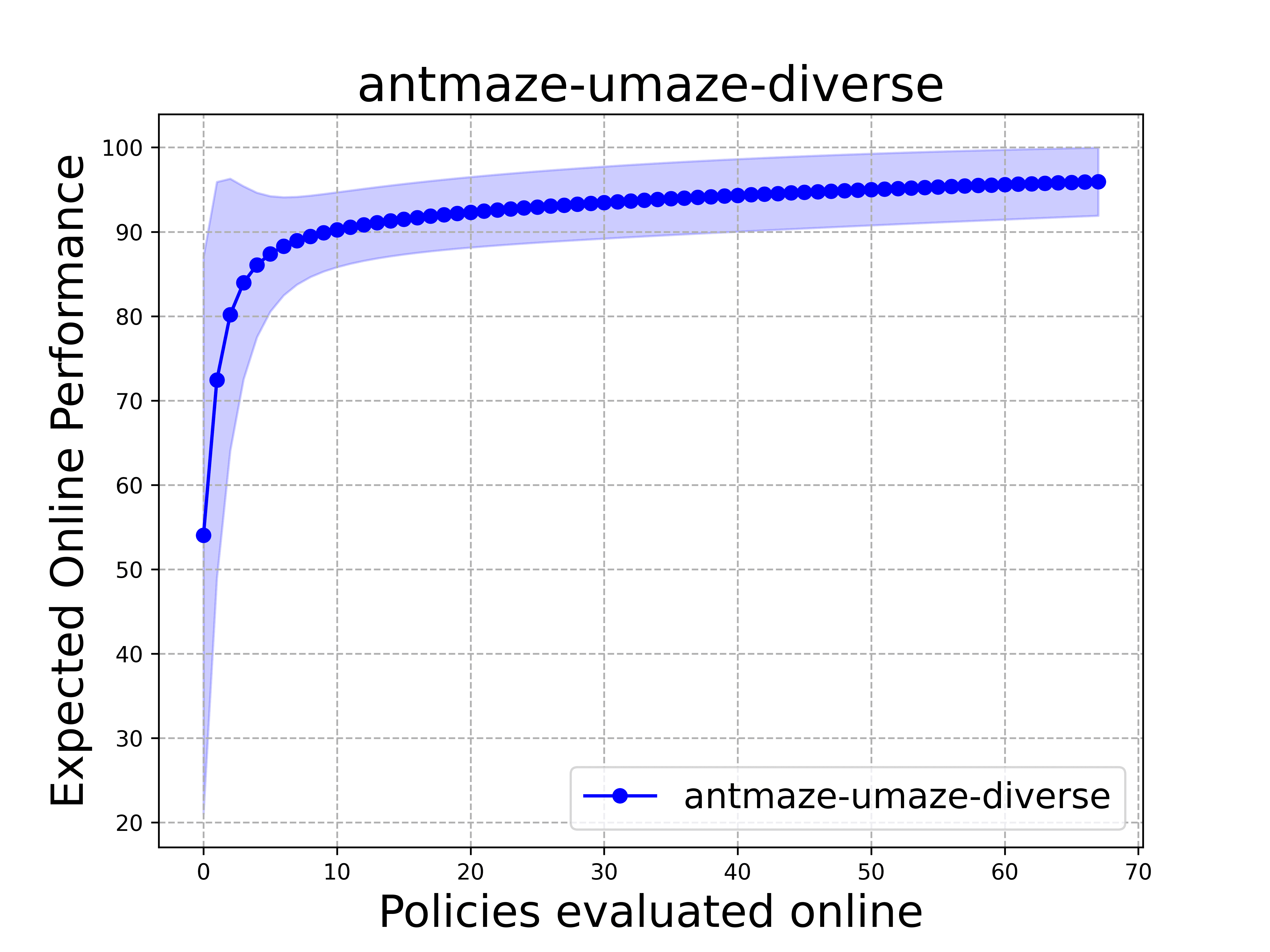}
    \end{minipage}
    }
    \subfigure{
    \begin{minipage}[b]{0.15\linewidth}
        \centering
        \includegraphics[width=\linewidth]{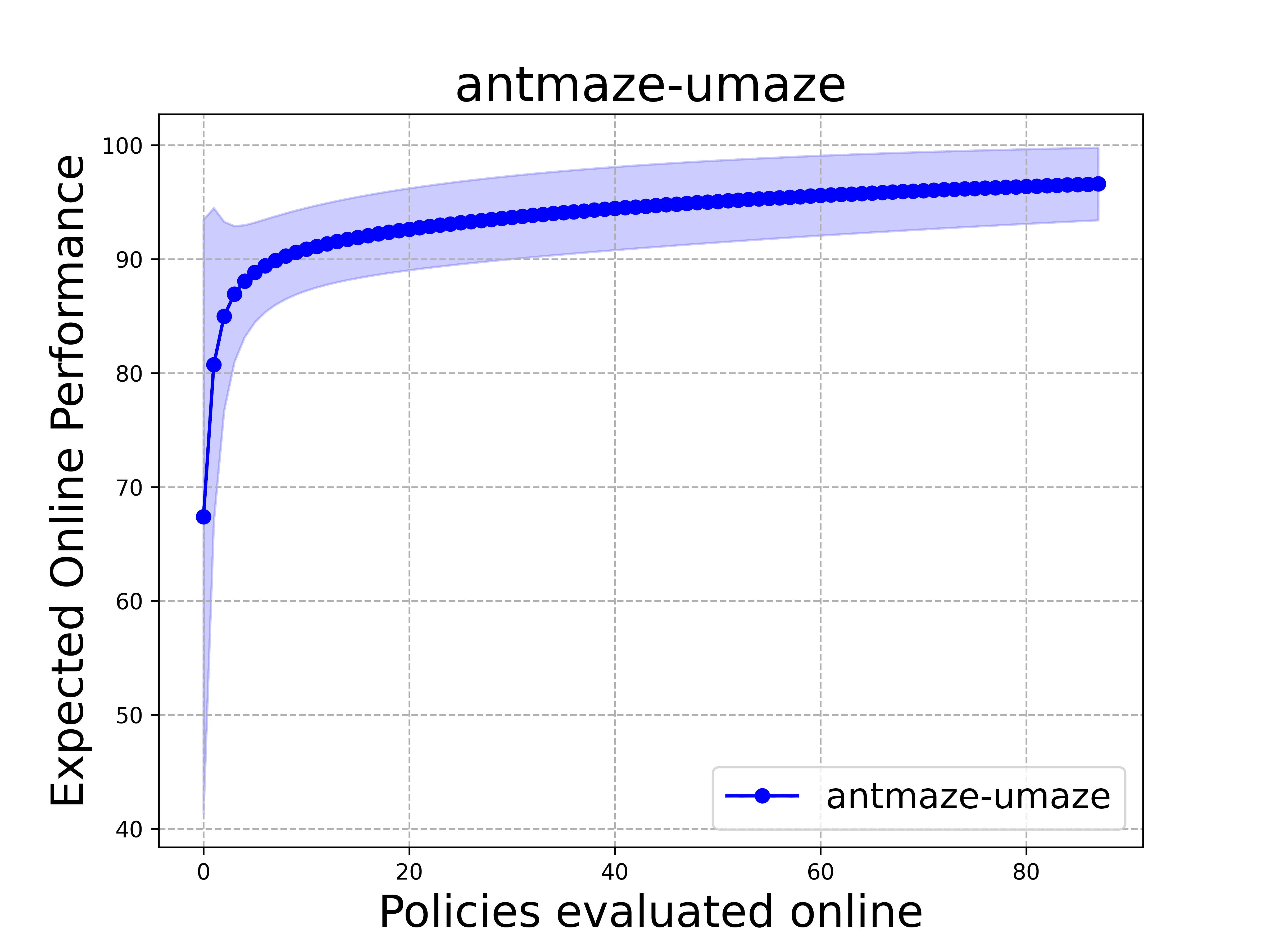}
    \end{minipage}
    }

    \subfigure{
    \begin{minipage}[t]{0.15\linewidth}
        \centering
        \includegraphics[width=\linewidth]{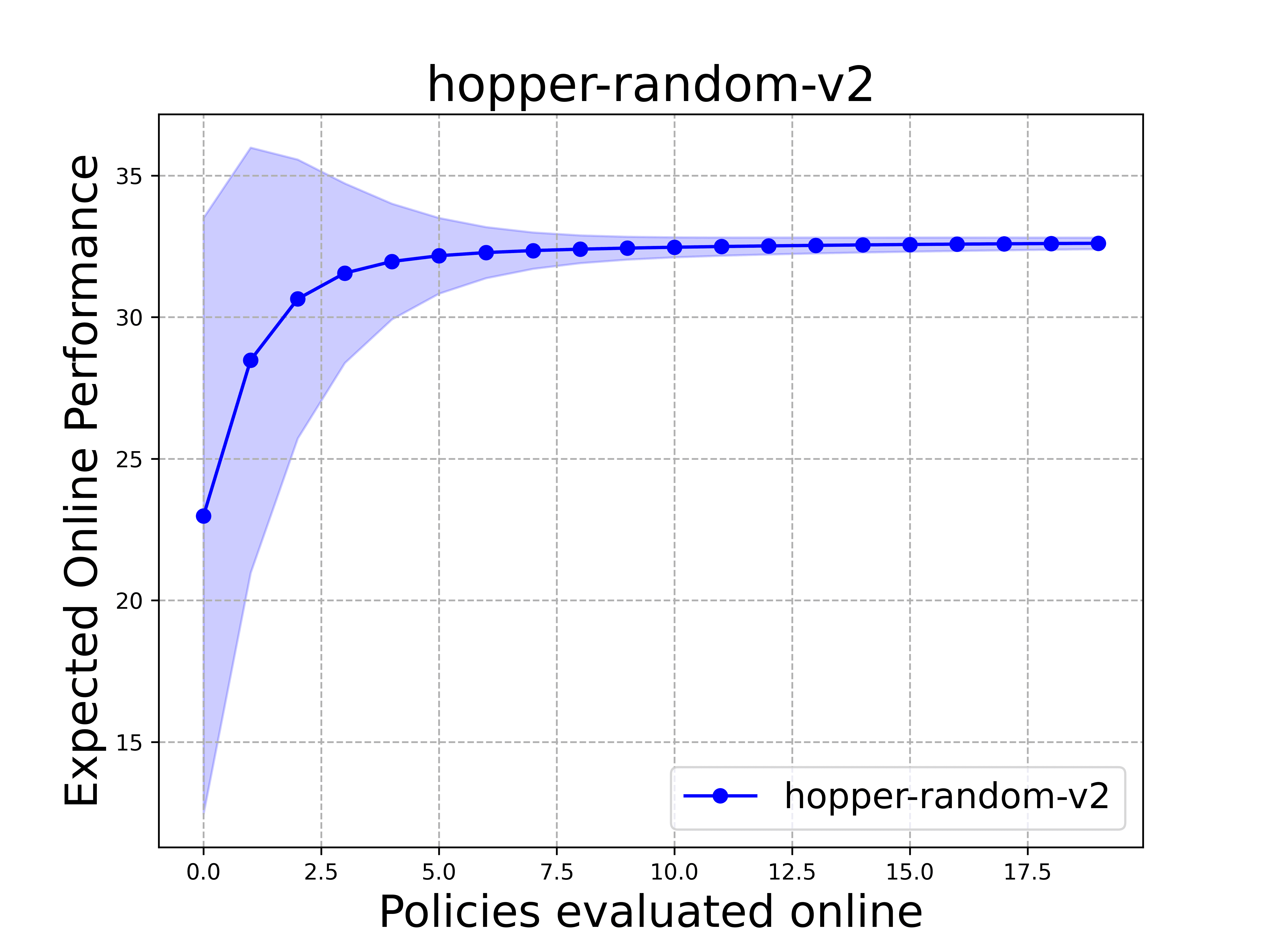}
    \end{minipage}%
    }
    \subfigure{
    \begin{minipage}[t]{0.15\linewidth}
        \centering
        \includegraphics[width=\linewidth]{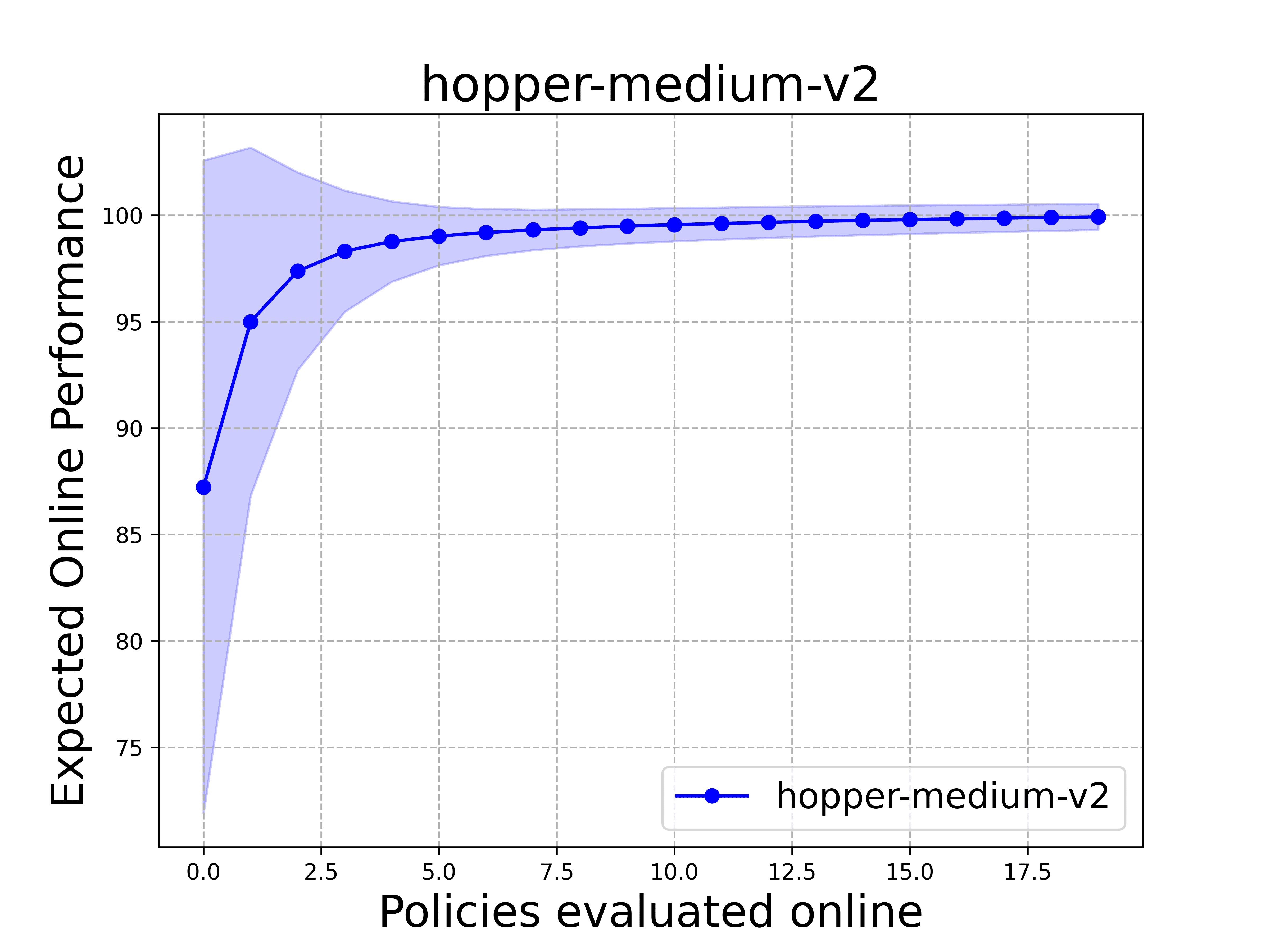}
    \end{minipage}
    }
    \subfigure{
    \begin{minipage}[b]{0.15\linewidth}
        \centering
        \includegraphics[width=\linewidth]{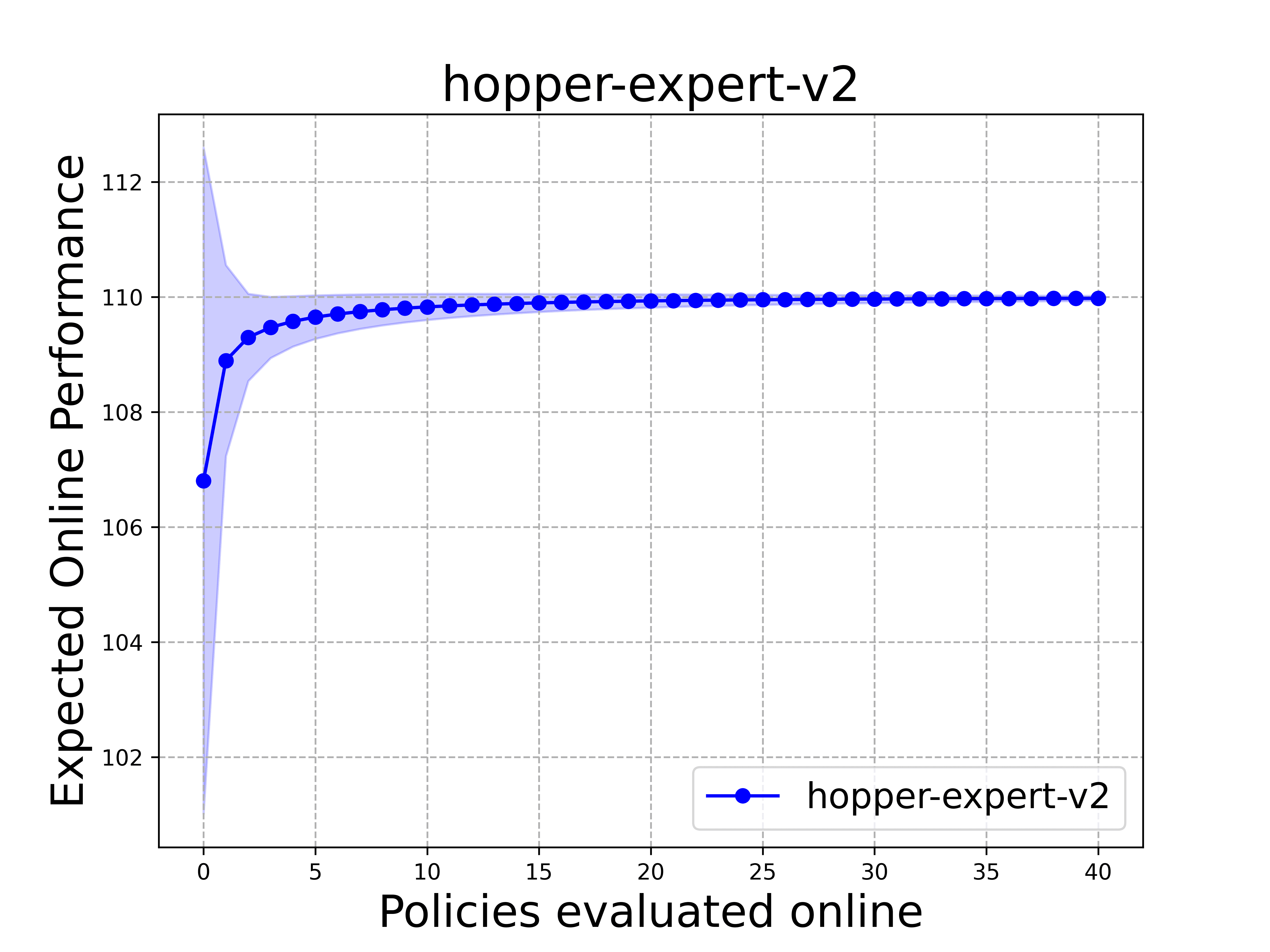}
    \end{minipage}
    }
    \subfigure{
    \begin{minipage}[b]{0.15\linewidth}
        \centering
        \includegraphics[width=\linewidth]{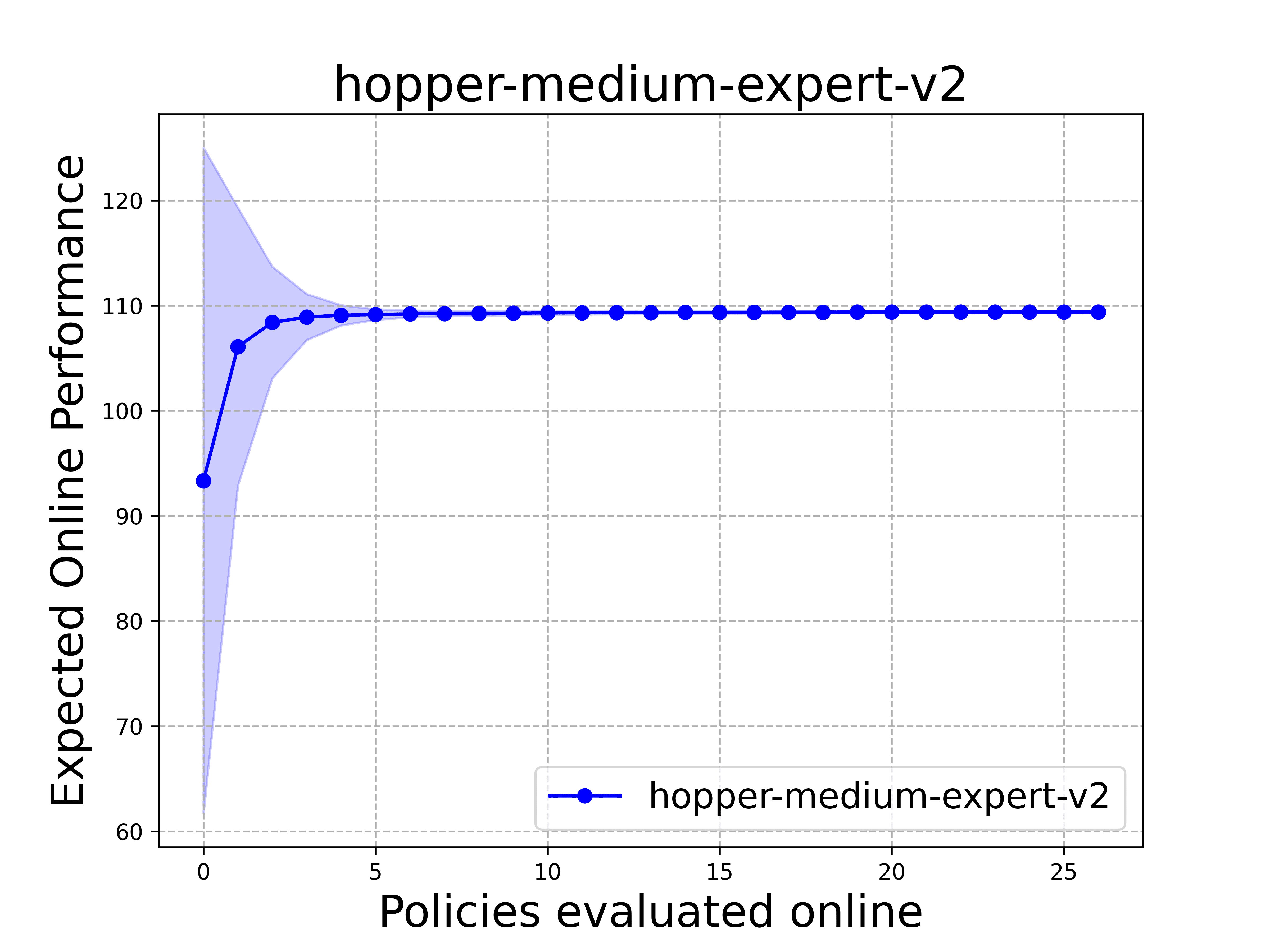}
    \end{minipage}
    }
    \subfigure{
    \begin{minipage}[b]{0.15\linewidth}
        \centering
        \includegraphics[width=\linewidth]{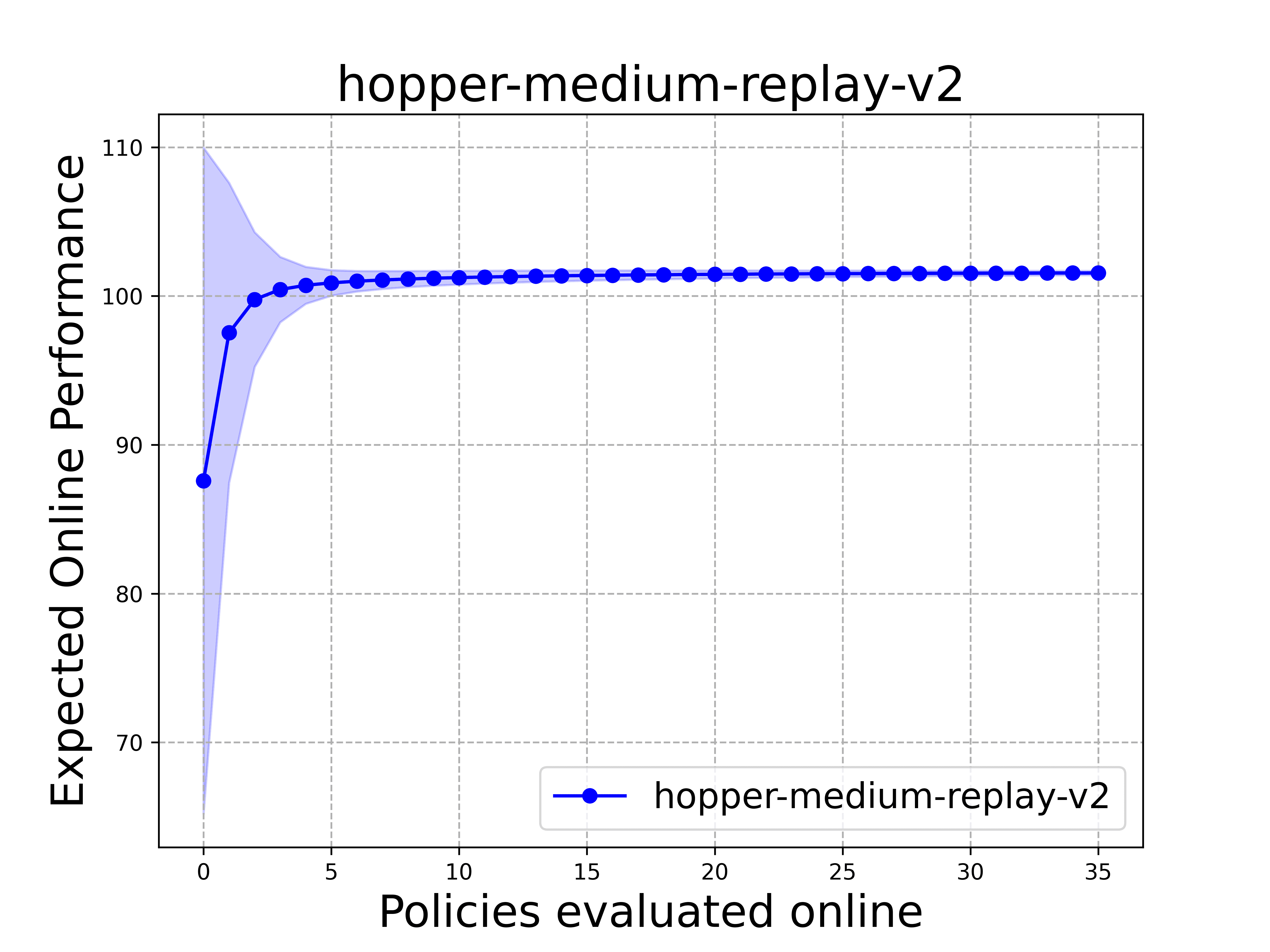}
    \end{minipage}
    }
    \subfigure{
    \begin{minipage}[b]{0.15\linewidth}
        \centering
        \includegraphics[width=\linewidth]{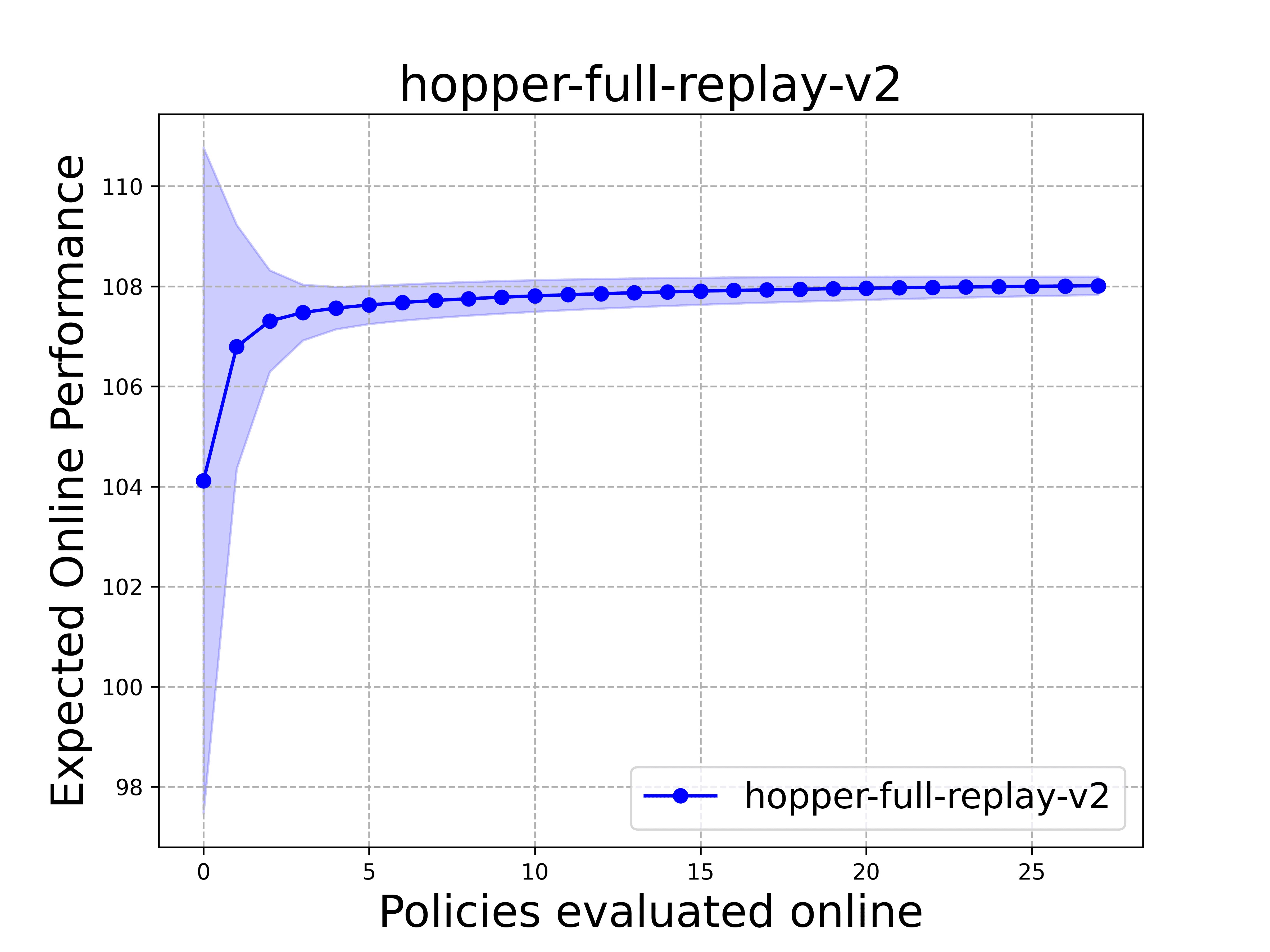}
    \end{minipage}
    }
    \\
    \subfigure{
    \begin{minipage}[t]{0.15\linewidth}
        \centering
        \includegraphics[width=\linewidth]{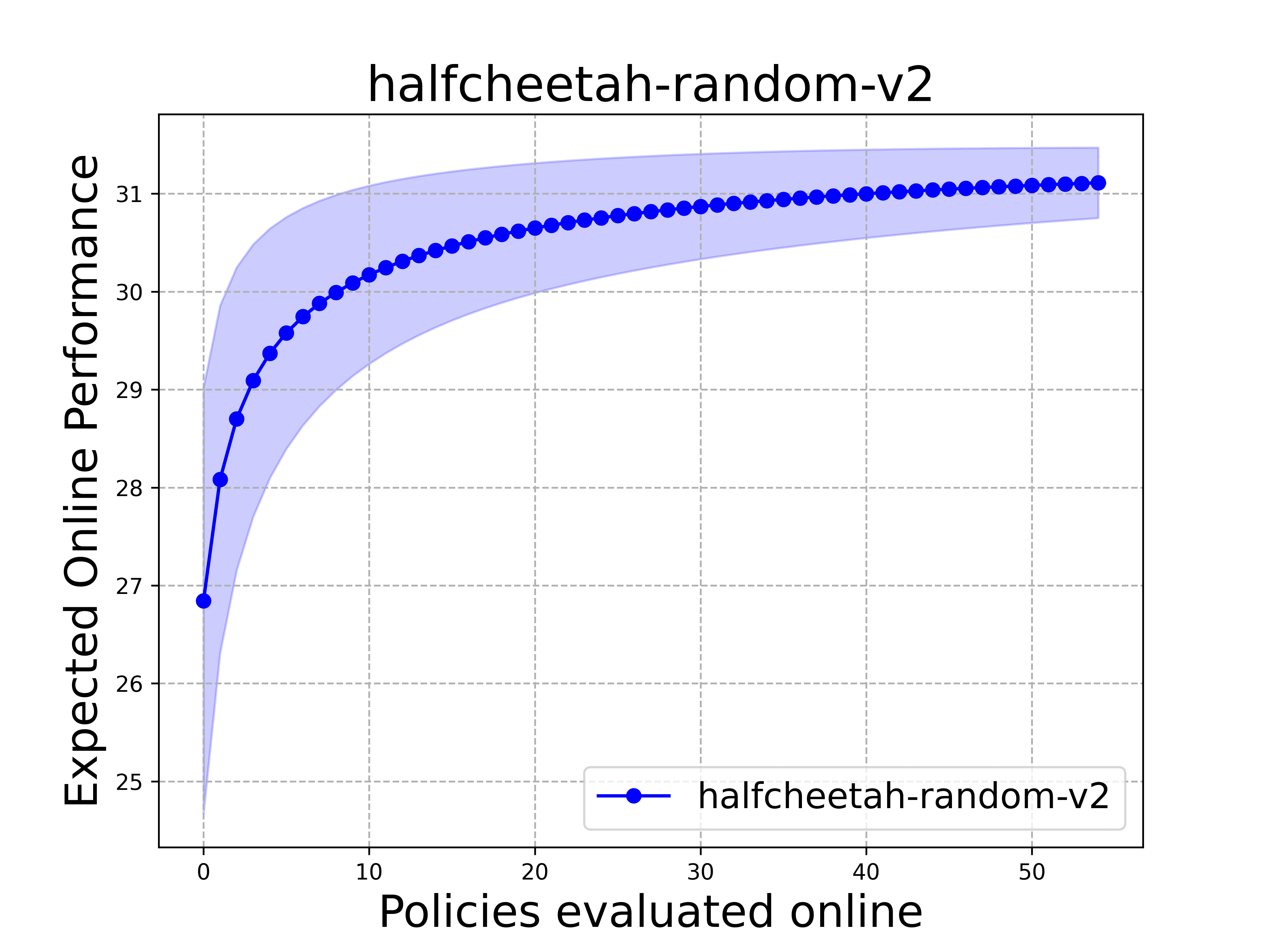}
    \end{minipage}%
    }
    \subfigure{
    \begin{minipage}[t]{0.15\linewidth}
        \centering
        \includegraphics[width=\linewidth]{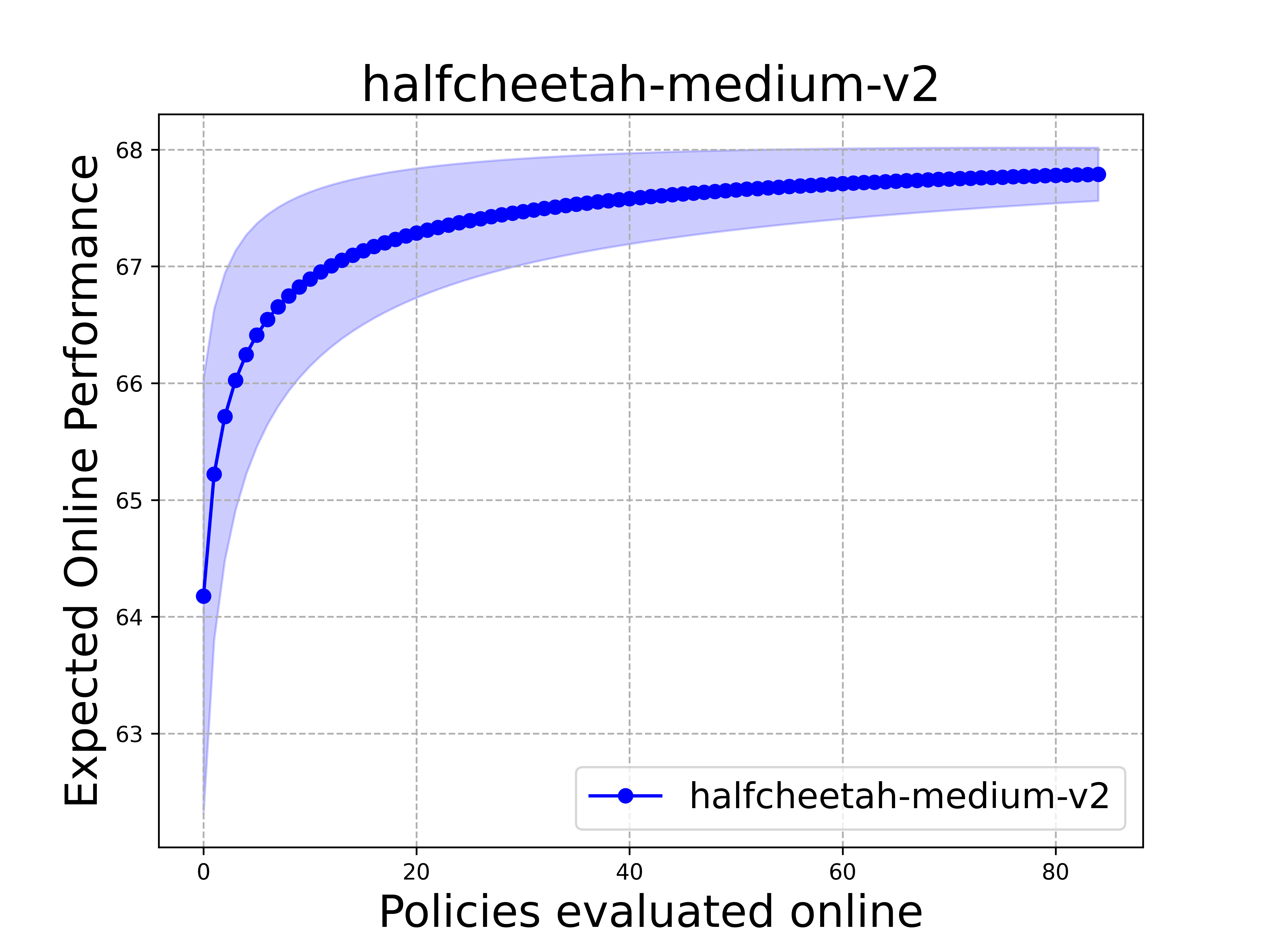}
    \end{minipage}
    }
    \subfigure{
    \begin{minipage}[b]{0.15\linewidth}
        \centering
        \includegraphics[width=\linewidth]{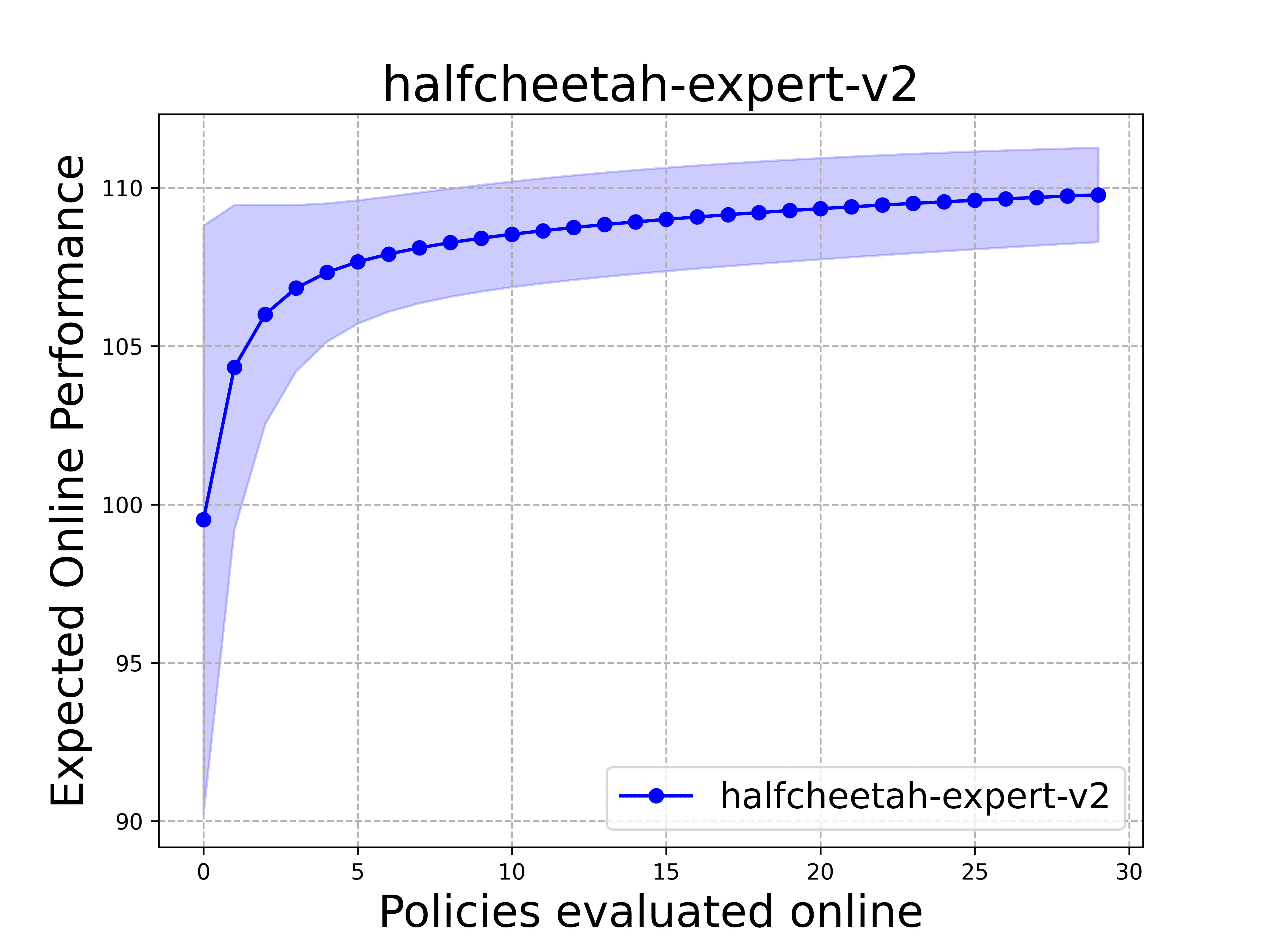}
    \end{minipage}
    }
    \subfigure{
    \begin{minipage}[b]{0.15\linewidth}
        \centering
        \includegraphics[width=\linewidth]{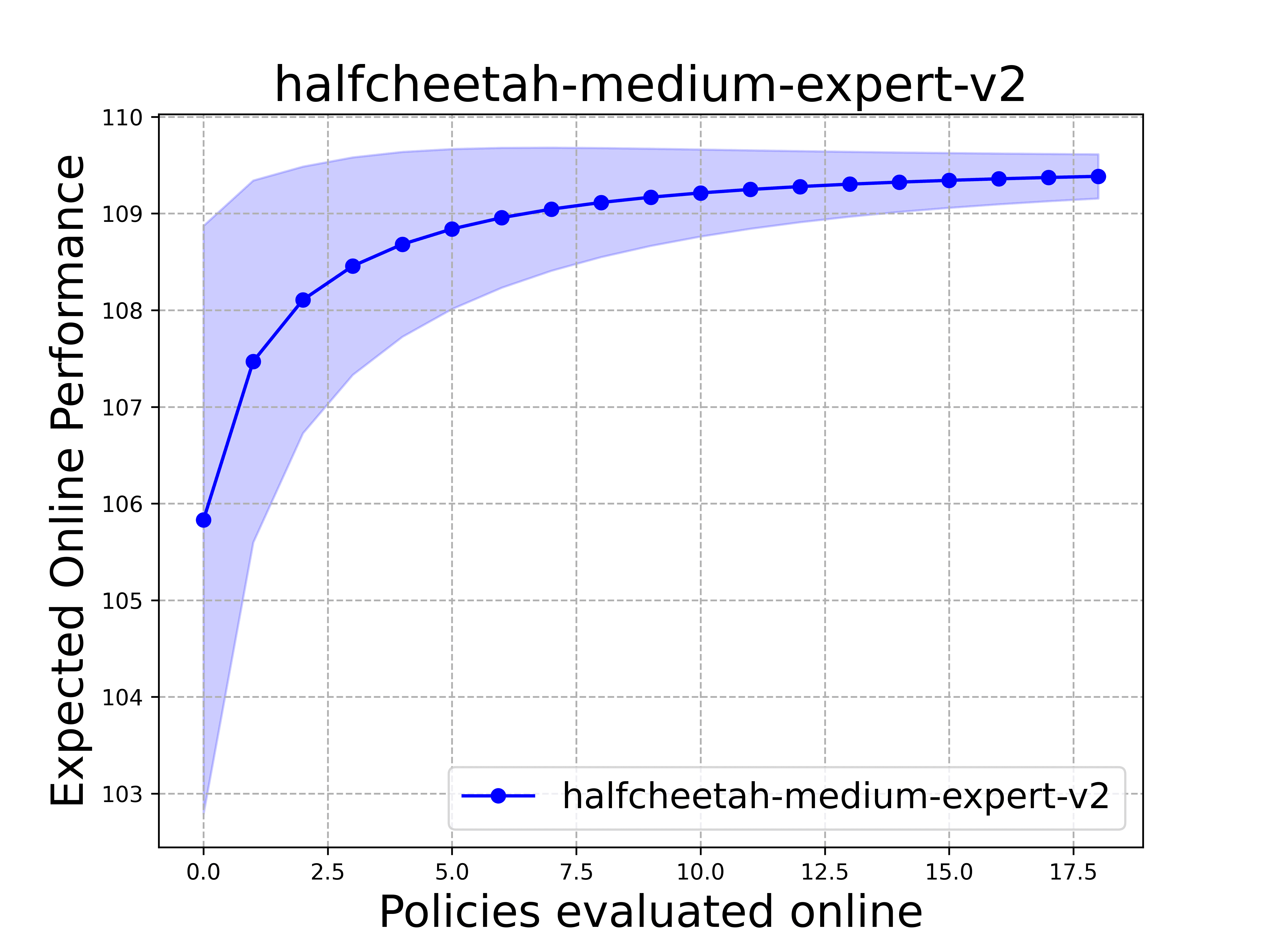}
    \end{minipage}
    }
    \subfigure{
    \begin{minipage}[b]{0.15\linewidth}
        \centering
        \includegraphics[width=\linewidth]{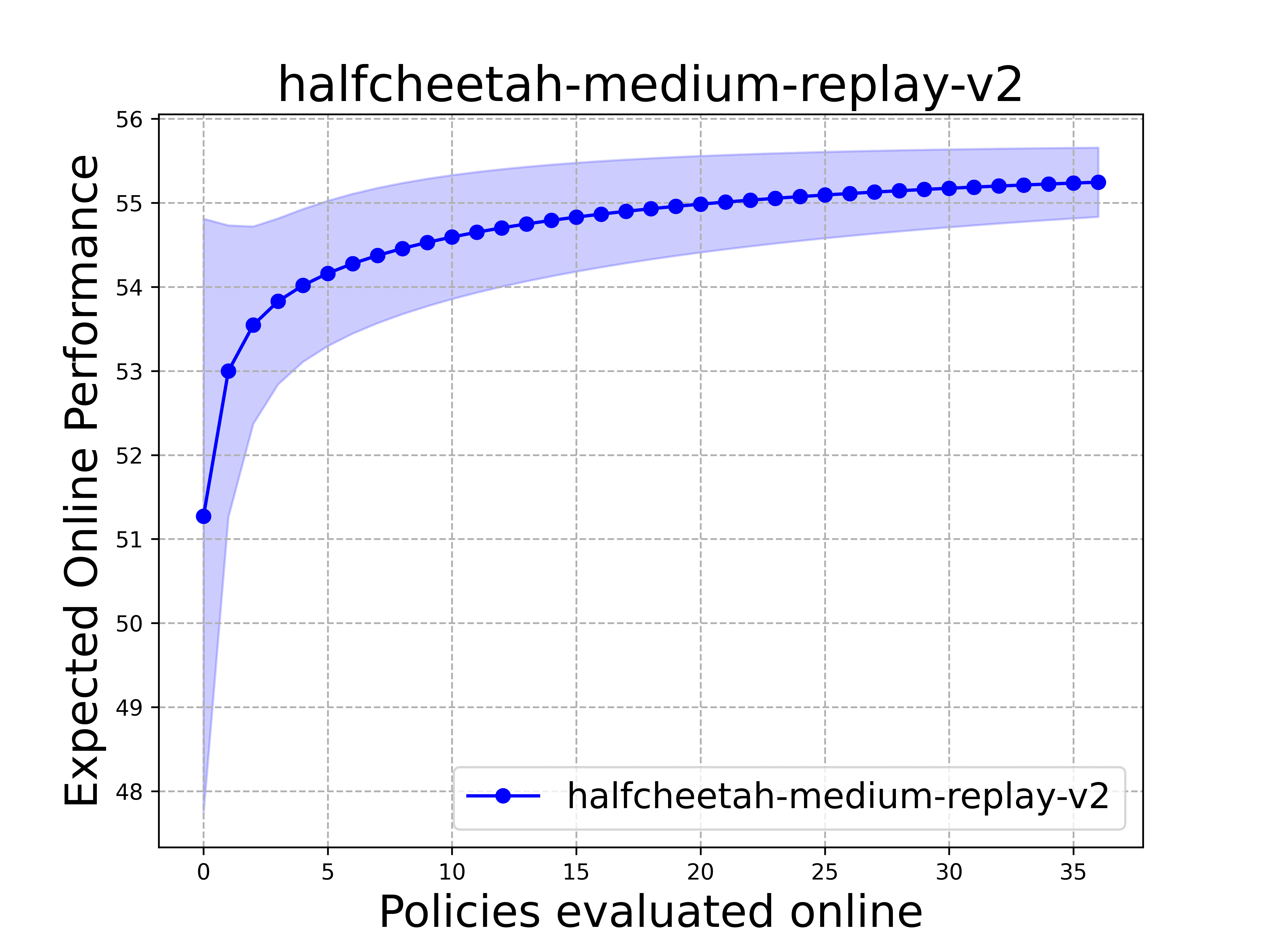}
    \end{minipage}
    }
    \subfigure{
    \begin{minipage}[b]{0.15\linewidth}
        \centering
        \includegraphics[width=\linewidth]{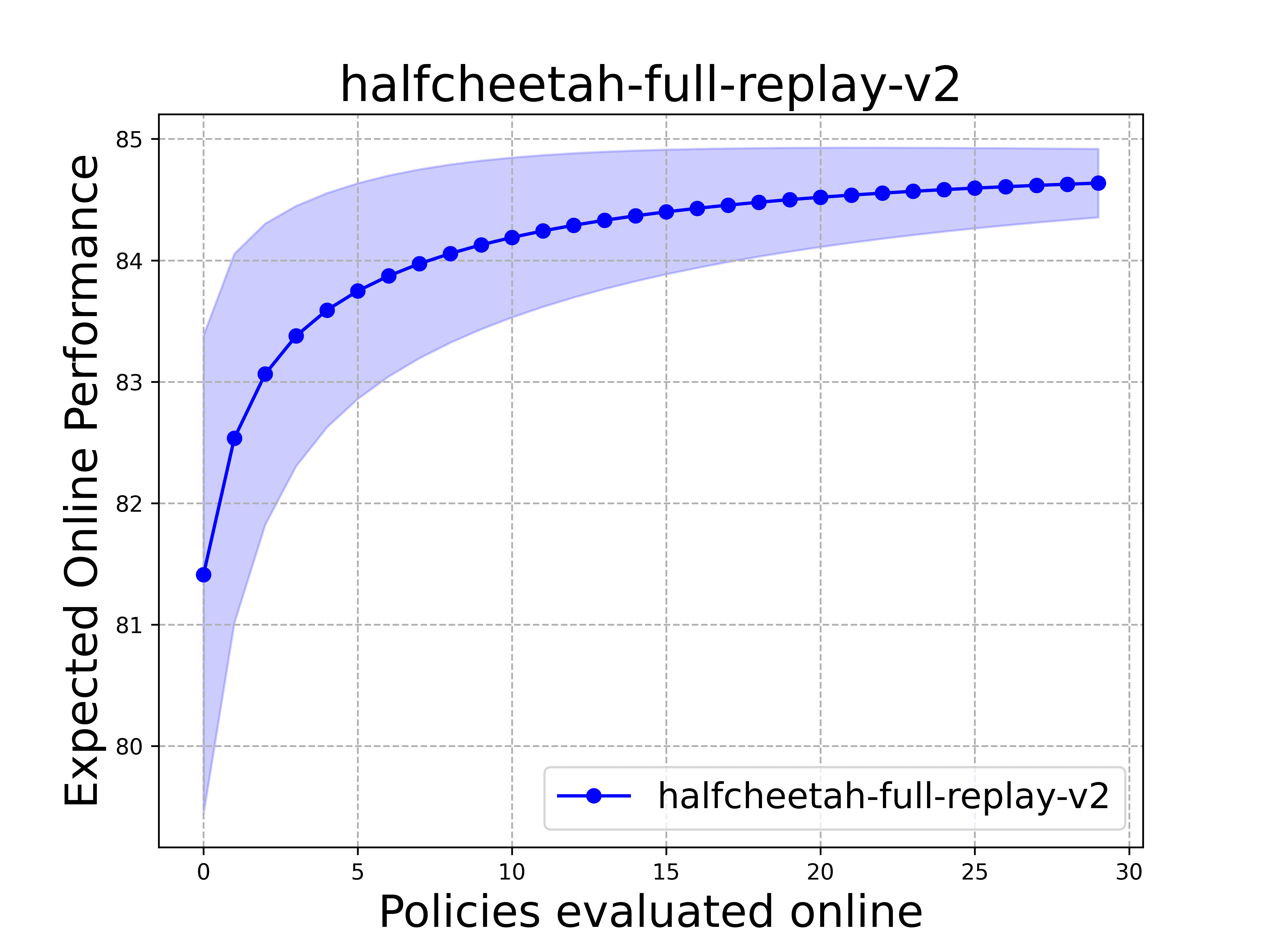}
    \end{minipage}
    }

    \subfigure{
    \begin{minipage}[t]{0.15\linewidth}
        \centering
        \includegraphics[width=\linewidth]{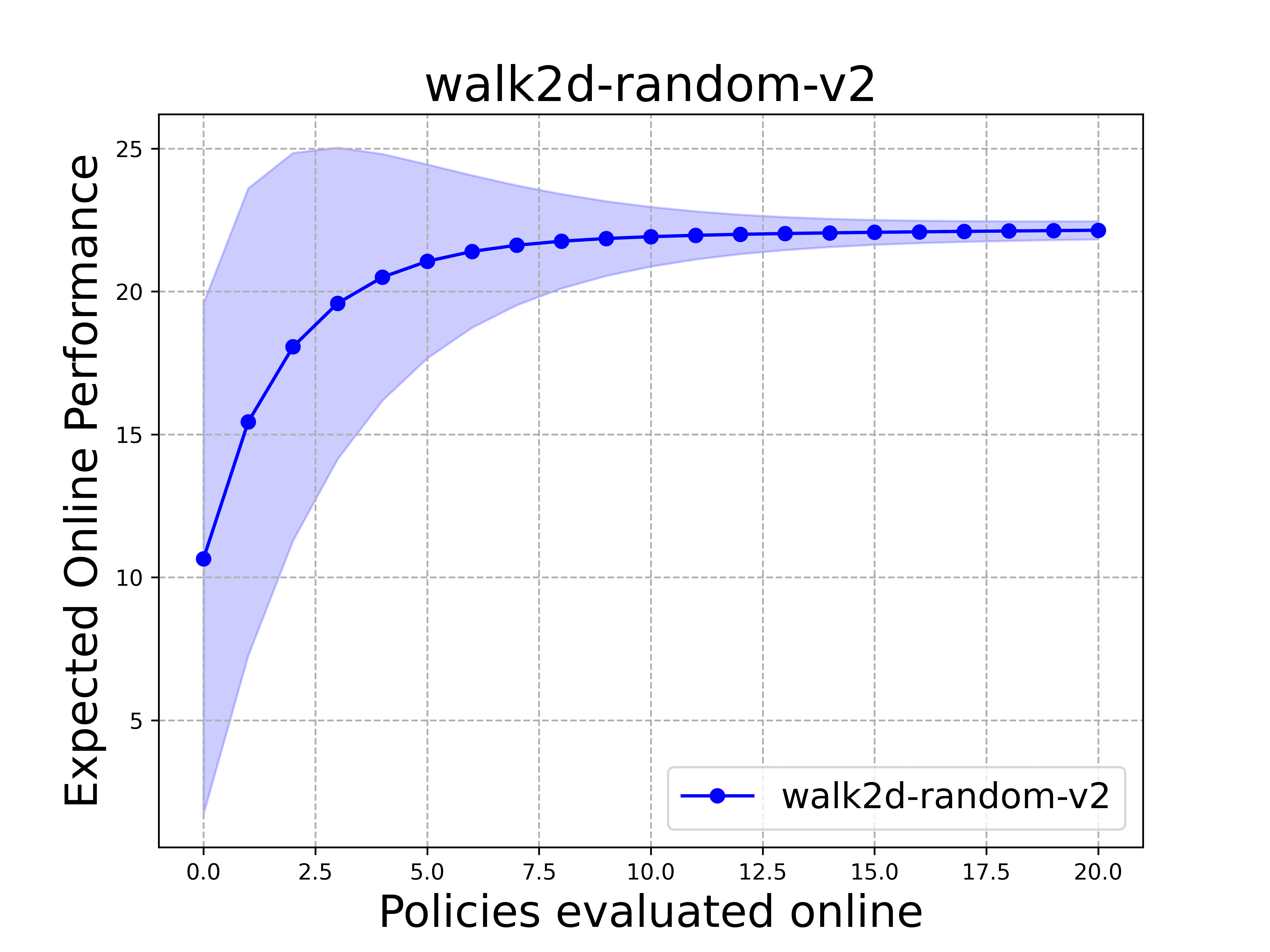}
    \end{minipage}%
    }
    \subfigure{
    \begin{minipage}[t]{0.15\linewidth}
        \centering
        \includegraphics[width=\linewidth]{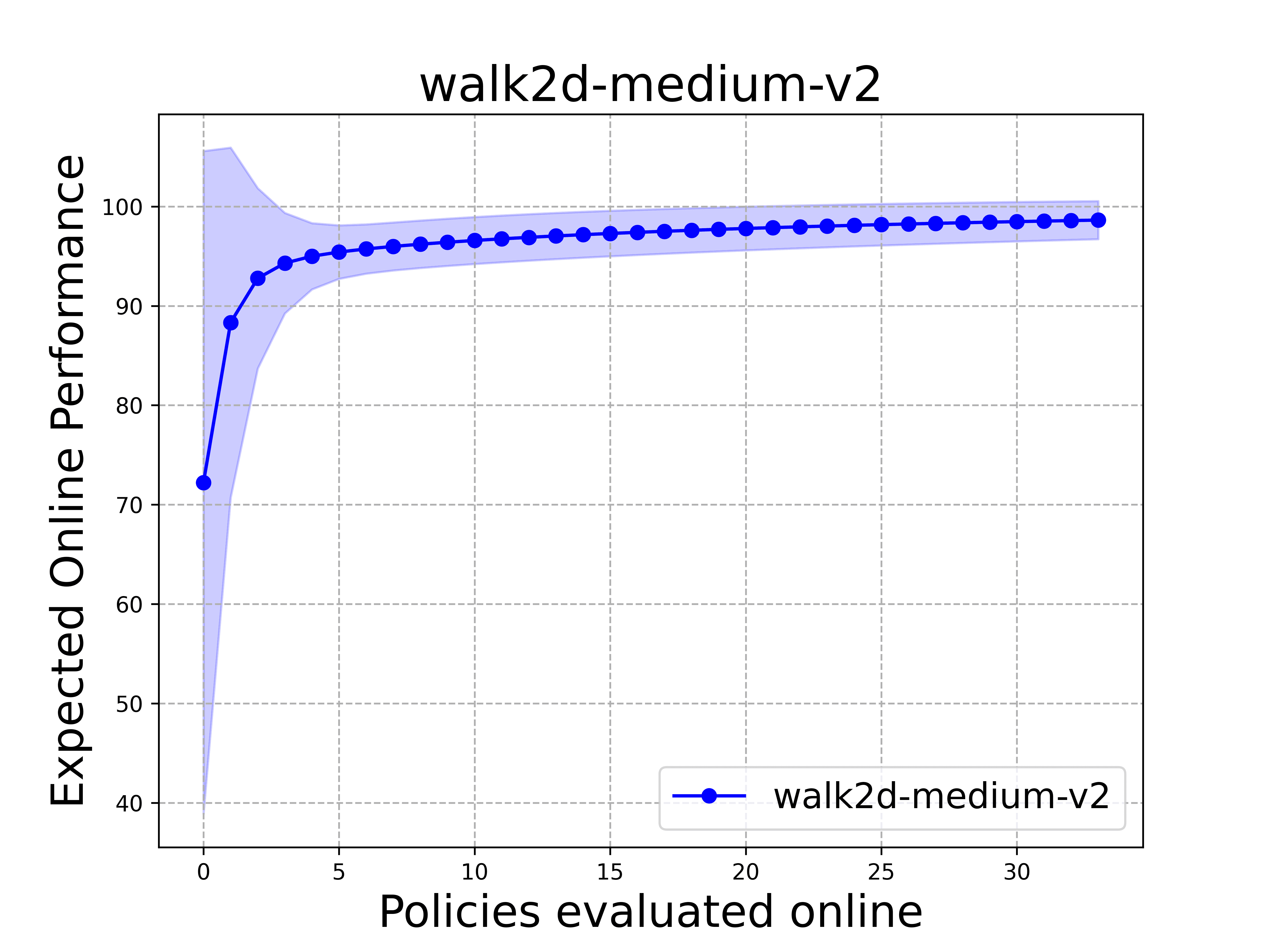}
    \end{minipage}
    }
    \subfigure{
    \begin{minipage}[b]{0.15\linewidth}
        \centering
        \includegraphics[width=\linewidth]{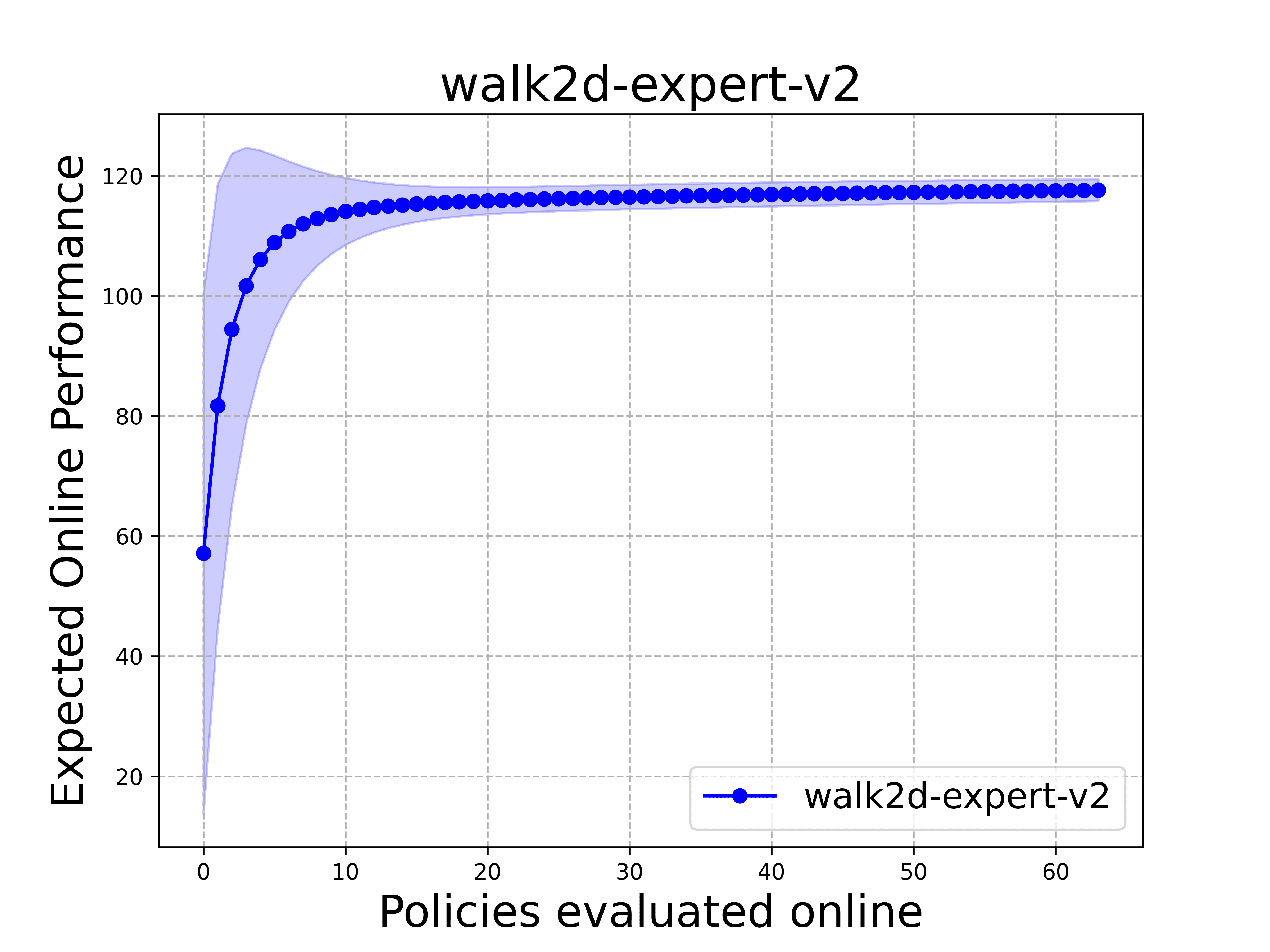}
    \end{minipage}
    }
    \subfigure{
    \begin{minipage}[b]{0.15\linewidth}
        \centering
        \includegraphics[width=\linewidth]{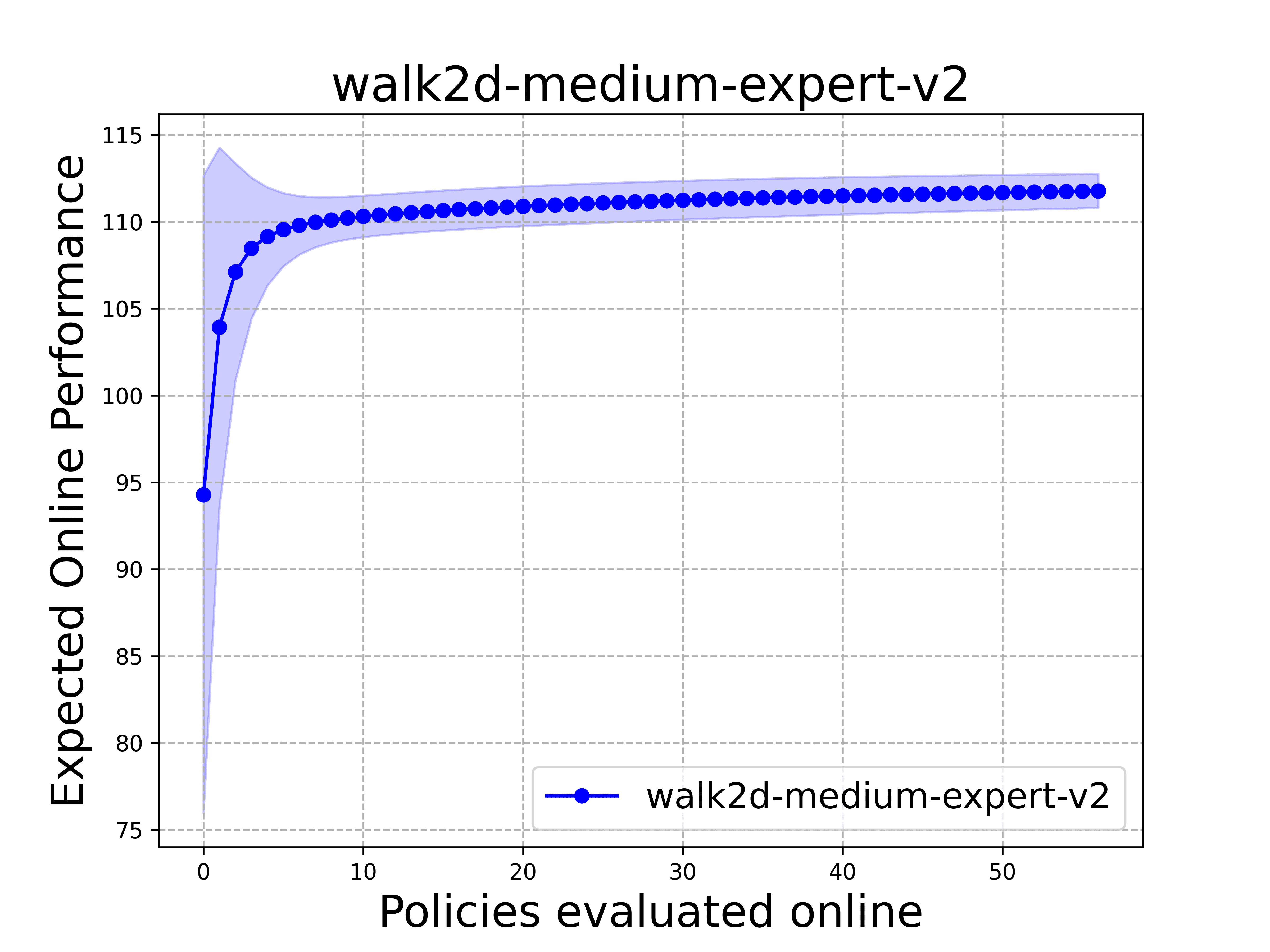}
    \end{minipage}
    }
    \subfigure{
    \begin{minipage}[b]{0.15\linewidth}
        \centering
        \includegraphics[width=\linewidth]{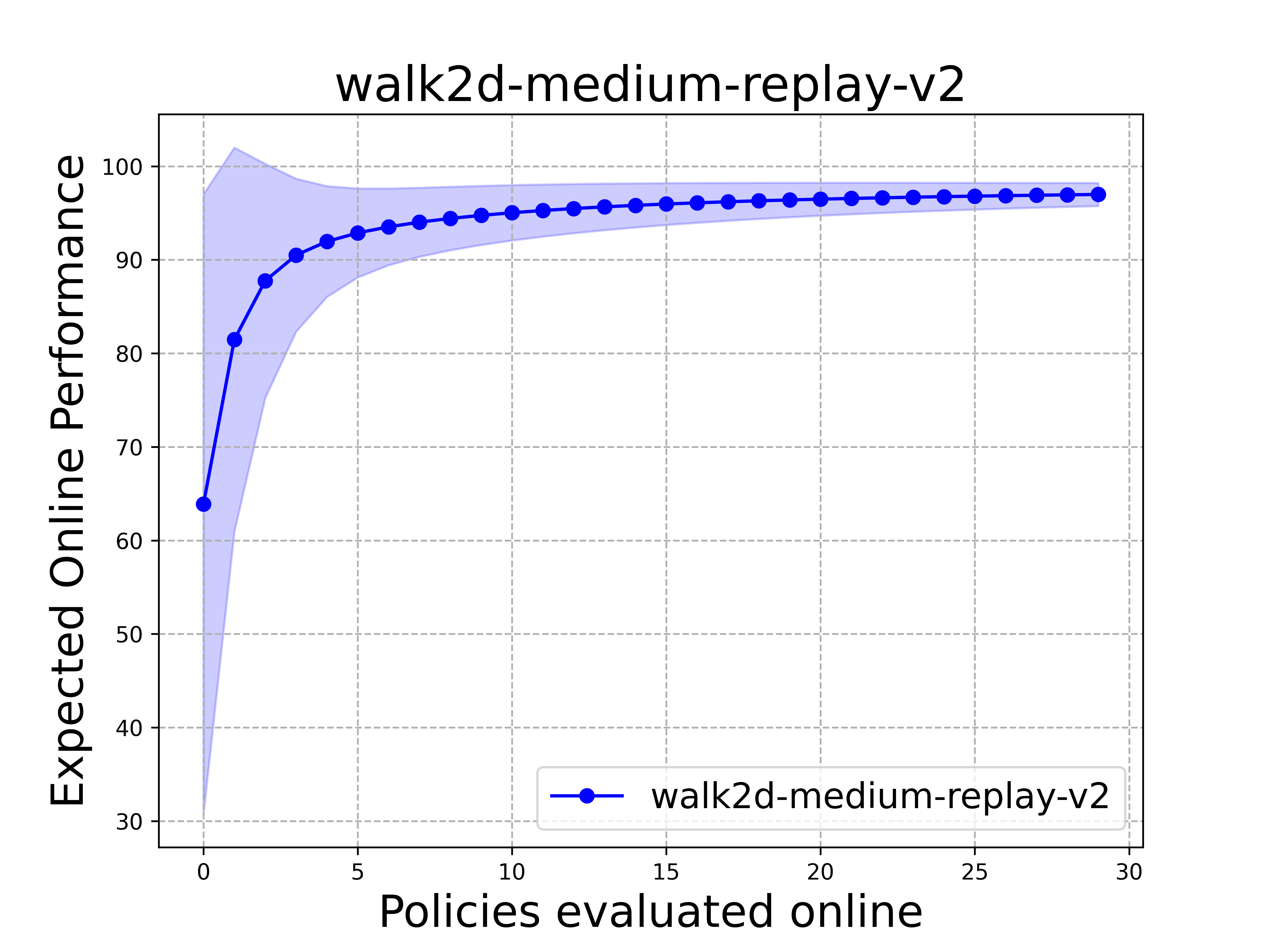}
    \end{minipage}
    }
    \subfigure{
    \begin{minipage}[b]{0.15\linewidth}
        \centering
        \includegraphics[width=\linewidth]{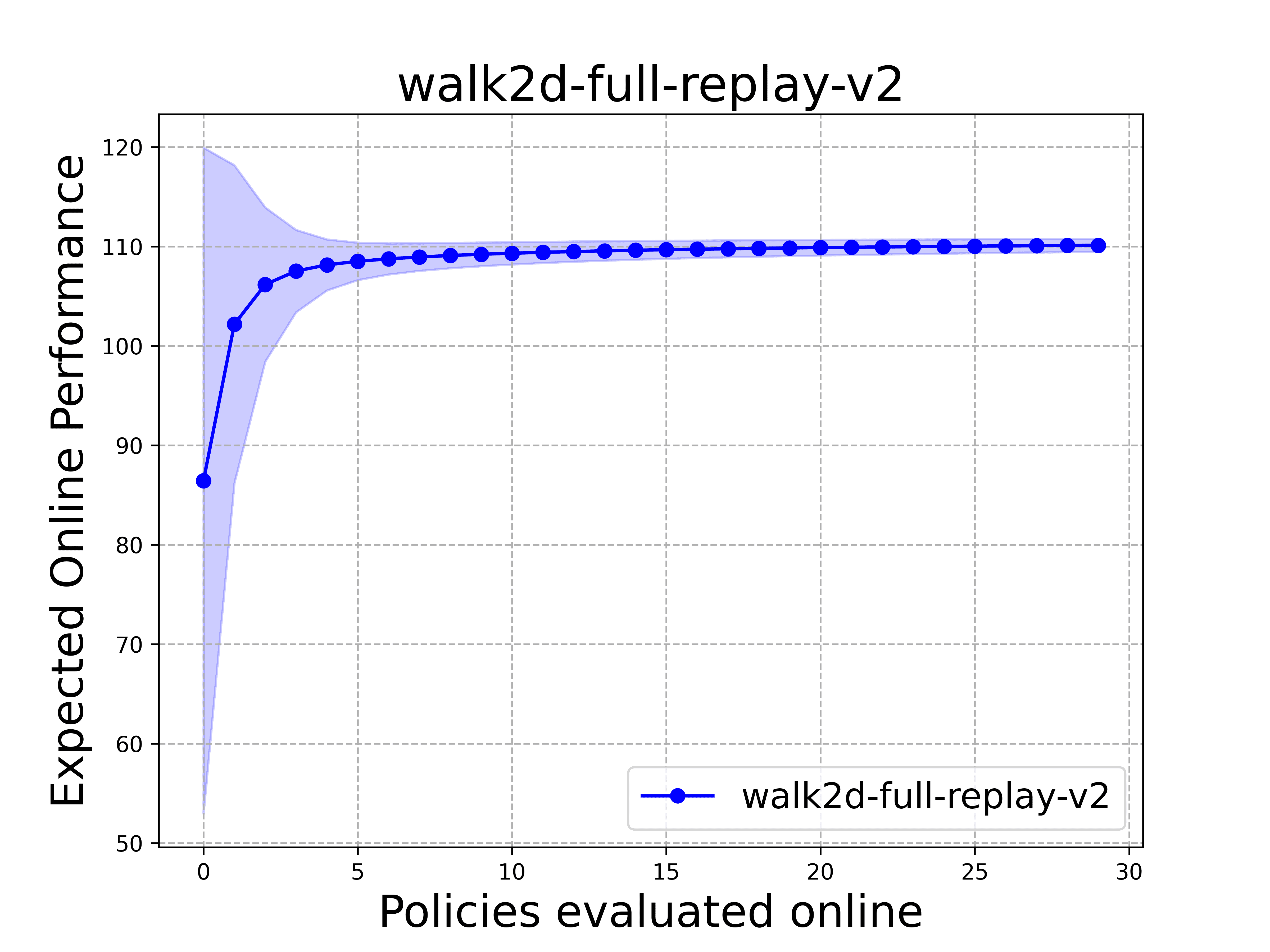}
    \end{minipage}
    }
    \end{minipage}
    \caption{Expected Online Performance lines for Gym-MuJoCo and AntMaze.}
    \label{figure_eop}
\end{figure}

\subsection{Parameter Study on the Number of Target Network}
Our study explored the relationship between the number of different targets and their corresponding final scores in both online MuJoCo tasks and D4RL offline tasks. In our approach, if $\alpha$ is not equal to 1, then $N$ must satisfy the condition $N > 1$. In the following charts, we fill in the values of RND at $N=1$ as a reference for the single target network results.

\subsubsection{Online Tasks}
We conduct adversarial attack experiments with different numbers of target networks in DRND. As shown in Figure \ref{target_ablation_online}, the robustness of DRND generally improves with an increase in the target number $N$. Considering both runtime and performance, we chose $N=10$ as the optimal number of targets for our online experiments.

\begin{figure*}[h]
    \centering
    \subfigure{
    \begin{minipage}[t]{0.24\linewidth}
        \centering
        \includegraphics[height=80.3pt,width=118pt]{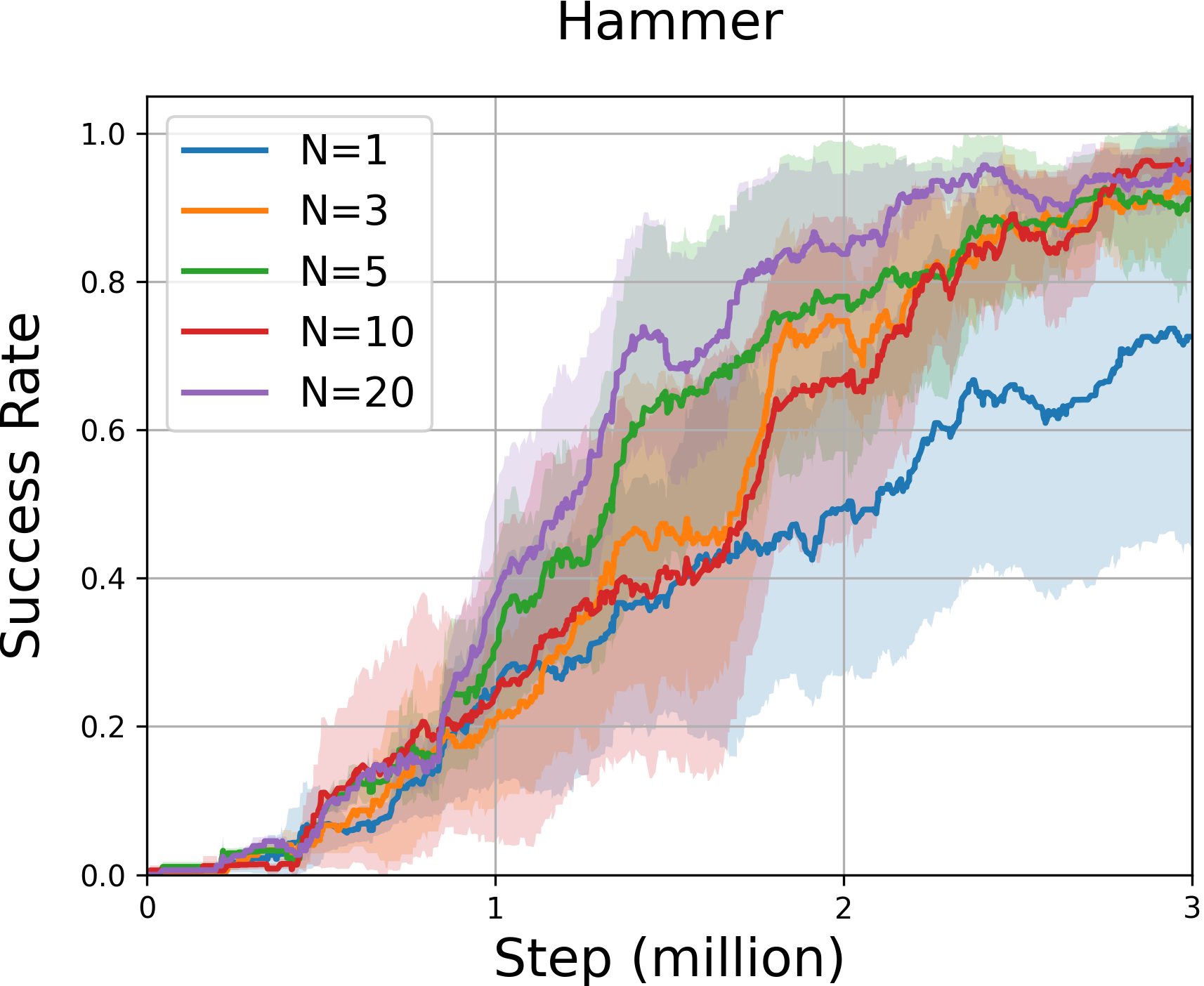}
    \end{minipage}%
    }
    \subfigure{
    \begin{minipage}[t]{0.23\linewidth}
        \centering
        \includegraphics[height=80pt,width=110pt]{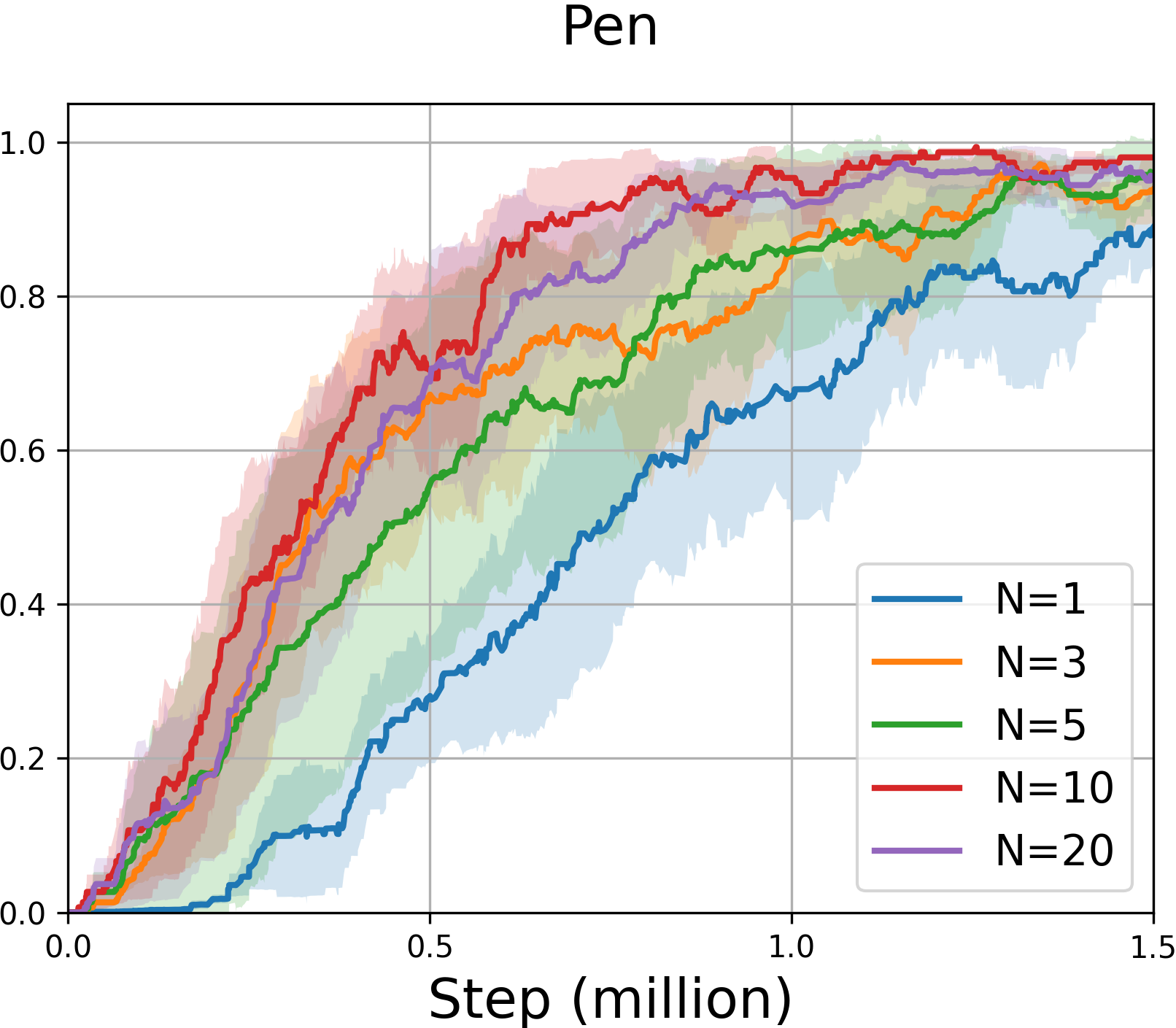}
    \end{minipage}
    }
    \subfigure{
    \begin{minipage}[t]{0.23\linewidth}
        \centering
        \includegraphics[height=80pt,width=110pt]{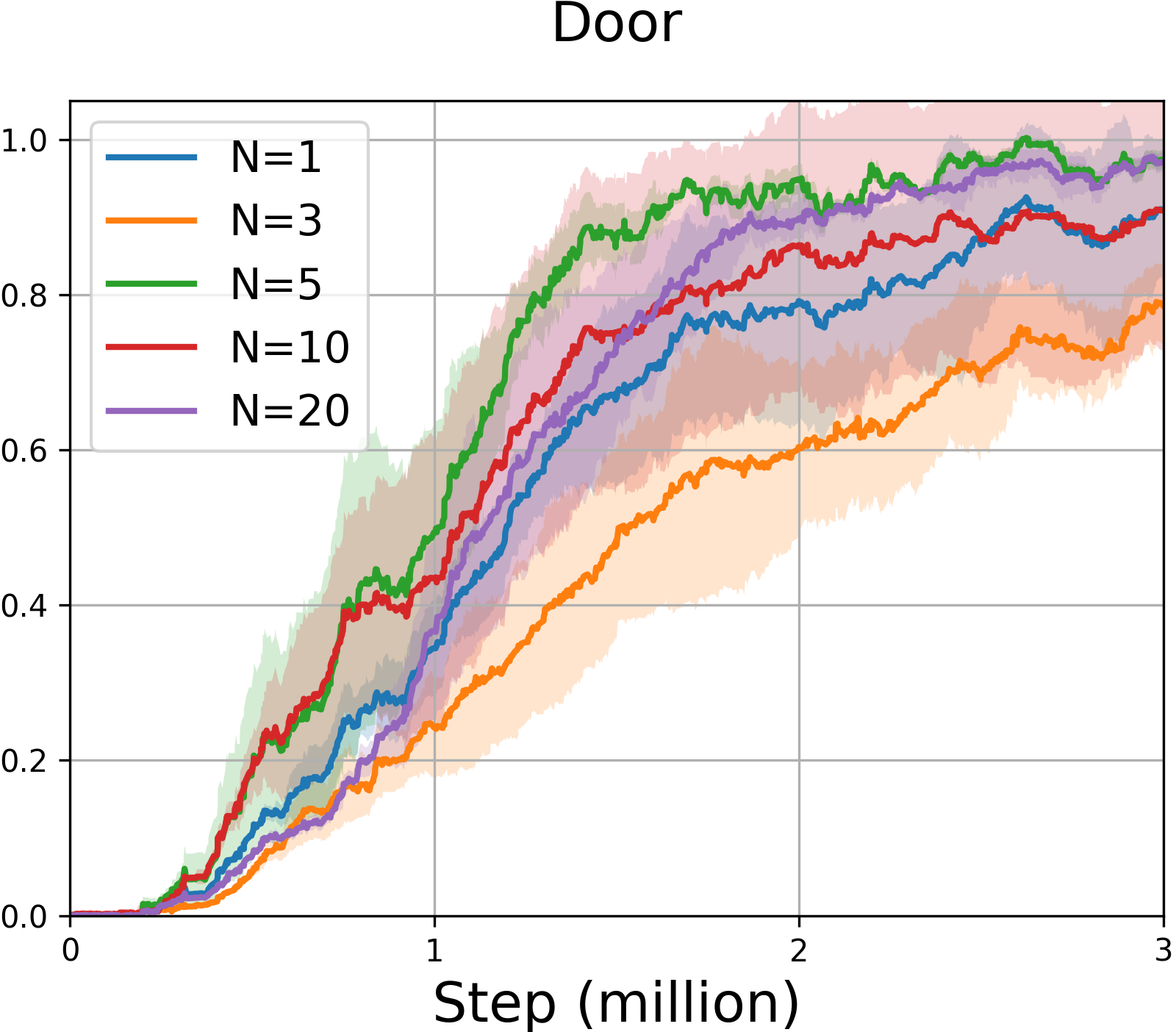}
    \end{minipage}
    }
    \subfigure{
    \begin{minipage}[t]{0.23\linewidth}
        \centering
        \includegraphics[height=80pt,width=110pt]{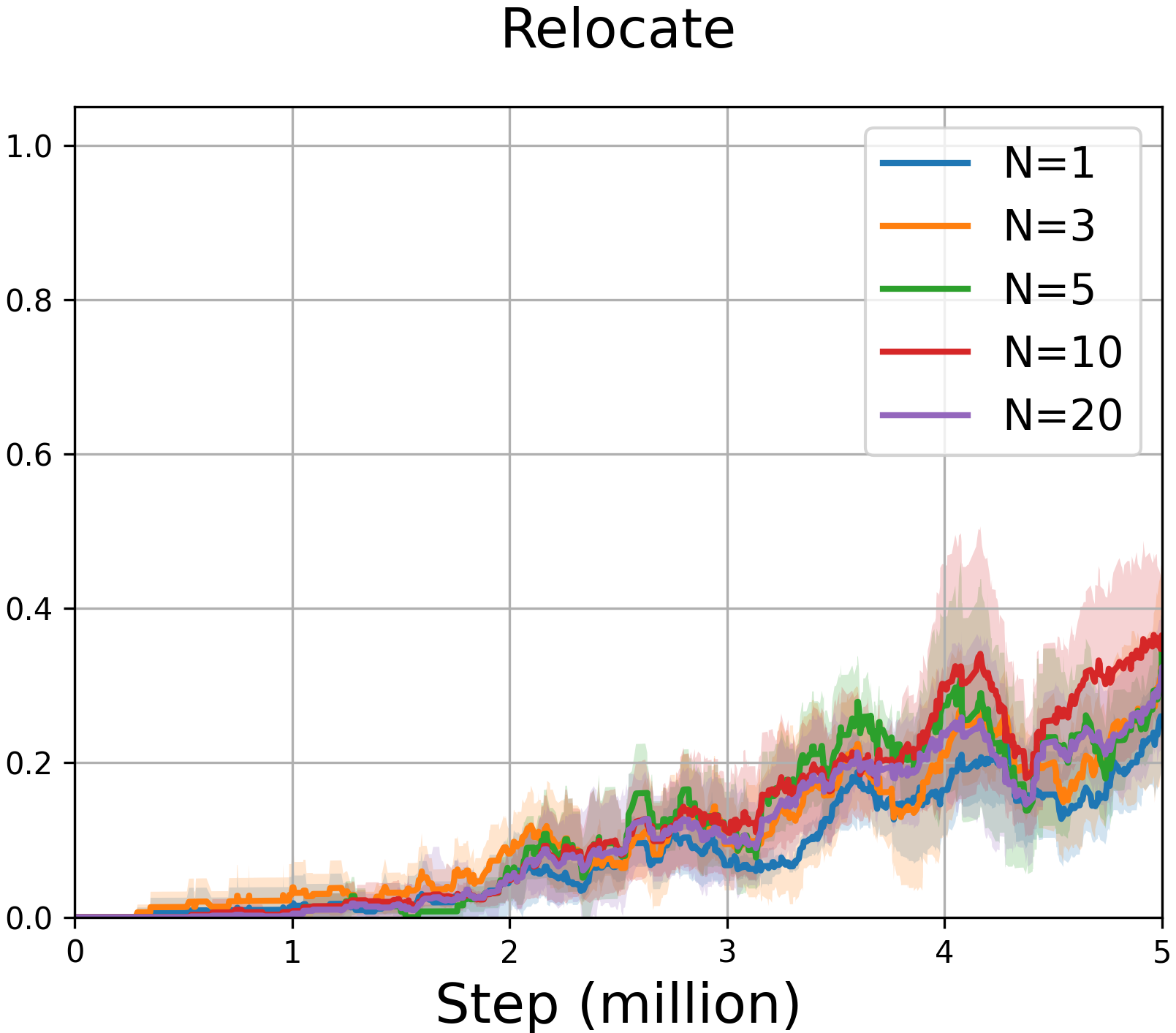}
    \end{minipage}
    }
    \caption{Training curves with different $N$ in Adroit tasks. All curves are averaged over 5 runs.}
    \label{target_ablation_online}
\end{figure*}

\subsubsection{Offline Tasks}
The results are shown in Table \ref{target}. The results indicate that the average score demonstrates an upward trend as the number of targets increases. At the same time, its variance decreases, which suggests that a higher number of targets generally leads to improved and more consistent outcomes. However, it's worth noting that there are diminishing returns; for instance, the differences between the results at $N=10$ and $N=20$ are marginal. Considering these considerations, we chose $N=10$ for our offline experiments. Furthermore, the algorithm exhibits limited sensitivity to variations in the number of targets in online and offline settings.
\begin{table}[h]
\caption{Parameter study of $N$ in offline tasks}
\centering
\renewcommand{\arraystretch}{1.1}
\setlength{\tabcolsep}{4mm}{
\resizebox{\linewidth}{20mm}{
\begin{tabular}{l c c c c c}
\toprule
\diagbox[width=10em,height=1.5em]{Dataset}{N} & $1$ & $3$ & $5$ & $10$ & $20$\\
\midrule
hopper-medium & 92.1 ± 8.4 &93.3 ± 3.7 &97.8 ± 2.4 &98.5 ± 1.1 &\textbf{99.0} ± 0.6\\
halfcheetah-medium & 66.4 ± 1.4 & 65.8 ± 1.8 & 66.7 ± 0.6 & 67.3 ± 0.2 & \textbf{67.4} ± 0.4\\
walker2d-medium& 91.6 ± 2.8 & 94.5 ± 0.9 & 94.0 ± 1.6 & \textbf{95.2} ± 1.2 & 94.7 ± 1.2\\
\hline
average score & 83.4 & 84.5 &  86.2 & \textbf{87.0} & \textbf{87.0} \\
\bottomrule
\label{target}
\end{tabular}
}
}
\end{table}

\subsection{Runtime comparison}
To verify no significant increase in computational overhead between our method and the RND method, we conducted experiments on the medium datasets in the offline D4RL tasks, comparing the computational costs of both methods, as shown in Figure \ref{runtime}. It can be observed that the runtime of our method is slightly less than that of the RND method. And it can be seen that as the number of targets increases, the running time does not significantly improve. 

\begin{figure*}[h]
    \centering
    \includegraphics[height=180pt,width=300pt]{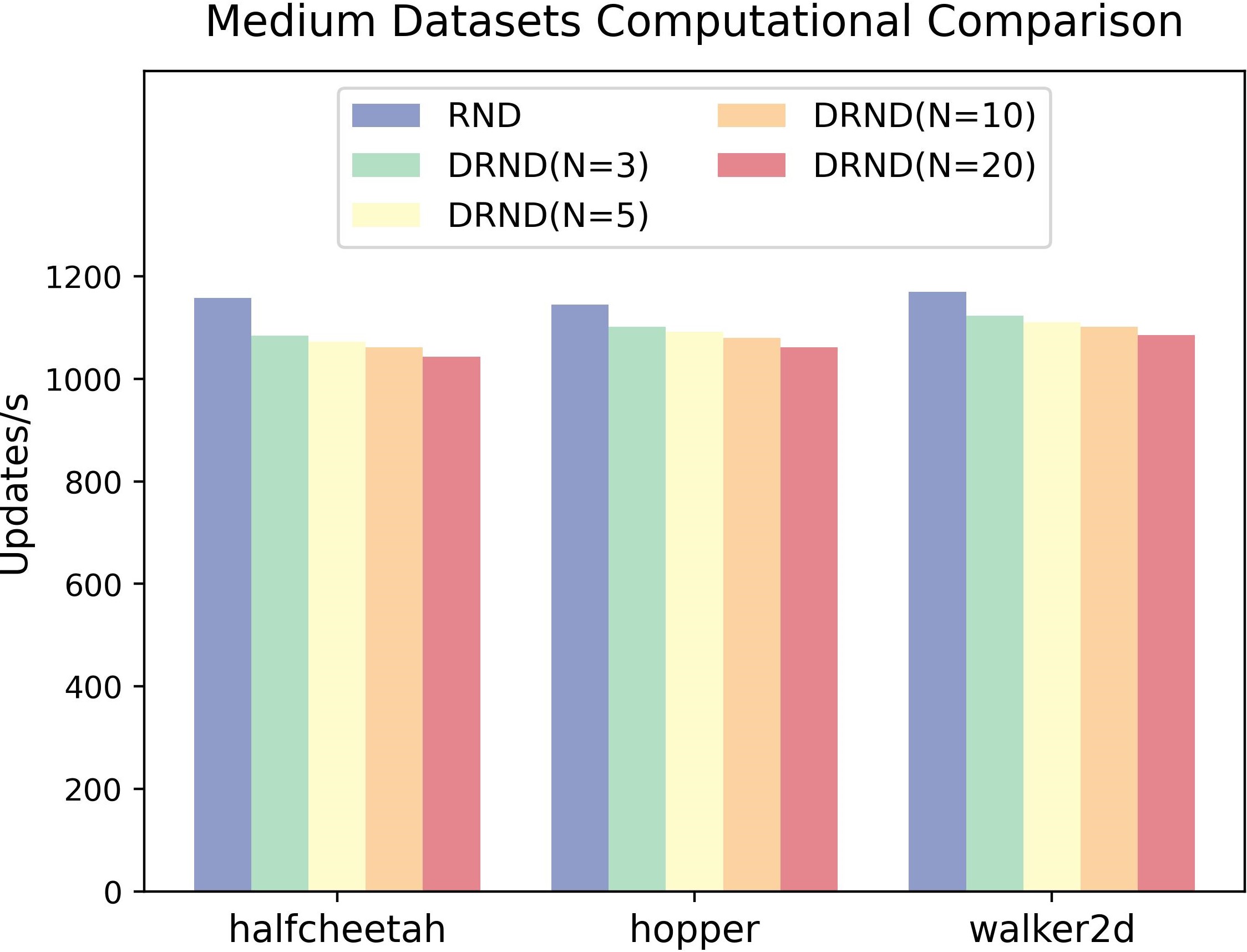}
    \caption{Comparison of updates per second between the RND and DRND methods. We assessed the execution time on a GPU (RTX 3090 24G) and one CPU (Intel(R) Xeon(R) Gold 6226R CPU) over 1M standard updates, using a batch size of 256 with the same network structure.}
    \label{runtime}
\end{figure*}

\subsection{Parameter Study on $\alpha$}
In this subsection, we provide the results of different $\alpha$ with both online and offline tasks. We use varying $\alpha\in\{0, 0.1, 0.5, 0.9, 1\}$.
\subsubsection{Online Tasks}
We study the performance under attacks with different $\alpha$ in online tasks. We chose Adroit continuous control environments as our experiment environments. In the results shown in Figure \ref{alpha_ablation_online}, We observed that the performance is excellent when $\alpha=0.5$ or $\alpha=0.9$ in all four environments. The performance when $\alpha=1$ is not as good as when $\alpha=0.9$, which indirectly confirms the effect of the second bonus term. We chose $\alpha=0.9$ as the hyperparameter for our online experiments.

\begin{figure*}[h]
    \centering
    \subfigure{
    \begin{minipage}[t]{0.24\linewidth}
        \centering
        \includegraphics[height=75pt,width=115pt]{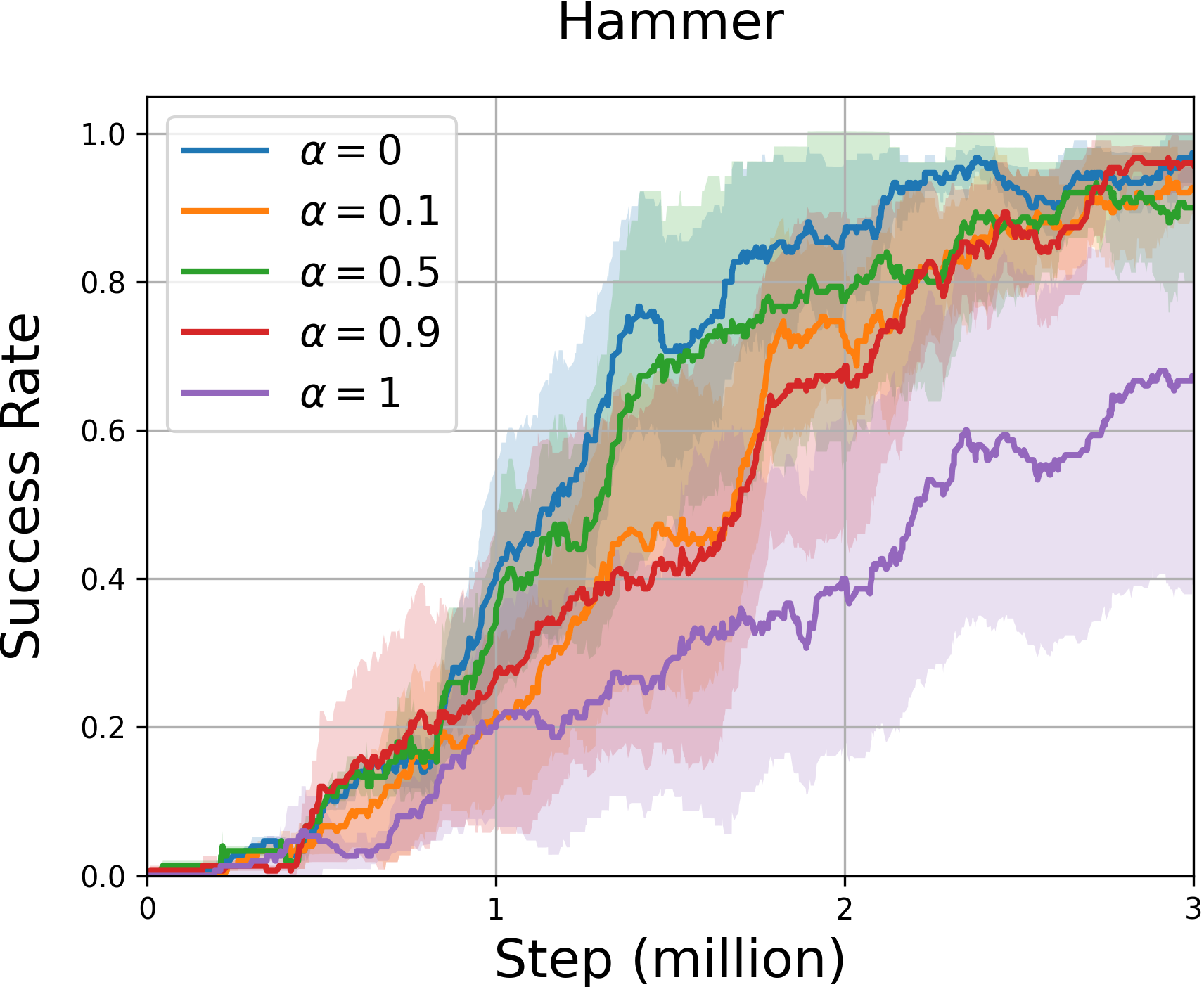}
    \end{minipage}%
    }
    \subfigure{
    \begin{minipage}[t]{0.22\linewidth}
        \centering
        \includegraphics[height=75pt,width=108pt]{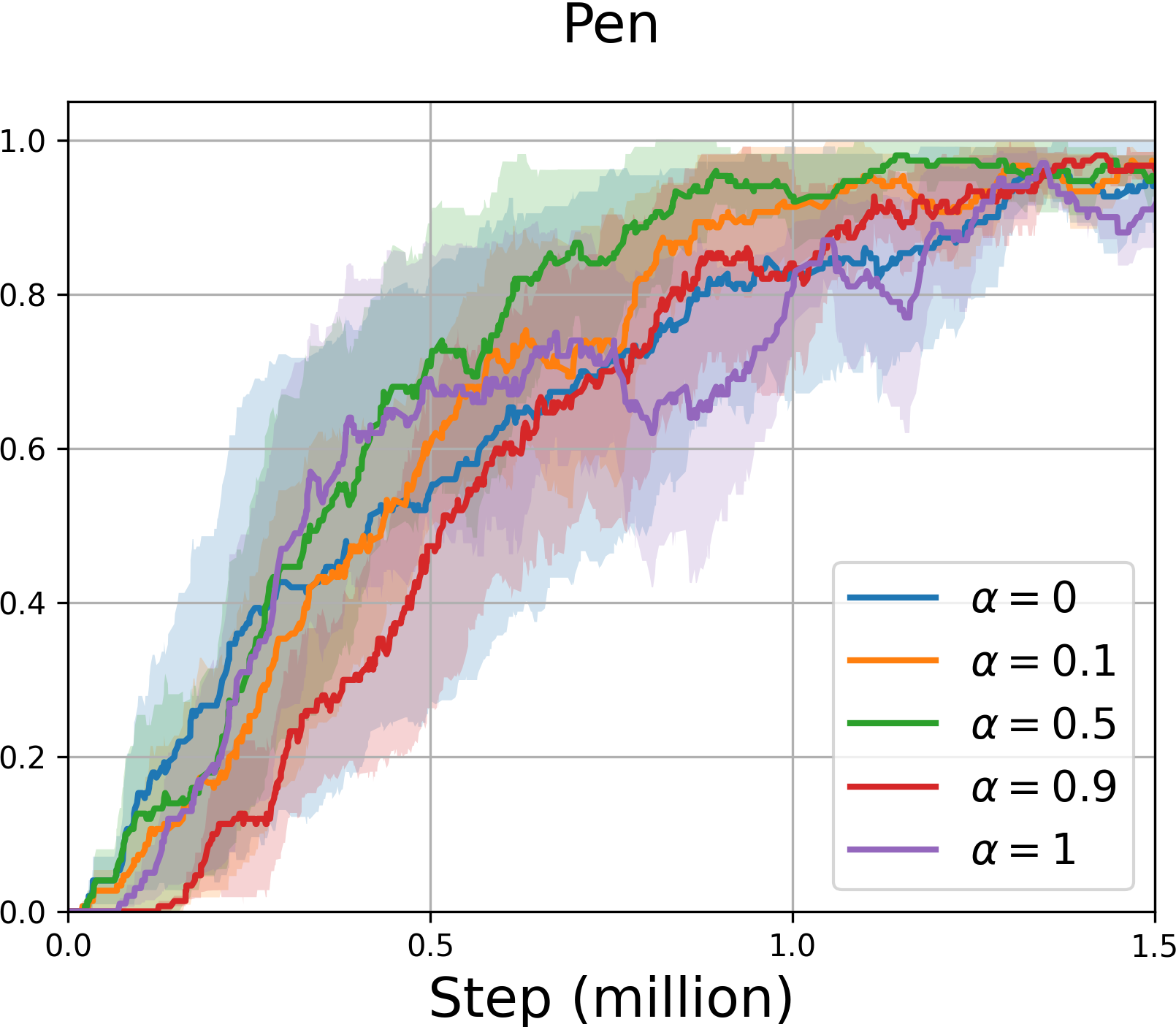}
    \end{minipage}
    }
    \subfigure{
    \begin{minipage}[b]{0.22\linewidth}
        \centering
        \includegraphics[height=75pt,width=108pt]{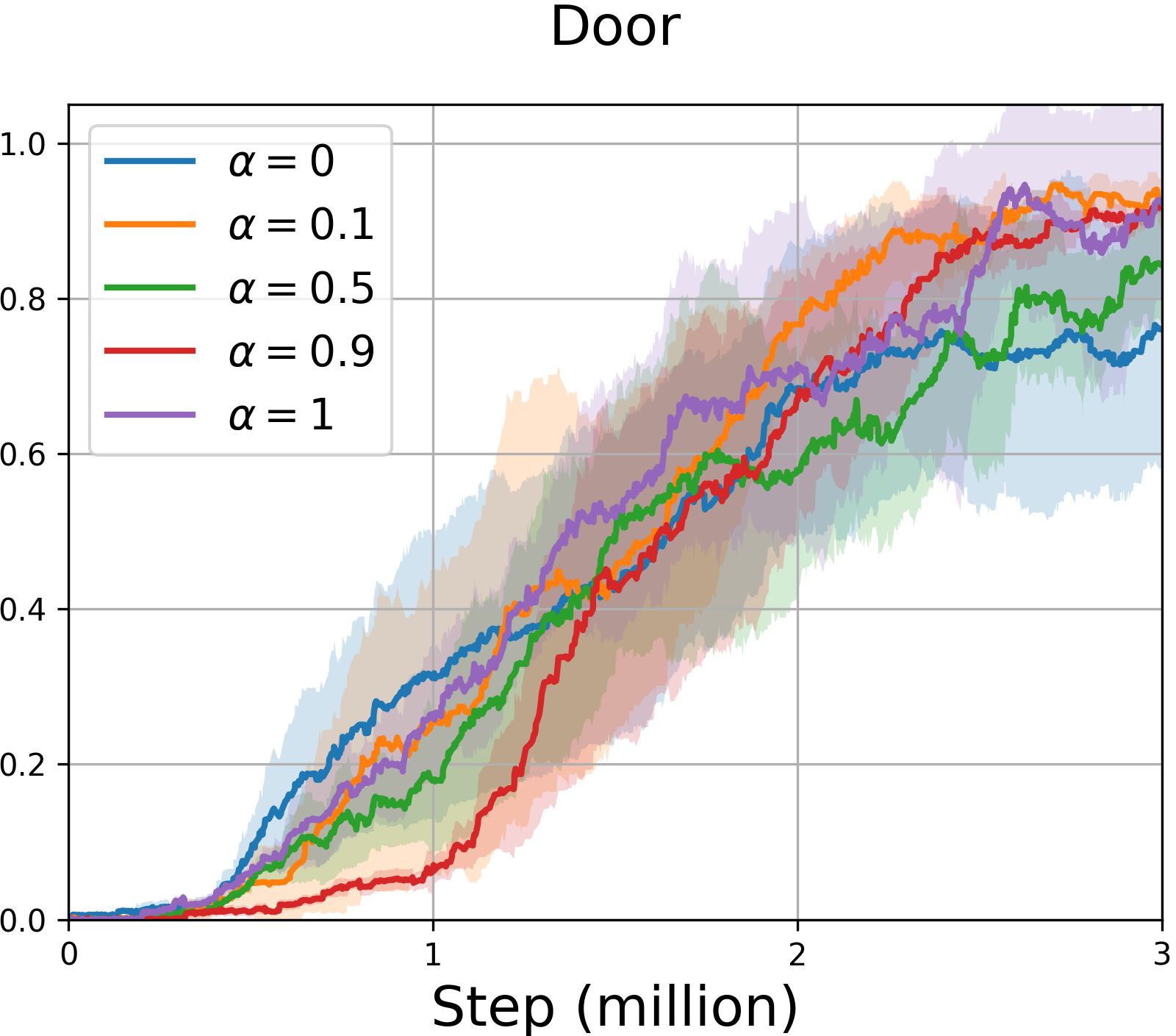}
    \end{minipage}
    }
    \subfigure{
    \begin{minipage}[b]{0.22\linewidth}
        \centering
        \includegraphics[height=75pt,width=110pt]{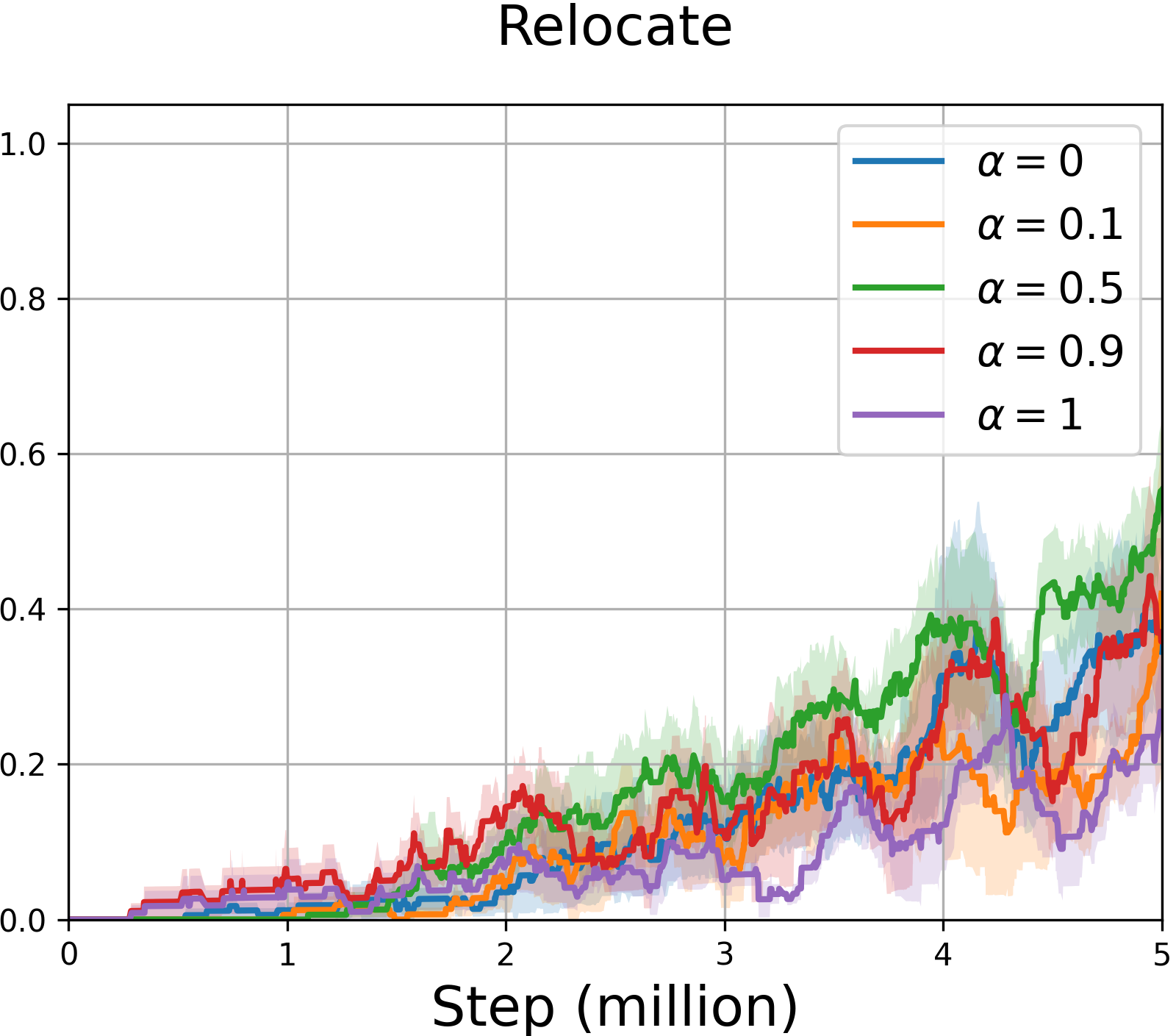}
    \end{minipage}
    }
    \caption{Training curves with different $\alpha$ in Adroit tasks. All curves are averaged over 5 runs.
 }
    \label{alpha_ablation_online}
\end{figure*}

\subsubsection{Offline Tasks}
We examine the influence of $\alpha$ on offline tasks using the D4RL dataset. We employ various values of $\alpha$ to train an offline agent on the `medium' datasets. The final scores are presented in Table \ref{alpha}, and the training curves are shown in Figure \ref{alpha_ablation}. It is observed that in some cases, when $\alpha=0.9$, the final score is higher, and the training curve exhibits greater stability. Consequently, we consistently opted for $\alpha=0.9$ in our offline experiments. When $\alpha=1$, only the first bonus term comes into play, and the results are not as favorable as when $\alpha=0.9$, demonstrating the effectiveness of the second bonus term. Additionally, when examining the final results, it becomes evident that our first bonus outperforms the RND.

Also, for ease of comparison, we provide the training curves of SAC-RND on three datasets: hopper-medium, halfcheetah-medium, and walker2d-medium in Figure \ref{sacrnd}.

\begin{figure}[h]
    \centering
    \subfigure{
    \begin{minipage}[t]{0.33\linewidth}
        \centering
        \includegraphics[height=110pt,width=160pt]{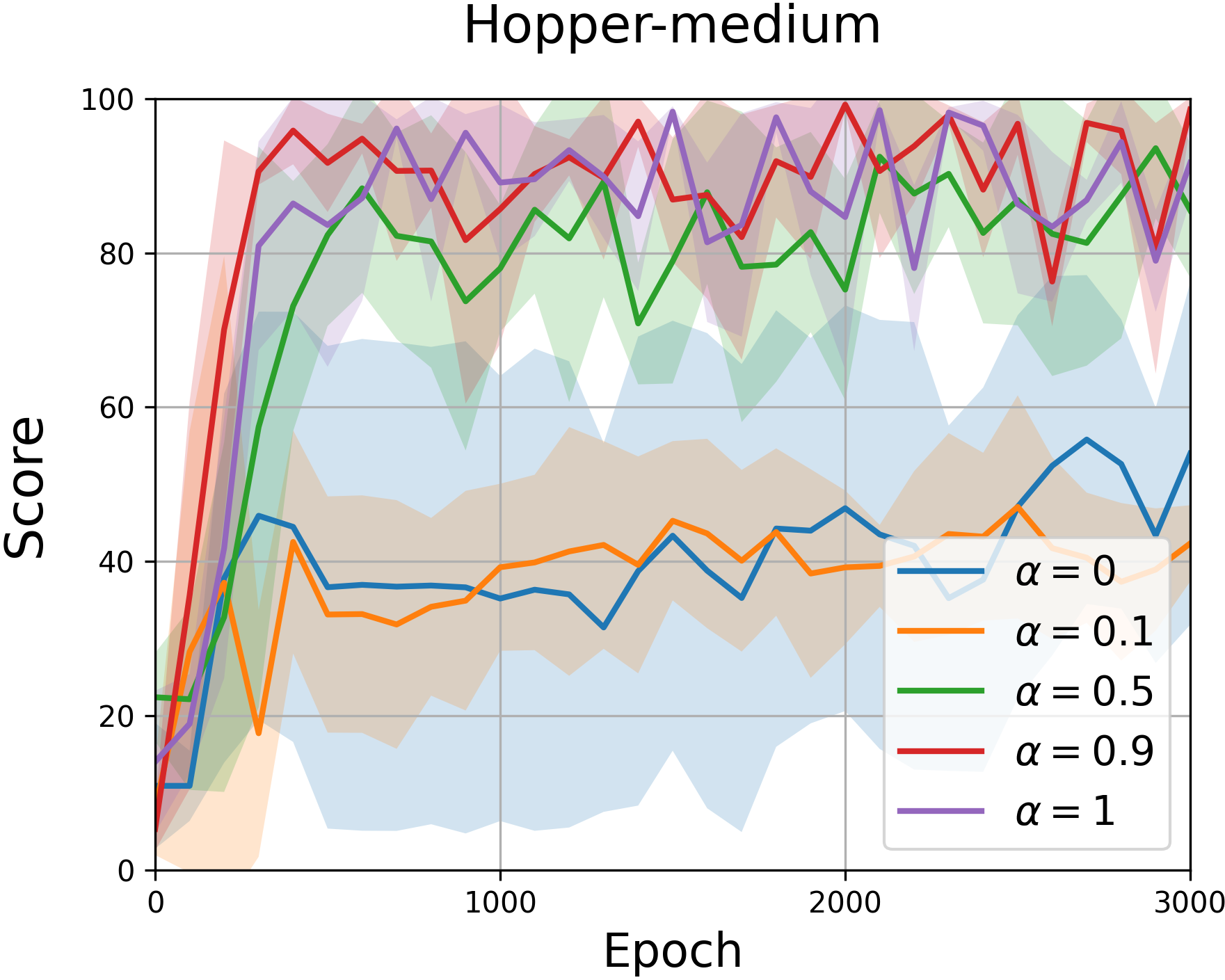}
    \end{minipage}%
    }
    \subfigure{
    \begin{minipage}[t]{0.29\linewidth}
        \centering
        \includegraphics[height=110pt,width=150pt]{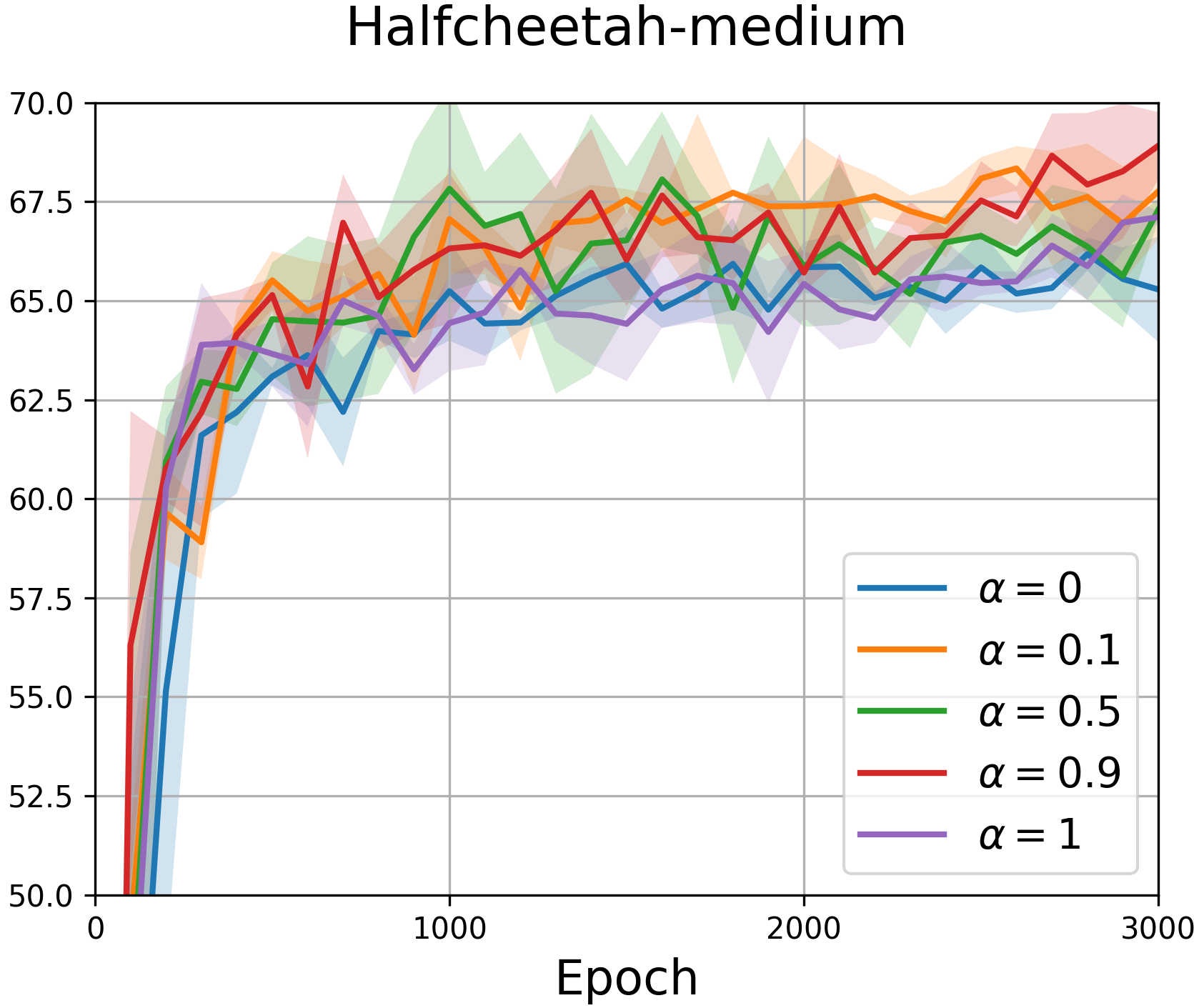}
    \end{minipage}
    }
    \subfigure{
    \begin{minipage}[b]{0.29\linewidth}
        \centering
        \includegraphics[height=110pt,width=150pt]{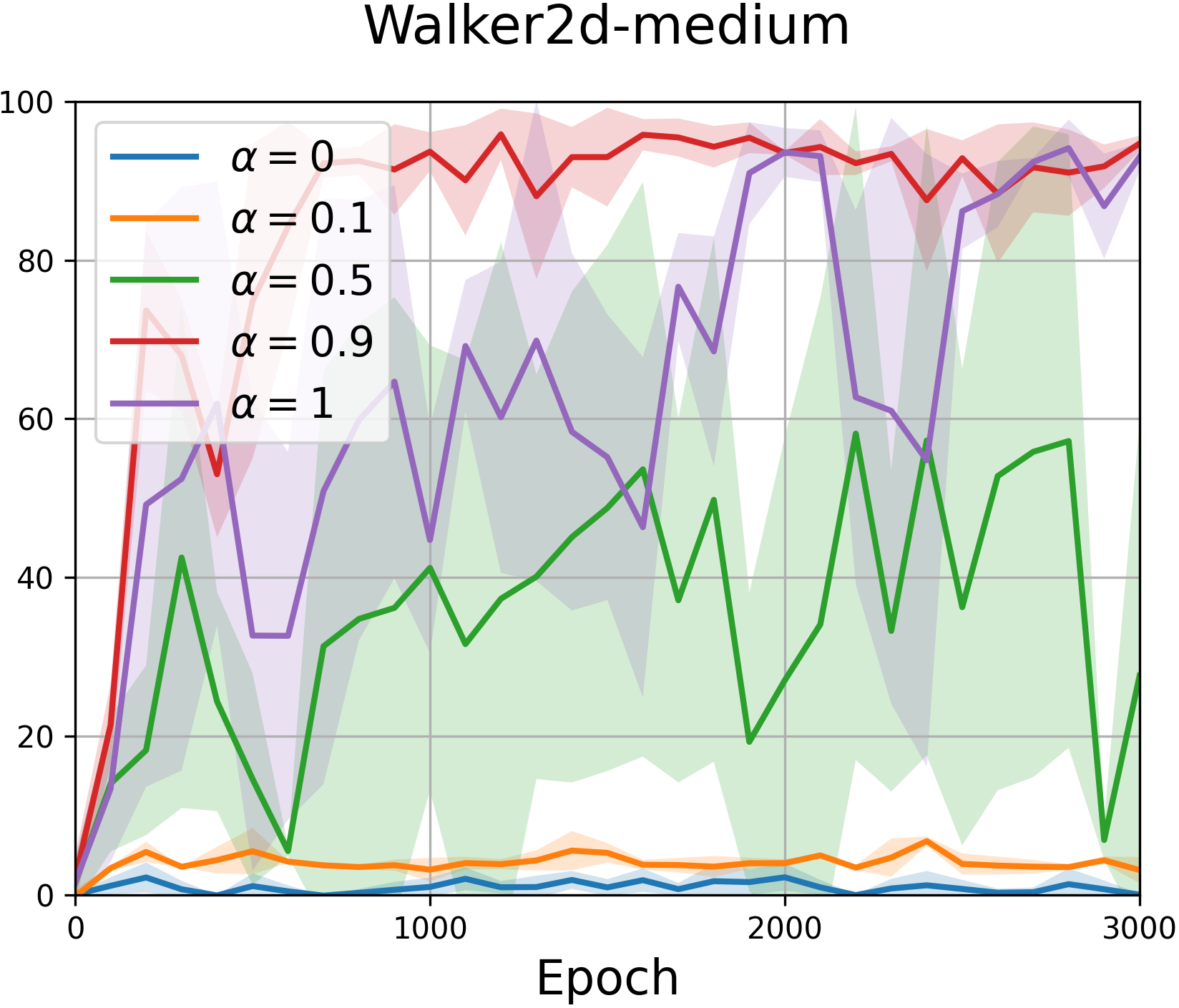}
    \end{minipage}
    }
    \caption{Training curves with different $\alpha$. All curves are averaged over 5 runs.
 }
    \label{alpha_ablation}
\end{figure}

\begin{table}[h]
\caption{The final scores of different $\alpha$ in offline tasks. We also compare the results of SAC-RND since it is a special case when $\alpha=1$ and $N=1$.}
\centering
\renewcommand{\arraystretch}{1.1}
\resizebox{\linewidth}{20mm}{
\begin{tabular}{l c c c c c c}
\toprule
\diagbox[width=10em,height=1.5em]{Dataset}{$\alpha$} & $0.0$ & $0.1$ & $0.5$ & $0.9$ & $1.0$ & SAC-RND\\
\midrule
hopper-medium & 54.0 ± 31.1&42.3 ± 4.95 &85.5 ± 8.7 & \textbf{98.5} ± 5.6 & 91.9 ± 4.9&91.1 ± 10.1\\
halfcheetah-medium & 65.3 ± 1.3 & 67.5 ± 1.2 & 67.3 ± 0.6  & \textbf{68.6}± 0.4 & 67.1 ± 0.2&66.4 ± 1.4\\
walker2d-medium & -0.0 ± 0.1 & 3.1 ± 1.7 & 27.6 ± 31.9  & \textbf{94.7} ± 1.0 & 93.0 ± 1.8& 91.6 ± 2.8\\
\hline
average score & 39.8 & 37.6 &  60.1 & \textbf{87.2} & 84.0 & 83.0\\
\bottomrule
\label{alpha}
\end{tabular}
}
\end{table}

\subsection{Parameter Study on $\lambda$}
We present the results of our parameter study on $\lambda$, which are detailed in the table below. We choose some environments and show the results. All experiments are conducted across 5 seeds, and we report the mean and standard deviation of their results.

\begin{table}[h]
  \centering
  \caption{Hopper-Medium-v2}
  \renewcommand{\arraystretch}{1.1}
   \resizebox{0.6\linewidth}{14mm}{
    \begin{tabular}{cccc}
    \toprule
    \diagbox[width=10em,height=1.5em]{$\lambda_{actor}$}{$\lambda_{critic}$} & $10$ & $15$ & $20$ \\
    \midrule
    $10$ & - & $98.7 \pm 2.0$ & - \\
    $15$ & $95.9 \pm 3.4$ & $98.5 \pm 1.1$ & $96.1 \pm 3.9$ \\
    $20$ & - & $94.65 \pm 2.6$ & - \\
    \bottomrule
    \end{tabular}%
    }
  \label{tab:addlabel}%
\end{table}%

\begin{table}[h]
  \centering
  \caption{Hopper-Medium-Replay-v2}
  \renewcommand{\arraystretch}{1.1}
  \resizebox{0.6\linewidth}{14mm}{
    \begin{tabular}{cccc}
    \toprule
    \diagbox[width=10em,height=1.5em]{$\lambda_{actor}$}{$\lambda_{critic}$} & $5$ & $10$ & $15$ \\
    \midrule
    $3$ & - & $99.9 \pm 1.3$ & - \\
    $5$ & $94.7 \pm 5.9$ & $100.5 \pm 1.0$ & $99.3 \pm 0.4$ \\
    $8$ & - & $99.8 \pm 1.2$ & - \\
    \bottomrule
    \end{tabular}%
  }
\end{table}%

\begin{table}[h]
  \centering
  \caption{HalfCheetah-Medium-v2}
  \renewcommand{\arraystretch}{1.1}
  \resizebox{0.6\linewidth}{14mm}{
    \begin{tabular}{cccc}
    \toprule
\diagbox[width=10em,height=1.5em]{$\lambda_{actor}$}{$\lambda_{critic}$} & $0.05$ & $0.1$ & $0.2$ \\
    \midrule
    $0.5$ & - & $65.4 \pm 0.8$ & - \\
    $1$ & $67.8 \pm 0.4$ & $68.3 \pm 0.2$ & $68.2 \pm 1.1$ \\
    $2$ & - & $68.0 \pm 0.3$ & - \\
    \bottomrule
    \end{tabular}%
  }
\end{table}%

\begin{table}[h]
  \centering
  \caption{HalfCheetah-Full-Replay-v2}
  \renewcommand{\arraystretch}{1.1}
  \resizebox{0.6\linewidth}{14mm}{
    \begin{tabular}{cccc}
    \toprule
    \diagbox[width=10em,height=1.5em]{$\lambda_{actor}$}{$\lambda_{critic}$} & $0.5$ & $1$ & $2$ \\
    \midrule
    $0.5$ & - & $81.2 \pm 0.5$ & - \\
    $1$ & $81.5 \pm 0.7$ & $81.4 \pm 1.7$ & $81.7 \pm 1.3$ \\
    $2$ & - & $80.9 \pm 0.9$ & - \\
    \bottomrule
    \end{tabular}%
  }
\end{table}%

\begin{table}[h]
  \centering
  \caption{Walker-Full-Replay-v2}
  \renewcommand{\arraystretch}{1.1}
  \resizebox{0.6\linewidth}{14mm}{
    \begin{tabular}{cccc}
    \toprule
\diagbox[width=10em,height=1.5em]{$\lambda_{actor}$}{$\lambda_{critic}$} & $1$ & $3$ & $5$ \\
    \midrule
    $1$ & - & $107.2 \pm 1.2$ & - \\
    $3$ & $61.3 \pm 25.8$ & $109.6 \pm 0.7$ & $109.4 \pm 0.4$ \\
    $5$ & - & $109.4 \pm 0.6$ & - \\
    \bottomrule
    \end{tabular}%
  }
\end{table}%

\begin{table}[h]
  \centering
  \caption{Walker-Medium-Expert-v2}
  \renewcommand{\arraystretch}{1.1}
  \resizebox{0.6\linewidth}{14mm}{
    \begin{tabular}{cccc}
    \toprule
\diagbox[width=10em,height=1.5em]{$\lambda_{actor}$}{$\lambda_{critic}$} & $15$ & $20$ & $25$ \\
    \midrule
    $10$ & - & $100.5 \pm 5.6$ & - \\
    $15$ & $93.9 \pm 8.8$ & $109.6 \pm 1.0$ & $85.2 \pm 15.0$ \\
    $20$ & - & $110.3 \pm 0.6$ & - \\
    \bottomrule
    \end{tabular}%
  }
\end{table}%

\subsection{Evaluation on offline-to-online D4RL}\label{offline2online}
We report offline-to-online performance on AntMaze tasks. We followed the methodology outlined by \cite{tarasov2022corl}. We report the scores after the offline stage and online tuning in Table \ref{offline2online_table}.

\begin{table}[h]
\tiny
    \centering
     \setlength{\tabcolsep}{1mm}{
    \resizebox{\linewidth}{18mm}{
    \begin{tabular}{l | c c c c}
\toprule
    \textbf{Task Name}  & \textbf{TD3+BC} & \textbf{IQL}& \textbf{ReBRAC} &\textbf{SAC-DRND}\\
    \midrule
    antmaze-umaze & 66.8 $\rightarrow$ 91.4 & 77.00 $\rightarrow$ 96.50 & 97.8 $\rightarrow$ \textbf{99.8} & 95.8 $\rightarrow$ 98.3\\
    antmaze-umaze-diverse & 59.1 $\rightarrow$ 48.4 & 59.50 $\rightarrow$ 63.75& 85.7 $\rightarrow$ \textbf{98.1} & 87.2 $\rightarrow$ 98.0\\
    antmaze-medium-play & 59.2 $\rightarrow$ 94.8 & 71.75 $\rightarrow$ 89.75 & 78.4 $\rightarrow$ 97.7 & 86.2 $\rightarrow$ \textbf{98.3}\\
    antmaze-medium-diverse & 62.6 $\rightarrow$ 94.1 & 64.25 $\rightarrow$ 92.25 & 78.6 $\rightarrow$ \textbf{98.5} & 83.0 $\rightarrow$ 95.9\\
    antmaze-large-play & 21.5 $\rightarrow$ 0.1 & 38.5 $\rightarrow$ \textbf{64.50} & 47.0 $\rightarrow$ 39.5 & 53.2 $\rightarrow$ 51.5\\
    antmaze-large-diverse & 9.5 $\rightarrow$ 0.4 & 26.75 $\rightarrow$ 64.25 & 66.7 $\rightarrow$ \textbf{77.6} & 50.8 $\rightarrow$ 55.9\\
    \hline
    \textbf{Average} & 46.4 $\rightarrow$ 54.8(\textcolor{green}{+8.4})& 56.29 $\rightarrow$ 78.50(\textcolor{green}{+22.21}) & 75.7 $\rightarrow$ \textbf{85.2}(\textcolor{green}{+8.5}) &  76.0 $\rightarrow$ 83.0(\textcolor{green}{+7.0})\\
    \bottomrule
    \end{tabular}
    }
    }
   \caption{Evaluation on Offline-to-online Setting. We compared the TD3+BC, IQL and ReBRAC algorithms, and their values were copied from \cite{tarasov2023revisiting}.}
   \label{offline2online_table}
\end{table}

\subsection{More Detailed Change Process of RND Bonus}\label{rndgif}
\begin{figure}[h]
    \centering
    \subfigure{
    \begin{minipage}[t]{0.2\linewidth}
        \centering
        \includegraphics[height=75pt,width=95pt]{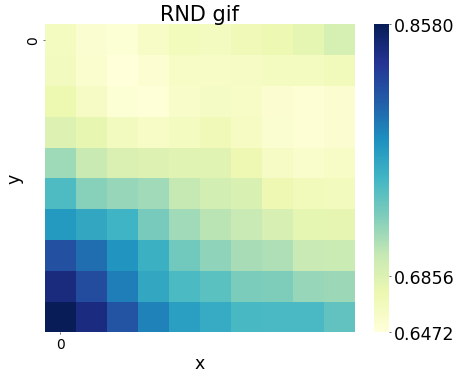}
    \end{minipage}%
    }
    \subfigure{
    \begin{minipage}[t]{0.18\linewidth}
        \centering
        \includegraphics[height=75pt,width=90pt]{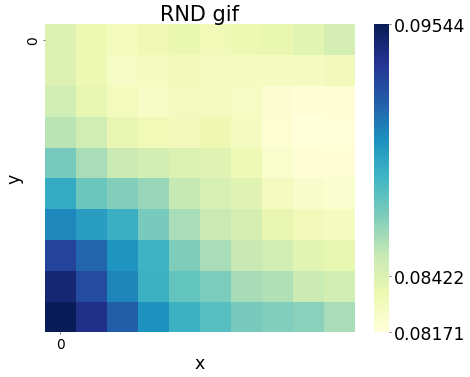}
    \end{minipage}
    }
    \subfigure{
    \begin{minipage}[b]{0.18\linewidth}
        \centering
        \includegraphics[height=75pt,width=90pt]{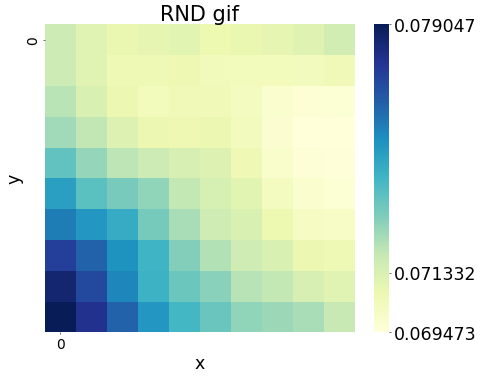}
    \end{minipage}
    }
    \subfigure{
    \begin{minipage}[b]{0.18\linewidth}
        \centering
        \includegraphics[height=75pt,width=90pt]{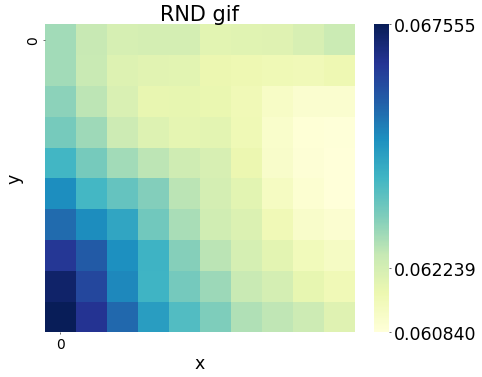}
    \end{minipage}
    }
    \subfigure{
    \begin{minipage}[b]{0.18\linewidth}
        \centering
        \includegraphics[height=75pt,width=90pt]{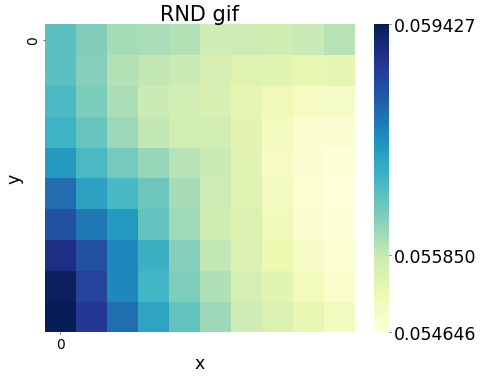}
    \end{minipage}
    }
    \caption{More Detailed Change Process of RND Bonus.}
\end{figure}

\begin{figure}[h]
    \centering
    \includegraphics[height=0.5\linewidth,width=0.7\linewidth]{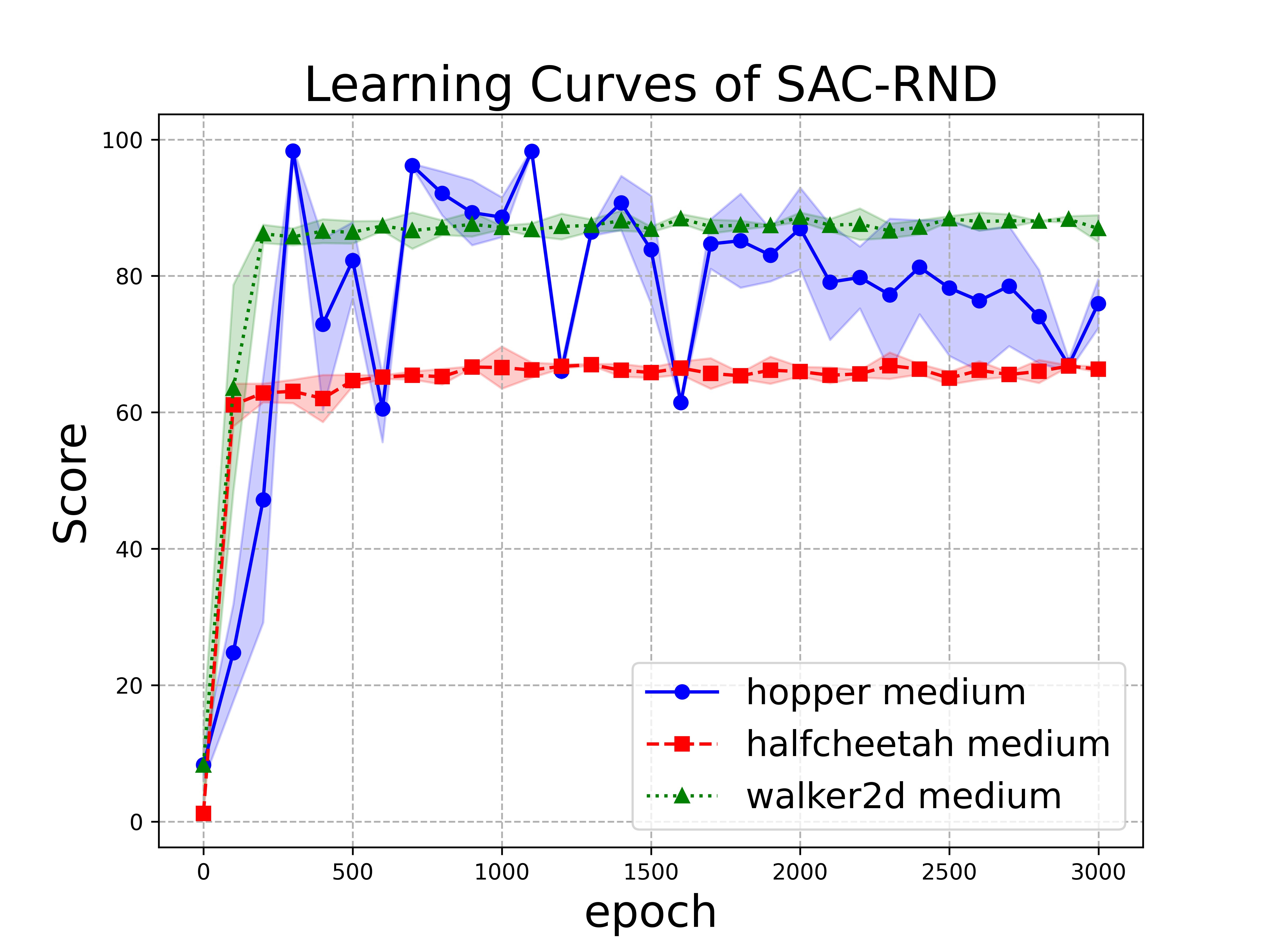}
    \caption{Learning curves of SAC-RND. The parameters are the same as in the original paper.}
    \label{sacrnd}
\end{figure}

\end{document}